\documentclass{article}

\PassOptionsToPackage{numbers, sort&compress}{natbib}

 \usepackage[final]{neurips_2020}

\usepackage[utf8]{inputenc} 
\usepackage[T1]{fontenc}    
\usepackage[hidelinks]{hyperref}       
\usepackage{url}            
\usepackage{booktabs}       
\usepackage{amsfonts}       
\usepackage{nicefrac}       
\usepackage{microtype}      
\usepackage{xspace}
\usepackage{wrapfig}
\usepackage{algorithmic}
\usepackage{subcaption}
\usepackage{graphicx}
\usepackage{url}
\graphicspath{{figures/}}

\usepackage{amsmath}
\usepackage{amsfonts}
\usepackage{amssymb}
\usepackage{amsthm}
\usepackage{bm}
\usepackage{bbm}
\usepackage{mathtools}
\usepackage{enumitem}
\usepackage{thmtools,thm-restate}
\usepackage{algorithm}
\usepackage{algorithmic}
\usepackage{natbib}
\usepackage[capitalise]{cleveref}
\usepackage{color}
\usepackage[dvipsnames]{xcolor}
\usepackage{graphicx}
\usepackage{comment}

\theoremstyle{plain}
\newtheorem{lemma}{Lemma}

\theoremstyle{definition}
\newtheorem{definition}{Definition}

\theoremstyle{remark}

\newcommand{\red}[1]{\leavevmode\color{red}{#1}}
\newcommand{\blue}[1]{{\leavevmode\color{black}{#1}}}

\def\AA{\mathcal{A}}
\def\DD{\mathcal{D}}

\def\PP{\mathcal{P}}
\def\SS{\mathcal{S}}
\def\XX{\mathcal{X}}

\def\Ebb{\mathbb{E}}

\def\Rbb{\mathbb{R}}

\def\R{\Rbb}

\def\*{\star}
\def\what{\widehat}
\DeclareMathSymbol{\mhef}{\mathord}{operators}{`\-}

\DeclareMathOperator*{\argmax}{arg\,max}

\def\regret{\textrm{Regret}}

\newcommand{\E}{\Ebb}

\def\algfull{Max-aggregation of Multiple Baselines\xspace}
\def\alg{MAMBA\xspace}
\def\agg{AggreVaTe(D)\xspace}
\def\aggd{AggreVaTeD\xspace}
\def\pg{PG-GAE\xspace}

\def\benchmark{max-aggregated baseline\xspace}

\def\PiE{{\Pi^{\text{e}}}}
\def\pie{{\pi^{\text{e}}}}
\def\tpie{{\pi_+^{\text{e}}}}
\def\ts{{t_e}}

\newcommand{\fmax}{\ensuremath{f^{\max}}}
\newcommand{\pimax}{\ensuremath{\pi^{\max}}}
\newcommand{\fmaxh}{\ensuremath{\what f^{\max}}}
\newcommand{\Amax}{\ensuremath{A^{\max}}}

\newcommand{\Ah}{\ensuremath{\what A}}

\newcommand{\Amaxpi}[1][\pi]{\ensuremath{A^{\max,#1}}}
\renewcommand{\comment}[1]{}

\newcommand{\cheng}[1]{{\color{CarnationPink}\textbf{Ching-An:} #1}}

\title{Policy Improvement via Imitation of Multiple Oracles} 
\author{
	Ching-An Cheng\\
	Microsoft Research\\
	Redmond, WA 98052\\
	\texttt{chinganc@microsoft.com}
	\And
	Andrey Kolobov\\
	Microsoft Research\\
	Redmond, WA 98052\\
	\texttt{akolobov@microsoft.com}
	\And
	Alekh Agarwal\\
	Microsoft Research\\
	Redmond, WA 98052\\
	\texttt{alekha@microsoft.com}
	}

\begin{document}

\maketitle

\begin{abstract}
Despite its promise, reinforcement learning's real-world adoption has been hampered by the need for costly exploration to learn a good policy. Imitation learning (IL) mitigates this shortcoming by using an oracle policy during training as a bootstrap to accelerate the learning process. However, in many practical situations, the learner has access to multiple suboptimal oracles, which may provide conflicting advice in a state.
The existing IL literature provides a limited treatment of such scenarios. Whereas in the single-oracle case, the return of the oracle's policy provides an obvious benchmark for the learner to compete against, neither such a benchmark nor principled ways of outperforming it are known for the multi-oracle setting.
In this paper, we propose the state-wise maximum of the oracle policies' values as a natural baseline to resolve conflicting advice from multiple oracles.
Using a reduction of policy optimization to online learning, we introduce a novel IL algorithm \alg, which can provably learn a policy competitive with this benchmark.
In particular, \alg optimizes policies by using a gradient estimator in the style of generalized advantage estimation (GAE). Our theoretical analysis shows that this design makes \alg robust and enables it to outperform the oracle policies by a larger margin than the IL state of the art, even in the single-oracle case.
In an evaluation against standard policy gradient with GAE and \agg, we showcase \alg's ability to leverage demonstrations both from a single and from multiple weak oracles, and significantly speed up policy optimization.

\end{abstract}

\section{Introduction} \label{sec:introduction}

Reinforcement learning (RL) promises to bring self-improving decision-making capability to many applications, including robotics~\citep{kober2013robotics}, computer systems~\citep{luong2019applications}, recommender systems~\citep{swami2017slate} and user interfaces~\citep{liu2018reinforcement}.
However, deploying RL in any of these domains is fraught with numerous difficulties, as vanilla RL agents need to do a large amount of trial-and-error exploration before discovering good decision policies~\citep{mania2018random}.
This inefficiency has motivated investigations into training RL agents with domain knowledge, an example of which is having access to oracle policies in the training phase. 

The broad class of approaches that attempt to mimic or improve upon an available oracle policy is known as \emph{imitation learning} (IL) \cite{Osa_2018}.
Generally, IL algorithms work by invoking oracle policy demonstrations
to guide an RL agent towards promising states and actions.
As a result, oracle-level performance can be achieved without global exploration, thus avoiding RL's main source of high sample complexity.
For IL with a single oracle policy, the oracle policy's return provides a natural benchmark for the agent to match or outperform.
Most existing IL techniques assume this single-oracle setting, with a good \emph{but possibly suboptimal} oracle policy.
{Behavior cloning}~\citep{pomerleau89alvinn} learns a policy from a fixed batch of trajectories in a supervised way by treating oracle actions as labels. Inverse reinforcement learning uses recorded oracle trajectories to infer the oracle's reward function
~\citep{abbeel2004irl, ziebart2008maxent, finn2016gcl,ho2016generative}. {Interactive} IL~\citep{ross2011dagger, ross2014reinforcement} assumes the learner can actively ask an oracle policy for a demonstration starting at the learner's current state. When reward information of the original RL problem is available, IL algorithms can outperform the oracle policy~\citep{sun2017deeply,chang2015learning,cheng2018loki}.

\begin{figure}
	\centering
	\hspace{-15mm}
	\begin{subfigure}[b]{0.32\textwidth}
		\centering
		\includegraphics[width=0.8\textwidth]{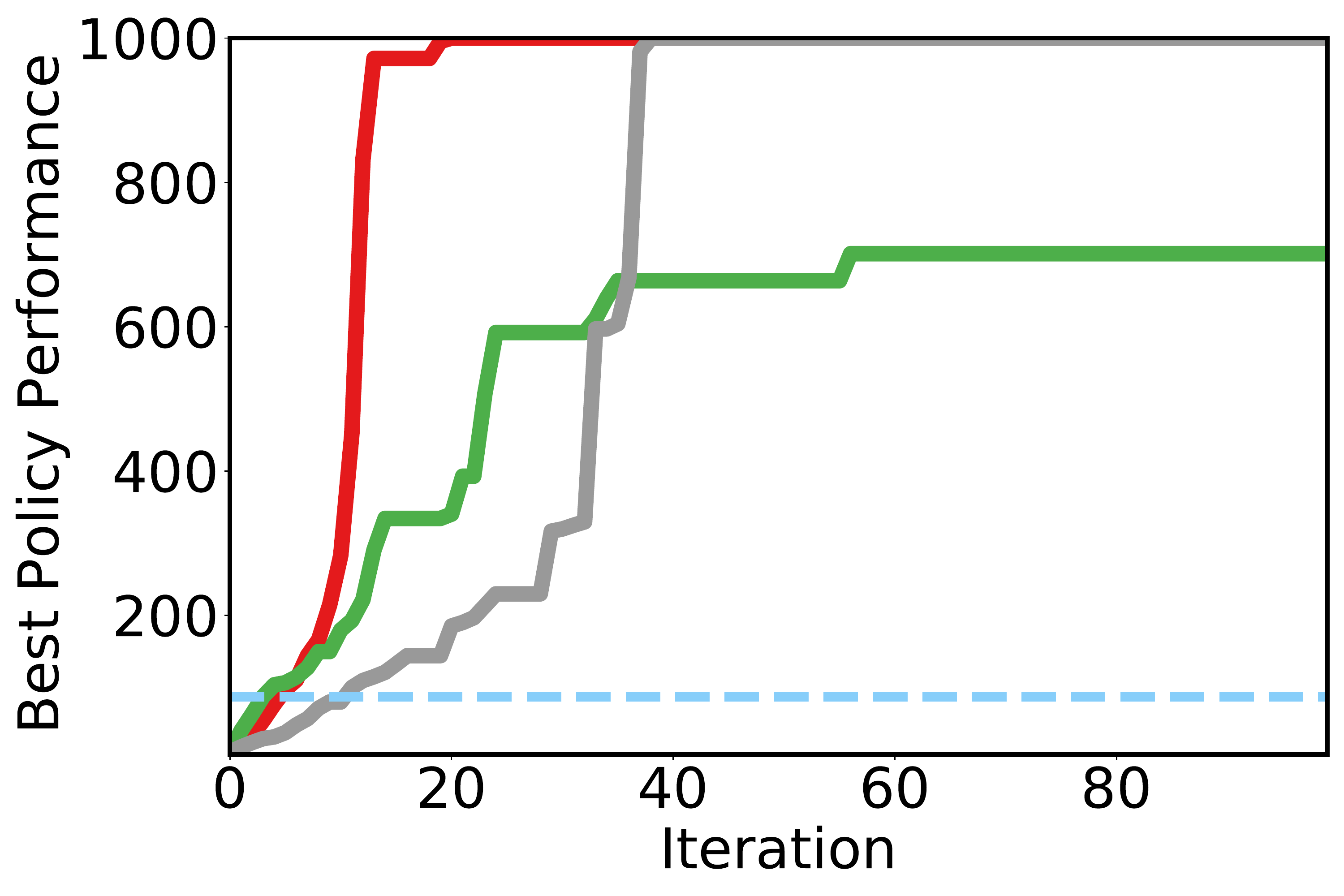}
		\caption{\small{CartPole (8 oracles)}}
		\label{fig:cp-all}
	\end{subfigure}
	\hspace{-10mm}
	\begin{subfigure}[b]{0.32\textwidth}
		\centering
		\includegraphics[width=0.8\textwidth]{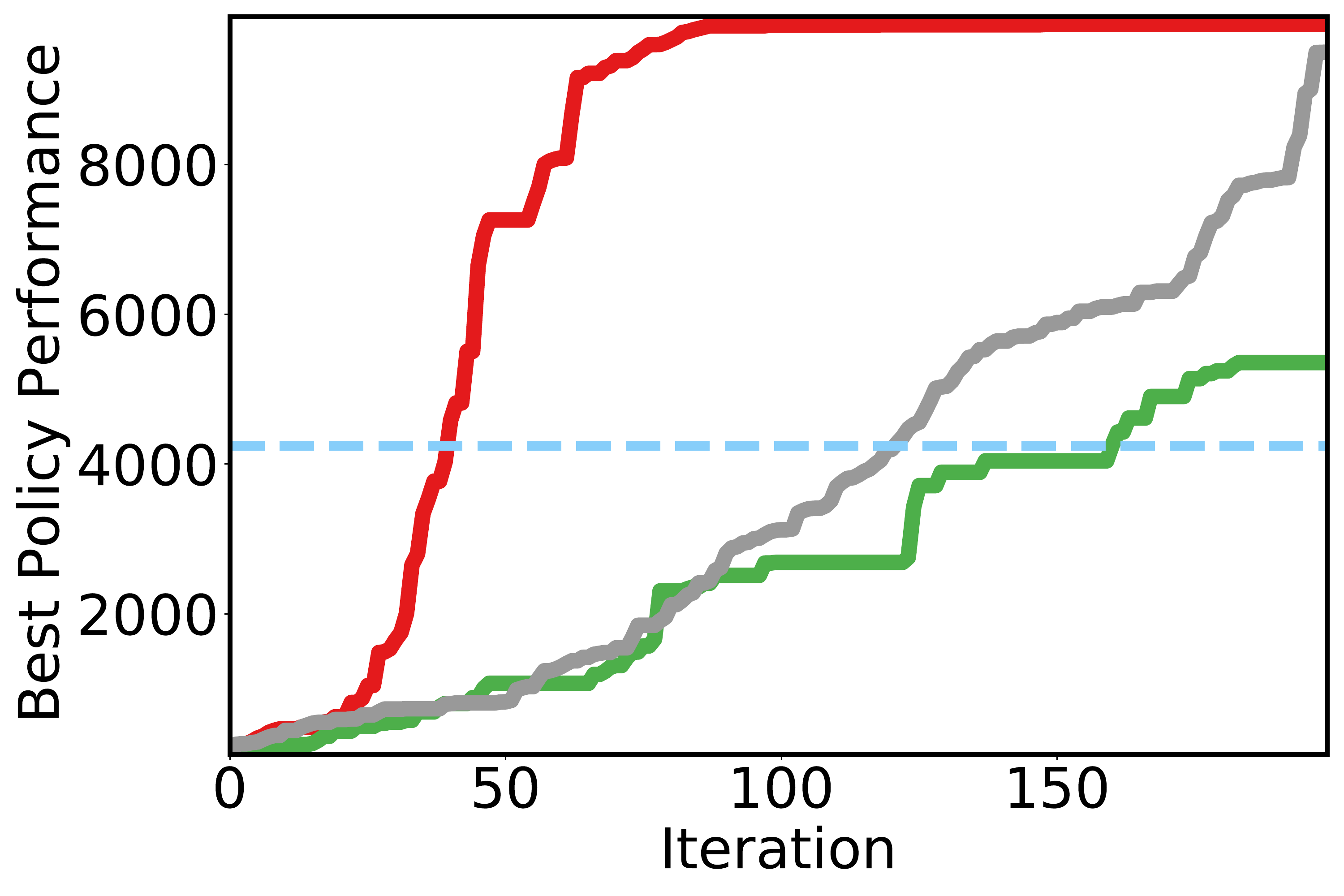}
		\caption{\small{DIP (4 oracles)}}
		\label{fig:dip-all}
	\end{subfigure}
	\hspace{-10mm}
	\begin{subfigure}[b]{0.32\textwidth}
		\centering
		\includegraphics[width=0.8\textwidth]{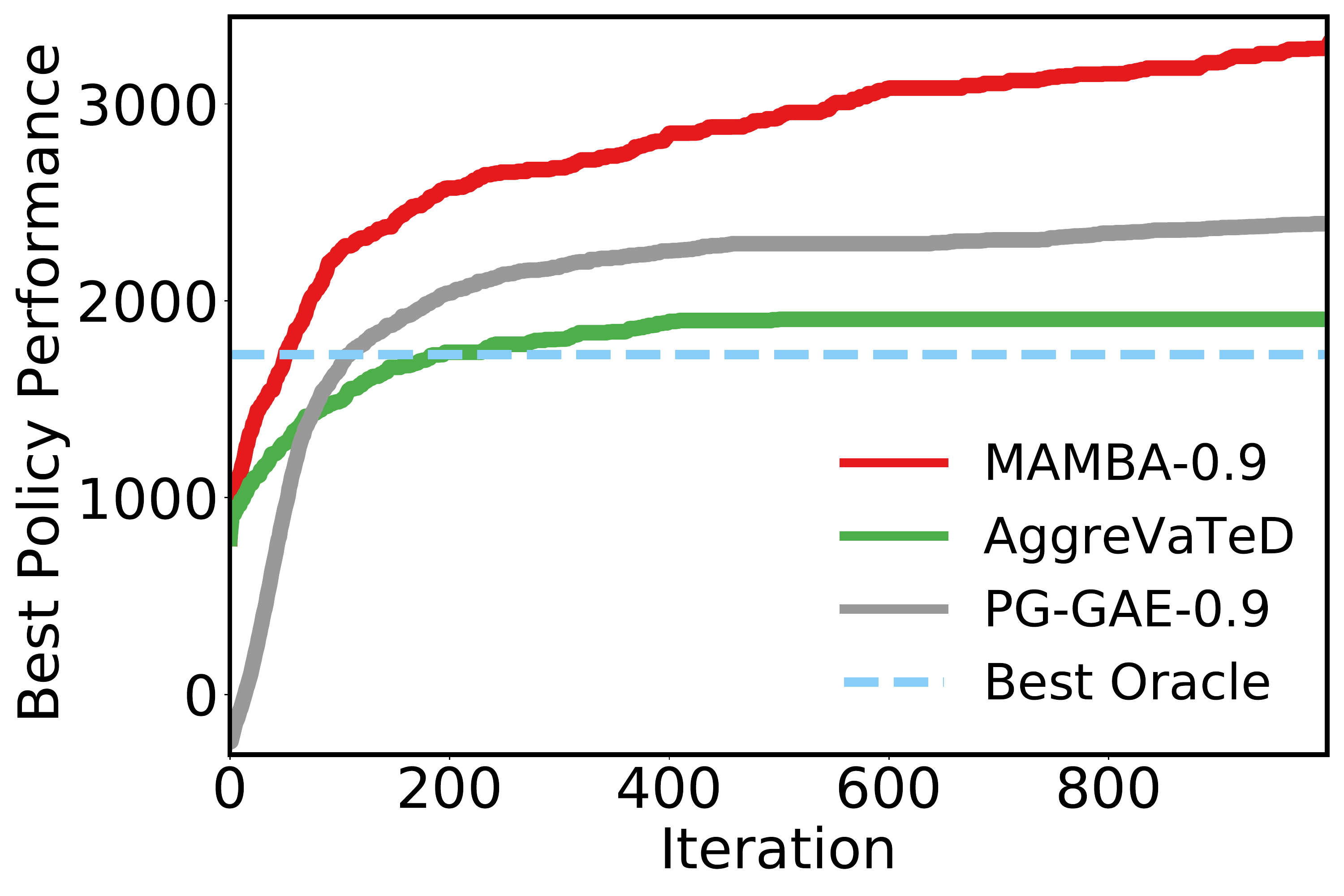}
		\caption{\small{Halfcheetah (2 oracles)}}
		\label{fig:halfcheetah-all}
	\end{subfigure}
	\hspace{-10mm}
	\begin{subfigure}[b]{0.32\textwidth}
		\centering
		\includegraphics[width=0.8\textwidth]{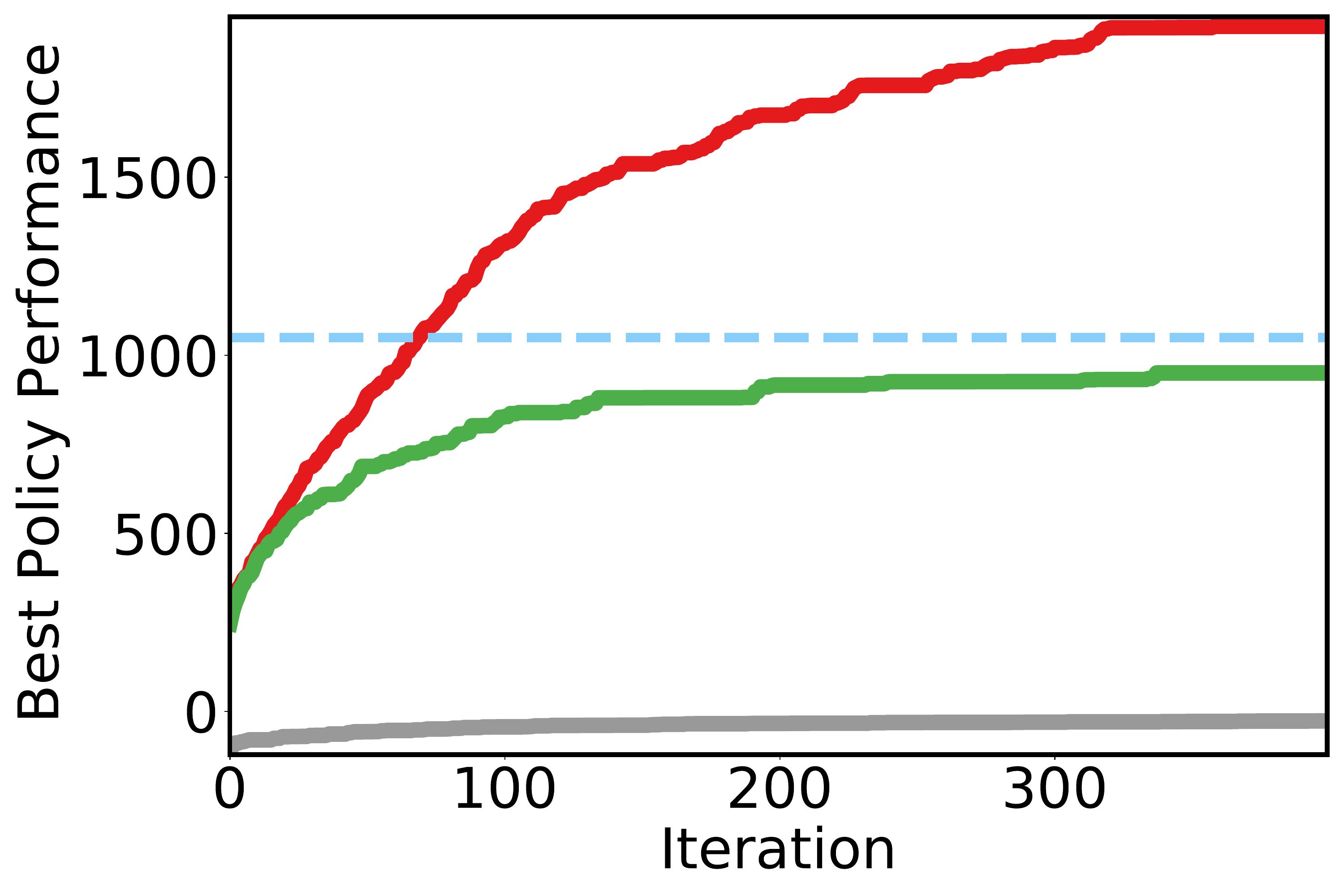}
		\caption{\small{Ant (2 oracles)}}
		\label{fig:ant-all}
	\end{subfigure}
	\hspace{-15mm}
	\caption{\small{Performance of the best policies returned by an RL algorithm (GAE policy gradient~\citep{schulman2015high}), the single-oracle IL algorithm (AggreVaTeD~\cite{sun2017deeply}), and our multi-oracle IL algorithm \alg.
	All oracle polices here are suboptimal and \aggd imitates the best one.
	In Halfcheetah and Ant, policies of IL algorihtms are initialized by behavior cloning with the best oracle policy.
	A curve shows an algorithm's median performance across 8 random seeds.
	Please see \cref{sec:experiments} for details.}}
	\vspace{-4mm}
	\label{fig:all-results}
\end{figure}

In this paper, we ask the question: how should an RL agent leverage domain knowledge encoded in \emph{more than one} (potentially suboptimal) oracle policies? We study this question in the aforementioned interactive IL setting.
Having multiple oracle policies is quite common in practice.
For instance, consider the problem of minimizing task processing delays via load-balancing a network of compute nodes. Existing systems and their simulators have a number of human-designed heuristic policies for load balancing that can serve as oracles~\citep{baek2019fog}. 
Likewise, in autonomous driving, available oracles can range from PID controllers to human drivers~\citep{li2018oil}.
In these examples, each oracle has its own strengths and can provide desirable behaviors for different situations.

Intuitively, because more oracle policies can provide more information about the problem domain, an RL agent should be able to learn a good policy faster than using a single oracle. However, in reality, the agent does not know the properties of each oracle. What it sees instead are conflicting demonstrations from the oracle policies.
Resolving this disagreement can be non-trivial, because there may not be a single oracle comprehensively outperforming the rest, and the quality of each oracle policy is unknown.
%
%
Recently, several IL and RL works have started to study this practically important class of scenarios. InfoGAIL~\citep{li2017infogail} conditions the learned policy on latent factors that motivate demonstrations of different oracles. AC-Teach~\citep{kurenkov2019acteach} models each oracle with a set of attributes and relies on a Bayesian approach to decide which action to take based on their demonstrations. OIL~\citep{li2018oil} tries to identify and follow the best oracle in a given situation. 
However, all existing approaches to IL from multiple oracles sidestep two fundamental questions: (a) What is a reasonable benchmark for policy performance is these settings, analogous to the single-oracle policy quality in conventional IL? (b) Is there a systematic way to stitch together several suboptimal oracles into a stronger baseline that we can further improve upon?

We provide answers to these questions, making the following contributions:
\begin{enumerate}[leftmargin=*]
\item We identify the state-wise maximum of oracle policies' values as a natural benchmark for learning from multiple oracles. We call it the \emph{\benchmark} and propose policy improvement from it as a natural strategy to combine these oracles together, creating a new policy that is uniformly better than all the oracles in every state. These insights establish the missing theoretical foundation for designing algorithms for IL with multiple oracles.
%

\item We propose a novel IL algorithm called \emph{\alg} (\algfull) to learn a policy that is competitive with the \benchmark by a reduction of policy optimization to online learning~\citep{ross2014reinforcement,cheng2018convergence}.
\alg is a first-order algorithm based on a new IL gradient estimator designed in the spirit of generalized advantage estimation (GAE)~\citep{schulman2015high} from the RL literature.
Like some prior works in IL, \alg interacts with the oracle in a roll-in/roll-out fashion~\cite{ross2014reinforcement,chang2015learning} and does not assume access to oracle actions.
%


\item We provide regret-based performance guarantees for \alg. In short, \alg generalizes a popular single-oracle IL algorithm \agg \citep{ross2014reinforcement,sun2017deeply} to learn from multiple oracles and to achieve larger improvements from suboptimal oracles.
Empirically, we evaluate \alg against the IL baseline (\aggd~\citep{sun2017deeply}) and direct RL (GAE policy gradient~\citep{schulman2015high}).
\cref{fig:all-results} highlights the experimental results, where \alg demonstrates the capability to bootstrap demonstrations from multiple weak oracles to significantly speed up policy optimization.

%
%
\end{enumerate}

\section{Background: Episodic Interactive Imitation Learning} \label{sec:episodic interactive IL}

\paragraph{Markov decision processes (MDPs)}
We consider finite-horizon MDPs with state space $\SS$ and action space $\AA$.
Let $T$, $d_0(s)$, $\PP(s'|s,a)$, and $r:\SS\times\AA\to[0,1]$ denote the problem horizon, the initial state distribution, the transition dynamics, and the reward function, respectively. \emph{We assume that $d_0$, $\PP$, and $r$ are fixed but unknown.}
Given a class of state-dependent policies $\Pi$, 
our goal is to find a policy $\pi\in\Pi$ that maximizes the $T$-step return with respect to the initial state distribution $d_0$:
\begin{align} \label{eq:return}
	\textstyle
	V^\pi(d_0) \coloneqq \E_{s_0\sim d_0} \E_{\xi_0\sim\rho^\pi|s_0}\left[ \sum_{t=0}^{T-1} r(s_t,a_t) \right],
\end{align}
where $\rho^\pi(\xi_t|s_t)$ denotes the distribution over trajectory $\xi_t = s_t, a_t, \dots, s_{T-1}, a_{T-1}$ generated by running policy $\pi$ starting from the state $s_t$ at time $t$ to the problem horizon.
To compactly write down non-stationary processes, we structure the state space $\SS$ as $\SS = \bar{\SS} \times \{0,T-1\}$, where $\bar{\SS}$ is some basic state space; thus, $\PP$ and $r$ can be non-stationary in $\bar{\SS}$. We allow $\bar{\SS}$ and $\AA$ to be either discrete or continuous. We use the subscript of $t$ to emphasize the time index. When writing $s_t$ we assume it is at time $t$,
and every transition from $s$ to $s'$ via $\PP(s'|s,a)$ increments the time index by $1$. 

\paragraph{State distributions and value functions}
We let $d_t^\pi$ stand for the state distribution at time $t$ induced by running policy $\pi$ starting from $d_0$ (i.e. $d_0^\pi = d_0$ for any $\pi$), and define the \emph{average state distribution} as $d^\pi \coloneqq \frac{1}{T}\sum_{t=0}^{T-1}d_t^\pi$. Sampling from $d^\pi$ returns $s_t$, where $t$ is uniformly distributed. Therefore, we can re-cast a policy's $T$-step return in \eqref{eq:return} as $V^\pi(d_0) = T \E_{s\sim d^\pi}\E_{a\sim\pi|s}[r(s,a)]$.
With a slight abuse of notation, we denote by $V^\pi:\SS\to\R$ as the value function of policy $\pi$, which satisfies $V^\pi(d_0) = \E_{s\sim d_0}[V^\pi(s)]$. 
Given a function $f:\SS\to\R$ such that $f(s_T) = 0$, we define the \emph{$Q$-function w.r.t. $f$} as $Q^f(s,a) \coloneqq r(s,a) + \E_{s'\sim\PP|s,a}[f(s')]$ and the \emph{advantage function w.r.t. $f$} as
\begin{align} \label{eq:advantage function}
A^f(s,a) \coloneqq Q^f(s,a) - f(s) = r(s,a) + \E_{s'\sim\PP|s,a}[f(s')] - f(s)
\end{align}
When $f = V^\pi$, we also write $A^{V^\pi} \eqqcolon A^\pi$ and $Q^{V^\pi} \eqqcolon Q^\pi$, which are the standard advantage and $Q$-functions of a policy $\pi$. We write {$f$'s advantage function under a policy $\pi$} as
$A^f(s,\pi) \coloneqq \E_{a\sim\pi|s} [A^f(s,a)]$ and similarly $Q^f(s,\pi)$ and $f(d) \coloneqq \E_{s\sim d}[f(s)]$ given a state distribution $d$. We refer to functions $f$ that index $Q$ or $A$ functions as \emph{baseline value functions}, because we aim to improve upon the value they provide in each state.
%
\begin{definition}\label{d:impf}
We say a baseline value function $f$ is \emph{improvable} w.r.t $\pi$ if $A^f(s,\pi)\geq 0$, $\forall s\in\SS$.
\end{definition}

\paragraph{Policy optimization with multiple oracle policies}
The setup above describes a generic episodic RL problem, where the agent faces the need to perform strategic exploration and long-term credit assignment. A common approach to circumvent the exploration challenge in practice is by leveraging an oracle policy.
%
In this paper, we assume access to \emph{multiple} (potentially suboptimal) oracle policies during training, and leverage episodic interactive IL to improve upon them.
We suppose that the learner (i.e. the agent) has access to a set of 
oracle policies $\PiE=\{\pi^k\}_{k\in[K]}$. During training, the learner can interact with the oracles in a roll-in-roll-out (RIRO) paradigm to collect demonstrations.
In each episode, the learner starts from an initial state sampled from $d_0$ and runs its policy $\pi\in\Pi$ up to a switching time $\ts\in[0,T-1]$; then the learner asks an oracle policy $\pie\in\PiE$ to take over and finish the trajectory. 
At the end, the learner records the entire trajectory, including reward information. 
%
Note that we do not assume that oracle actions are observed. In addition, as sampled rewards are available here, the learner can potentially improve upon the oracle policies.

\section{A Conceptual Framework for Learning from Multiple Oracles} \label{sec:conceptual algorithm}

In this paper, we focus on the scenario where the set $\PiE=\{\pi^k\}_{k\in[K]}$ contains more than one oracle policy.
Having multiple oracle policies offers an opportunity to gain more information about the problem domain. Each oracle may be good in different situations, so the learner can query suitable oracles at different states for guidance.
But how exactly can we leverage the information from multiple oracles to learn more efficiently than from any single one of them?

\paragraph{Some natural baselines}

One approach for leveraging multiple oracles is to combine them into a single oracle, such as by using a fixed weighted mixture~\citep{jacobs1991adaptive}, or  multiplying their action probabilities in each state~\citep{hinton2002training}.\footnote{These approaches are proposed for supervised learning and not specifically IL.} 
But the former can be quite bad even if only one oracle is bad, and the latter fails to combine two deterministic oracles. Another alternative is to evaluate each oracle and run a single-oracle IL algorithm with the one with the highest return.
However, in Appendix~\ref{sec:compare} we show an example where two oracles have identical suboptimal returns, but switching between them 
results in the optimal behavior.
%

We ask, \emph{is there a general principle for combining multiple oracles?} If we seek to switch among multiple oracles, how should switching points be chosen? Can we learn a rule for doing so reliably? In this section, we show that the issues mentioned above can be addressed by performing policy improvement upon the state-wise maximum over the oracles' values, i.e. the \benchmark.
We describe two conceptual algorithms: one is based on the perfect knowledge of the MDP and the oracles' value functions, while the other builds on the first one using online learning to handle an unknown MDP and oracle values in the interactive IL setup.
The insights gained from these two conceptual algorithms will be used to design their practical variation, \alg, in \cref{sec:practical algorithm}.


\vspace{-1mm}
\subsection{Max-aggregation with Policy Improvement} \label{sec:IL as policy improvement}
\vspace{-1mm}
To illustrate the key idea, first let us suppose for a moment that perfect knowledge of the MDP and the oracles' value functions is available.
In this idealized setting, the IL problem can be rephrased as follows: find a policy that is at least as good as all the oracle policies, and do so in a way whose computational complexity is \emph{independent} of the problem horizon.
The restriction on the complexity is important; otherwise we can just use the MDP knowledge to solve for the optimal policy.

How do we solve this idealized IL problem? When $\PiE$ contains only a single oracle, which we denote as $\pie$, a natural solution is the policy given by one-step policy improvement from $\pie$, i.e. the policy $\tpie$ that acts according to $\argmax_{a\in\AA} r(s,a) + \E_{s'\sim\PP|s,a}[V^\pie(s')]$.
It is well known that this policy $\tpie$ is uniformly better than $\pie$ for all the states, i.e. $V^{\tpie}(s) \geq V^{\pi^e}(s)$ (cf. ~\citep{puterman2014markov} and \cref{cr:improvable baseline} below). 
However, this basic approach no longer applies when $\PiE$ contains multiple oracles, and natural attempts to invoke a single-oracle algorithm do not work in general as discussed earlier.

\vspace{-2mm}
\paragraph{A max following approach} A simple way to remedy the failure mode of uniformly mixing the oracle policies is to take a non-uniform mixture that is aware of the quality of each oracle. \emph{If} we have the value function of each oracle, we have a natural measure of their quality. 
With this intuition, for the $k$-th oracle policy $\pi^k\in\PiE$, let us write $V^k \coloneqq V^{\pi^k}$.
A natural candidate policy based on this idea is the greedy policy that follows the best oracle in any given state: 
\begin{align} \label{eq:k_s and pi_bullet}
	\textstyle
	\pi^\bullet(a|s) \coloneqq \pi^{k_s}(a|s),
	\qquad \text{where} \quad
	k_s \coloneqq \argmax_{k\in[K]} V^k(s)
\end{align}

Imitating a benchmark similar to $\pi^\bullet$ was recently proposed as a heuristic for IL with multiple oracles in \citep{li2018oil}.
Our first contribution is a theoretical result showing that the intuition behind this heuristic holds mathematically: $\pi^\bullet$ indeed satisfies $V^{\pi^\bullet}(s) \geq \max_{k\in[K]} V^k(s)$.
To show this, we construct a helper corollary based on the useful Performance Difference Lemma (\cref{lm:pdl})\footnote{Lemma~\ref{lm:pdl} is an adaptation of the standard Performance Difference Lemma to using $f$ that is not necessarily the value function of any policy. We provide proofs of Lemma~\ref{lm:pdl} and Corollary~\ref{cr:improvable baseline} in \cref{app:proofs} for completeness.}. 
\begin{restatable}{lemma}{PDL}\label{lm:pdl}
	{\normalfont \citep{ng1999policy,kakade2002approximately}}
	Let $f:\SS\to\R$ be such that $f(s_T)=0$. For any MDP and policy $\pi$,
	\begin{align} \label{eq:pdl}
\textstyle	V^\pi(d_0) - f(d_0) = T \E_{ s\sim d^\pi}  [ A^f(s,\pi) ].
	\end{align}
\end{restatable}

\begin{restatable}{corollary}{ImprovableBaselinValueFunction} \label{cr:improvable baseline}
	If $f$ is improvable w.r.t. $\pi$ 
	, then $V^{\pi}(s)\geq f(s)$, $\forall s\in\SS$.
\end{restatable}

Corollary~\ref{cr:improvable baseline} implies that a policy $\pi$ has a better performance than all the oracles in $\PiE$ if there is a baseline value function $f$ that is improvable w.r.t. $\pi$ (i.e. $A^f(s,\pi)\geq 0$) and dominates the value functions of all oracle policies everywhere (i.e. $f(s)\geq V^k(s)$, $\forall k\in[K], s\in\SS$).
%

This observation suggests a natural value baseline for studying IL with multiple oracles:
\begin{align} \label{eq:max V}
	\textstyle
	\fmax(s) \coloneqq \max_{k\in[K]} V^k(s).
\end{align}

Below we prove that this \benchmark $\fmax$ in \eqref{eq:max V} is improvable with respect to $\pi^\bullet$.
Together with \cref{cr:improvable baseline}, 
this result implies that $\pi^\bullet$ is a valid solution to the idealized IL problem with multiple oracles.
We write the advantage $A^{\fmax}$ with respect to $\fmax$ in short as $\Amax$.

\begin{restatable}{proposition}{VmaxImprovable}	\label{pr:improvable_bullet}
	$\fmax$ in \eqref{eq:max V} is improvable with respect to $\pi^\bullet$, i.e. $\Amax(s,\pi^\bullet) \geq 0$.
\end{restatable}

\vspace*{-1mm}
\paragraph{Degeneracy of max-following}
 The policy $\pi^\bullet$ above, however, suffers from a degenerate case: when there is one oracle in $\PiE$ that is uniformly better than all the other oracles (say $\pie$), we have $\pi^\bullet=\pie$, whereas we know already $\tpie$ is a uniformly better policy that we can construct using the same information. In this extreme case, in the standard IL setting with one oracle, $\pi^\bullet$ would simply return the suboptimal oracle.

 \vspace*{-1mm}
\paragraph{A max-\emph{aggregation} approach} Having noticed the failure mode of $\pi^\bullet$, we obtain a natural fix by combining the same value baseline~\eqref{eq:max V} with the standard policy improvement operator. We define
\begin{align} \label{eq:policy improving from f}
	\textstyle
	\pimax(a|s) \coloneqq \delta_{a=a_s}, \qquad\text{where} \qquad a_s \coloneqq \argmax_{a \in \AA} \Amax(s,a),
\end{align}
and $\delta$ denotes the delta distribution. In contrast to $\pi^\bullet$, $\pimax$ looks one step ahead and takes the action with the largest advantage under $\fmax$. Note that this is \emph{not} necessarily the same as following highest-value oracle in the current state.
Since $\pimax$ satisfies $\Amax(s,\pimax)\geq \Amax(s,\pi^\bullet)\geq 0$, by \cref{cr:improvable baseline}, $\pimax$ is also a valid solution to the idealized IL problem with multiple oracles.
%
The use of $\pimax$ is novel in IL to our knowledge, although $\pimax$ is called Multiple Path Construction Algorithm in controls~\citep[Chapter 6.4.2]{bertsekas1995dynamic}. \cref{cr:improvable baseline} provides a simple 
proof of why $\pimax$ works.

In general, $V^{\pimax}(s)$ and $V^{\pi^\bullet}(s)$ are not comparable. But, crucially, in the degenerate case above we see that $\pimax$ reduces to $\tpie$ and therefore would perform better than $\pi^\bullet$, although in Appendix~\ref{sec:compare} we also show an MDP where $\pi^\bullet$ is better.
 Intuitively, this happens as $\fmax$ implicitly assumes using a single oracle for the remaining steps, but both $\pi^\bullet$ and $\pimax$ re-optimize their oracle choice at every step, so their relative quality can be arbitrary.
While both $\pi^\bullet$ and $\pimax$ improve upon all the oracles, in this paper, we choose $\pimax$ as our imitation benchmark, because it is consistent with prior works in the single-oracle case and does not require observing the oracles' actions in IL, unlike $\pi^\bullet$.

\blue{Finally, we remark that \citet{barreto2020fast} recently proposed a generalized policy improvement operator, $\argmax_{a \in \AA} \max_{k\in[K]} Q^{\pi^k}(s,a)$, for multi-task RL problems. It yields a policy that is similar to $\pimax$ in \eqref{eq:policy improving from f} and is uniformly better than all the oracles. However, despite similarities, these two policies are overall different. In particular, $\max_{k\in[K]}  Q^{\pi^k}(s,a) - \fmax(s) \leq \Amax(s,a)$, so \eqref{eq:policy improving from f} aims to improve upon a stronger baseline. And, as we will see in the next section, \eqref{eq:policy improving from f} allows for a geometric weighted generalization, which can lead to an even larger policy improvement.}

\subsection{Max-aggregation with Online Learning} \label{sec:IL as online learning}

The previous section shows that improving from the \benchmark $\fmax$ in \eqref{eq:max V} is a key to reconciling the conflicts between oracle policies. However doing so requires the knowledge of the MDP and the oracles' value functions, which are unavailable in the episodic interactive IL setting.

To compete with $\fmax$ without the above assumption, we design an IL algorithm via a reduction to online learning~\citep{zinkevich2003online}, a technique used in many prior works in the single-oracle setting~\citep{ross2011dagger,ross2014reinforcement,chang2015learning,sun2017deeply,cheng2018loki,cheng2019accelerating}.
%
To highlight the main idea, at first we still assume that the oracles' value functions are given, but only the MDP is unknown. Then we show how to handle unknown value functions.
For clarity, we use the subscript in $\pi_n$ to index the learner's policy in $\Pi$ generated in the $n$-th round  of online learning, while using the superscript in $\pi^k$ to index the oracle policy in $\PiE$.

\vspace{-1mm}
\paragraph{Ideal setting with known values} If the MDP dynamics and rewards are unknown, we can treat $d^{\pi_n}$ as the adversary in online learning and define the online loss in the $n$-th round as
\begin{align} \label{eq:online loss}
\textstyle	\ell_n(\pi) \coloneqq -T\E_{s \sim d^{\pi_n}} \left[ \Amax(s,\pi) \right].
\end{align}
By \cref{lm:pdl}, making $\ell_n(\pi_n)$ small ensures that $V^{\pi_n}(d_0)$ is not much worse than $\fmax(d_0)$. 
Formally, averaging this argument over $N$ rounds of online learning, we obtain
\begin{align} \label{eq:regret equality}
	\textstyle
	\frac{1}{N}\sum_{n\in[N]} V^{\pi_n}(d_0) = \fmax(d_0) + \Delta_N - \epsilon_N(\Pi) - \frac{\regret_N}{N},
\end{align}
where we define $\regret_N \coloneqq \textstyle \sum_{n=1}^{N} \ell_n(\pi_n) -  \min_{\pi\in\Pi}  \sum_{n=1}^{N}\ell_n(\pi)$,
\begin{gather}
	\hspace{-1mm} \Delta_N \coloneqq \textstyle \frac{-1}{N}\sum_{n=1}^{N} \ell_n(\pimax),  \label{eq:improvement amount}%
	\text{ and }
	\epsilon_N(\Pi) \coloneqq \textstyle \min_{\pi\in\Pi}  \frac{1}{N} \left( \sum_{n=1}^{N}\ell_n(\pi) - \sum_{n=1}^{N} \ell_n(\pimax) \right).
\end{gather}

In \eqref{eq:regret equality}, the regret characterizes the learning speed of an online algorithm, while $\epsilon_N(\Pi)$ captures the quality of the policy class.
If $\pimax\in\Pi$, then $\epsilon_N(\Pi)=0$; otherwise, $\epsilon_N(\Pi)\geq 0$. 
Furthermore, we have $\Delta_N\geq 0$, because 
we showed $\Amax(s,\pimax)\geq0$ in \cref{sec:IL as policy improvement}.
Thus, when $\pimax \in \Pi$, running a no-regret algorithm (i.e. an algorithm such that $\regret_N=o(N)$) to solve this online learning problem will guarantee producing a policy whose performance at least $\E_{s\sim d_0}[\max_{k\in[K]} V^k(s) ] + \Delta_N + o(1)$ after $N$ rounds.

The above reduction in \eqref{eq:online loss} generalizes AggreVaTE~\citep{ross2014reinforcement} from using $f=V^{\pie}$ in $A^f$ for define the online loss for the single oracle case to using $f=\fmax$ instead, which is also applicable to multiple oracles.
{When an oracle in $\PiE$ dominates the others for all the states, \eqref{eq:online loss} is the same as the online loss in AggreVaTE.}

\paragraph{Effect of approximate oracle values}
Recall that for the above derivation we assumed oracle policy values (and hence $\fmax$) are given. In practice, $\fmax$ is unavailable and needs to be approximated by some $\fmaxh$. Let $\Ah$ denote the shorthand of $A^{\fmaxh}$.
We can treat the approximation error as bias and variance in the feedback signal, such as the sample estimate of the gradient below:
\begin{align} \label{eq:aggrevate online gradient}
\nabla \what{\ell}_n(\pi_n)  &= -T \E_{s \sim d^{\pi_n}} \E_{a \sim \pi|s}\left[ \nabla \log \pi(a|s) \Ah(s,a)  \right] |_{\pi=\pi_n},
\end{align}
where $\nabla$ is with respect to the policy. We summarize the approximation effects as a meta theorem, \blue{where generally $\beta$ and $\nu$ increase as the problem horizon increases.}
\begin{restatable}{theorem}{MetaTheorem}  \label{th:meta performance theroem}
	Suppose a first-order online algorithm that satisfies $\E[\regret_N] \leq O( \beta N + \sqrt{\nu N})$ is adopted, where $\beta$ and $\nu$ are the bias and the variance of the gradient estimates, respectively. Then
	\begin{align}
		 \E[ \max_{n\in[N]} V^{\pi_n}(d_0)] \geq \E_{s\sim d_0}[\max_{k\in[K]} V^k(s) ] + \E[\Delta_N - \epsilon_N(\Pi)] - O(\beta + \sqrt{\nu}N^{-1/2})
	\end{align}
	where the expectation is over the randomness in feedback and the online algorithm.
\end{restatable}

\cref{th:meta performance theroem} describes considerations of using $\fmaxh$ in place of $\fmax$.
For the single-oracle case, $\fmaxh$ can be an unbiased Monte-Carlo estimate (i.e. $\nabla \what{\ell}_n = \nabla \ell_n$);
but a sample estimate of such $\nabla \what{\ell}_n(\pi_n)$ suffers from a variance that is $T$-times larger than the Monte-Carlo estimate of policy gradient due to the restriction of RIRO data collection protocol\footnote{As $V^{k_s}$ is not the value  of $\pi_n$ but $\pi^{k_s}$, computing an unbiased estimate of $\nabla\ell_n(\pi_n)$ requires uniformly selecting the switching time $\ts\in\{0,\dots, T-1\}$ in the RIRO setting, which amplifies the variance by $O(T)$. 
}.
%
Alternatively, one can use function approximators~\citep{sun2017deeply} as $\fmaxh$ to shift the variance from the gradient estimate to learning $\fmaxh$.
In this case, \eqref{eq:aggrevate online gradient} becomes akin to the actor-critic policy gradient. 
But when the accuracy of the value estimate $\fmaxh$ is bad,
 the bias in \eqref{eq:aggrevate online gradient} can also compromise the policy learning.
%

For the multi-oracle case,
unbiased Monte-Carlo estimates of $\fmax$ are infeasible, because $\fmax(s)=V^{k_s}(s)$ and $\pi^{k_s}$ is unknown (i.e. we do not know the best oracle policy at state $s$).
Therefore, $\fmaxh$ in \eqref{eq:aggrevate online gradient} must be a function approximator. But, due to the max operator in $\fmax$, learning $\fmaxh$ becomes challenging as all the oracles' value functions need to ba approximated uniformly well.
%


\begin{algorithm}[t]
	{\small
		\caption{\alg for IL with multiple oracles}\label{alg:mamba}
		\begin{algorithmic} [1]
			\renewcommand{\algorithmicensure}{\textbf{Input:}}
			\renewcommand{\algorithmicrequire}{\textbf{Output:}}
			\ENSURE Initial learner policy $\pi_1$, oracle polices $\{\pi^k\}_{k\in[K]}$, function approximators $\{\what{V}^k\}_{k\in[K]}$.
			\REQUIRE The best policy in $\{\pi_1, \dots, \pi_N\}$.
			\FOR {$n = 1\dots N-1$}
			\STATE Uniformly sample $\ts\in[T-1]$ and $k\in[K]$.
			\STATE Roll-in $\pi_n$ up to $\ts$ and switch to $\pi^k$ to complete the remaining trajectory to collect data $\DD_n$.
			\STATE Update $\what{V}^k$ using $\DD_n$ (e.g. using Monte-Carlo estimates).
			\STATE Roll-in $\pi_n$ for the full $T$-horizon to collect data $\DD_n'$.
			\STATE Compute the sample estimate $g_n$ of $\nabla \widehat{\ell}_n (\pi_n; \lambda)$ in \eqref{eq:lambda weighted gradient estimate} using $\DD_n'$ and $\fmaxh(s) = \max_{k\in[K]} \what{V}^k(s)$.
			\STATE Update $\pi_n$ to $\pi_{n+1}$ by giving $g_n$ to a first-order online learning algorithm (e.g. mirror descent).
			\ENDFOR
		\end{algorithmic}
	}
\end{algorithm}

\section{\alg: Multi-Step Policy Improvement upon Multiple Oracles} \label{sec:practical algorithm}

We propose \alg as a practical realization of the first-order reduction idea in \cref{th:meta performance theroem} (shown in \cref{{alg:mamba}}).
As discussed, obtaining a good sample estimate of~\eqref{eq:aggrevate online gradient} is nontrivial. 
As a workaround, we will design \alg  based on an alternate online loss $\ell_n(\pi;\lambda)$
that shares the same property as $\ell_n(\pi)$ in \eqref{eq:online loss} but has a gradient expression with tunable bias-variance trade-off. \blue{More importantly, we prove that adapting $\lambda$ leads to a continuous transition from one-step policy improvement to solving a full-scale RL problem. Consequently, fine-tuning this $\lambda$ ``knob" can additionally trade off the limited performance gain of performing one-step policy improvement and the sample inefficiency of solving a full-scale RL problem in order to obtain the best finite-sample policy performance in learning.}

\subsection{Trade-off between One-Step Policy Improvement and Full RL} 
\blue{The alternate online loss $\ell_n(\pi;\lambda)$ is based on the geoemtric weighting technique commonly used in RL algorithms (such as TD-$\lambda$~\cite{sutton1988learning}).} Specifically, for $\lambda\in[0,1]$, we propose to define the new online loss in the $n$-th round as
\begin{align} \label{eq:online loss with lambda weights}
	\ell_n(\pi;\lambda) \coloneqq  - (1-\lambda) T \E_{s \sim d^{\pi_n}}\left[ \Amaxpi_\lambda(s,\pi) \right]  - \lambda \E_{s \sim d_0} \left[ \Amaxpi_\lambda(s,\pi) \right]
\end{align}
where we define a $\lambda$-weighted advantage
\begin{align} \label{eq:lambda advantage}
	\textstyle \Amaxpi_\lambda(s,a)
	\coloneqq (1-\lambda) \sum_{i=0}^\infty \lambda^i  \Amaxpi_{(i)}(s,a)
\end{align}
by combining various $i$-step advantages:
\[
\Amaxpi_{(i)}(s_t,a_t) \coloneqq \E_{\xi_t \sim \rho^\pi|s_t} [r(s_t, a_t) + \dots + r(s_{t+i}, a_{t+i}) +  \fmax(s_{t+i+1}) ]  - \fmax(s_t).\]

\blue{The hyperparameter $\lambda$ in \eqref{eq:online loss with lambda weights} controls the attainable policy performance and the sample complexity of learning with \eqref{eq:online loss with lambda weights}. To examine this, let us first note some identities due to the geometric weighting: $\Amaxpi_0 = \Amaxpi_{(0)} = \Amax$ and $\Amaxpi_{1} = \Amaxpi_{(\infty)} = Q^\pi - \fmax$.
Therefore, when $\lambda=1$, the online loss in \eqref{eq:online loss with lambda weights} becomes the original RL objective for every round (one can show that $\ell_n(\pi;1) = \fmax(d_0) - V^\pi(d_0)$ for all $n$), where policy learning aims to discover the optimal policy but is sample inefficient. On the other hand, when $\lambda=0$ (we define $0^0=1$), \eqref{eq:online loss with lambda weights} reduces back to \eqref{eq:online loss} (i.e., $\ell_n(\pi;0)  = \ell_n(\pi)$), where policy learning is about one-step policy improvement from the max-aggregated baseline (namely IL). Although doing so does not necessarily yield the optimal policy, it can be done in a sample-efficient manner. 
By tuning $\lambda$ we can trade off performance bias and sample complexity to achieve the best performance given a \emph{finite} number of samples.
}


We can make a connection of \eqref{eq:online loss with lambda weights} to the IL literature. When there is a single oracle (i.e. $\fmax= V^\pie$), we can interpret the online loss \eqref{eq:online loss with lambda weights} in terms of known IL objectives.
In \eqref{eq:online loss with lambda weights}, the first term is the $\lambda$-weighted version of the AggreVaTe loss~\citep{ross2014reinforcement}, and the second term is the $\lambda$-weighted version of the THOR loss~\citep{sun2018truncated}. But none of these prior IL algorithms makes use of the $\lambda$-weighted advantages.

\blue{Lastly, let us formally establish the relationship between \eqref{eq:online loss with lambda weights} and the RL objective by a generalization of the Performance Difference Lemma (\cref{lm:pdl}) to take into account geometric weighting.
\begin{restatable}
	{lemma}{LambdaWeightedPDL} \label{pr:lambda weighed PDL}
	For any policy $\pi$, any $\lambda\in[0,1]$, and any baseline value function $f:\SS\to\R$,
	\begin{align}
		\textstyle
		V^{\pi}(d_0) - f(d_0) =  (1-\lambda) T \E_{s \sim d^{\pi}}\left[ A_\lambda^{f,\pi}(s,\pi) \right]  + \lambda \E_{s \sim d_0} \left[ A^{f,\pi}_\lambda(s,\pi) \right],
	\end{align}
	where $A^{f,\pi}_\lambda$ is defined like $\Amaxpi_\lambda$ but with a general $f$ instead of $\fmax$.
\end{restatable}
In other words, the new online loss function $\ell_n(\pi;\lambda)$ satisfies an equality similar to \cref{lm:pdl}, which $\ell_n(\pi)$ relies on to establish the reduction in \eqref{eq:regret equality}. Now, with the proper generalization given by \cref{pr:lambda weighed PDL}, we can formally justify  learning with $\ell_n (\pi_n; \lambda)$.
\begin{restatable}{theorem}{MambaPerformance} \label{th:mamba performance}
Performing no-regret online learning w.r.t. \eqref{eq:online loss with lambda weights} has the guarantee in \cref{th:meta performance theroem}, except now 
$\epsilon_N(\Pi)$ can be negative when $\lambda>0$.
\end{restatable}

\cref{th:mamba performance} shows that learning with $\ell_n(\pi;\lambda)$  has a similar performance guarantee to using $\ell_n(\pi)$  in \cref{th:meta performance theroem}, but with one important exception: now $\epsilon_N(\Pi)$ can be negative (which is in our favor), because $\pimax$ may not be the best policy for the multi-step advantage in $\ell_n(\pi;\lambda)$ when $\lambda>0$. \blue{This again can be seen from the fact that optimizing $\ell_n(\pi;1)$ is equivalent to direct RL.} As a result, when using larger $\lambda$ in \alg, larger improvements can be made from the oracle polices as we move from one-step policy improvement toward full-fledged RL.
However, looking multiple steps ahead with high $\lambda$ would also increase the feedback variance $\nu$ in \cref{th:meta performance theroem}, which results in a slower convergence rate. In practice, $\lambda$ needs to be tuned to achieve the best finite-sample performance.
}
%

\vspace{-1mm}
\subsection{Simple Gradient Estimation}
\vspace{-1.5mm}

\blue{While the online loss $\ell_n(\pi;\lambda)$ in \eqref{eq:online loss with lambda weights} appears complicated, interestingly, its gradient $\nabla \ell_n(\pi;\lambda)$ has a very clean expression, as we prove below.
\begin{restatable}{lemma}{LambdaWeightedOnlineGradient} \label{pr:lambda weighed online gradient}
	For any $\lambda\in[0,1]$, any baseline value function $f:\SS\to\R$, and any policy $\pi$, the following holds:
	\begin{align}
		\textstyle
		&h(\pi;\lambda) \coloneqq (1-\lambda) T \E_{s \sim d^{\mu}}\left[ A_\lambda^{f,\pi}(s,\pi) \right]  + \lambda \E_{s \sim d_0} \left[ A^{f,\pi}_\lambda(s,\pi) \right]\\
		&\nabla h(\pi;\lambda) |_{\mu=\pi} =  T \E_{s \sim d^{\pi}} \E_{a \sim \pi|s}[ \nabla \log\pi(a|s) A_\lambda^{f,\pi} (s,a)  ],
	\end{align}
	where $d^\mu$ denotes the average state distribution of a policy $\mu$.
\end{restatable}

Using \cref{pr:lambda weighed online gradient}, we design a gradient estimator for $\ell_n(\pi;\lambda)$ by approximating $\fmax$ in $\nabla \ell_n(\pi;\lambda)$ with a function approximator $\fmaxh$:
\begin{align} \label{eq:lambda weighted gradient estimate}
\nabla \widehat{\ell}_n (\pi_n; \lambda)  &=  -T\E_{s \sim d^{\pi_n}} \E_{a \sim \pi|s}[ \nabla \log\pi(a|s) \Ah_\lambda^{\pi}(s,a)  ] |_{\pi=\pi_n},
\end{align}
where $\Ah_\lambda$ is defined by replacing $\fmax$ in~\eqref{eq:lambda advantage} with $\fmaxh$. Implementation-wise, the infinite sum in $\Ah_\lambda$ can be computed by standard dynamic programming.
\begin{restatable}{lemma}{LambdaAdvantagePracticalFrom} \label{lm:lambda advantage (practical form)}
	Define $ \Ah(s,a) \coloneqq r(s,a) + \E_{s'|s,a}[\fmaxh(s')] - \fmaxh(s)$. It holds that for all $\lambda \in[0,1]$,
	\begin{align} \label{eq:lambda advantage (practical form)}
		\textstyle
		\Ah_\lambda^\pi(s_t,a_t) = \E_{\xi_t \sim \rho^\pi|s_t} \left[ \sum_{\tau=t}^{T-1} \lambda^{\tau-t} \Ah(a_\tau, s_\tau) \right]
	\end{align}
\end{restatable}
Therefore, given $\fmaxh$ as a function approximator, an unbiased estimate of \eqref{eq:lambda weighted gradient estimate} can be obtained by sampling trajectories from $\pi$ directly, without any RIRO interleaving with the oracles. This gradient estimator is reminiscent of GAE for policy gradient~\citep{schulman2015high}, but translated to IL.
}

\blue{The gradient expression in \eqref{eq:lambda weighted gradient estimate} reveals that $\lambda$ also plays a role in terms of bias-variance trade-off, in addition to the transition from one-step improvement to full-scale RL discussed above. Comparing \eqref{eq:aggrevate online gradient} and \eqref{eq:lambda weighted gradient estimate}, we see that $\nabla \widehat{\ell}_n (\pi; \lambda)$ in~\eqref{eq:lambda weighted gradient estimate} replaces $\Ah$ in $\nabla \widehat{\ell}_n (\pi)$ in \eqref{eq:aggrevate online gradient} with the $\lambda$-weighted version $\Ah_\lambda^\pi$. 
Controlling $\lambda$ regulates the effects of the approximation error $\fmaxh-\fmax$ on the difference $\Ah_\lambda^{\pi}-{\Amaxpi_\lambda}$, which in turn determines the gradient bias $\nabla \widehat{\ell}_n (\pi; \lambda) - \nabla \ell_n (\pi; \lambda)$ (namely $\beta$ in \cref{th:meta performance theroem}). This mechanism is similar to the properties of the GAE policy gradient~\citep{schulman2015high}.
We recover
 the gradient in \eqref{eq:aggrevate online gradient} when $\lambda =0$ and the policy gradient when $\lambda=1$.
 }
%


Finally, we emphasize that $\nabla \widehat{\ell}_n (\pi; \lambda)$ in \eqref{eq:lambda weighted gradient estimate} is \emph{not} an approximation of $\nabla \widehat{\ell}_n (\pi)$ in \eqref{eq:aggrevate online gradient}, because generally $\nabla \ell_n (\pi_n; \lambda) \neq \nabla \ell_n (\pi_n)$ even when $\fmaxh=\fmax$, except for  $\fmax = V^{\pi_n}$ \blue{(in this case, for all $\lambda\in[0,1]$, $\nabla \ell_n (\pi_n; \lambda) = \nabla \ell_n (\pi_n) = -\nabla V^{\pi_n}(d_0)$, which is the negative policy gradient).}
Therefore, while the GAE policy gradient \citep{schulman2015high} is an approximation of the policy gradient,
$\nabla \what{\ell}_n (\pi)$  and $\nabla \what{\ell}_n (\pi; \lambda)$ are gradient approximations of \emph{different} online loss functions in \eqref{eq:online loss} and \eqref{eq:online loss with lambda weights}.

\vspace{-2mm}
\section{Experiments and Discussion} \label{sec:experiments}
\vspace{-1.5mm}


\begin{wrapfigure}{r}{0.3\textwidth}
	\vspace{-6mm}
	\begin{center}
		\includegraphics[width=0.3\textwidth]{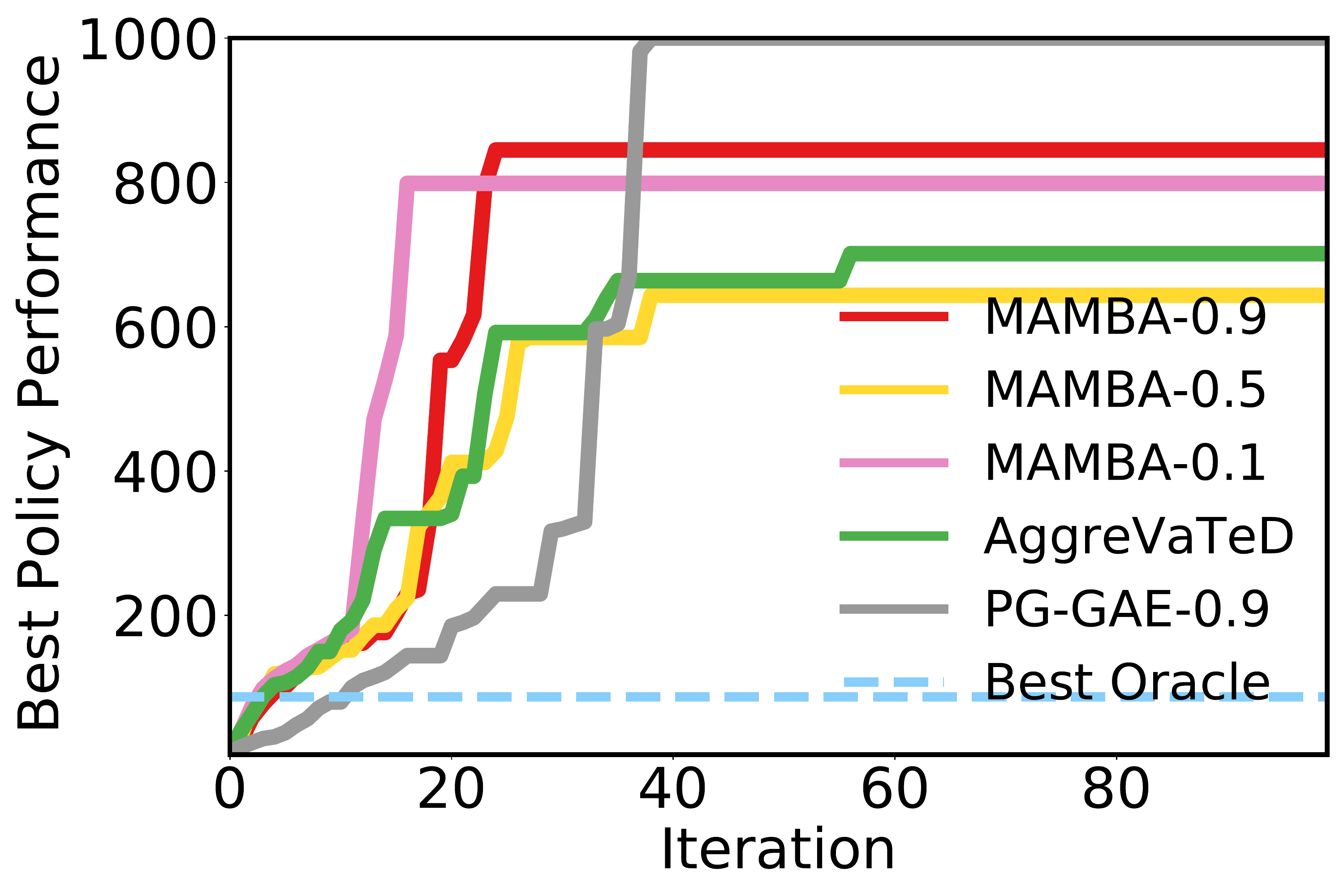}
		\vspace{-7mm}
	\end{center}
	\caption{\small \alg with different $\lambda$ values in a single-oracle setup in CartPole. A curve shows an algorithm's median performance across 8 random seeds.}
	\label{fig:lambda effects (cp)}
	\vspace{-3mm}
\end{wrapfigure}

We corroborate our theoretical discoveries with simulations of IL from multiple oracles. We compare \alg with two representative  algorithms: GAE Policy Gradient~\citep{schulman2015high} (\pg with $\lambda=0.9$) for direct RL and \aggd~\citep{sun2017deeply} for IL with a single oracle.
Because we can view these algorithms as different first-order feedback for policy optimization, comparing their performance allows us to study two important questions: 1) whether the proposed GAE-style gradient in \eqref{eq:lambda weighted gradient estimate} is an effective update direction for IL and 2) whether using multiple oracles helps the agent learn faster.

Four continuous Gym~\citep{brockman2016openai} environments are used: CartPole and DoubleInvertedPendulum (DIP) based on DART physics engine~\citep{lee2018dart}, and Halfcheetah and Ant based on Mujoco physics engine~\citep{todorov2012mujoco}.
To facilitate a meaningful comparison, we let these three algorithms use the same first-order optimizer\footnote{ADAM~\citep{kingma2014adam} for CartPole, and Natural Gradient Descent~\citep{amari1998natural} for DIP, Halfcheetah, and Ant}, train the same initial neural network policies, and share the same random seeds.
In each training iteration, an algorithm would perform $H$ rollouts following the RIRO paradigm (see also \cref{alg:mamba}),
where $H=8$ for CartPole and DIP and $H=256$ for Halfcheetah and Ant.
%
Each oracle policy here is a partially trained, \emph{suboptimal} neural network\footnote{The best and the worst oracles in scores of 87 and 9 in CartPole, 4244 and 2440 in DIP, 1726 and 1395 in Halfcheetah, and 1051  and 776 in Ant.}, and its value function approximator used in \alg is trained online along policy optimization, by Monte-Carlo estimates obtained from the RIRO data collection. Please see \cref{sec:exp details} for implementation details.
The codes are provided at \texttt{https://github.com/microsoft/MAMBA}.

\vspace{-1mm}
\paragraph{Effects of $\lambda$-weighting} First we consider in \cref{fig:lambda effects (cp)} the \emph{single}-oracle IL setting with the oracle that has the highest return in the CartPole environment. We see that with the help of the oracle policy, \aggd (which is \alg with $\lambda=0$) improves faster than \pg. However, while \aggd learns to significantly outperform the oracle, it does not reach the optimal performance, like \pg. To improve the learner performance, we use \alg with $\lambda>0$ to learn from the same suboptimal oracle. As shown in \cref{th:mamba performance} using a positive $\lambda$ can increase the amount of improvement that can be made from the oracle compared with $\lambda=0$. 
The trend in \cref{fig:lambda effects (cp)} supports this insight, where high $\lambda=0.9$ allows the learner to reach higher return. 
We found that using the middle $\lambda=0.5$ gives the worst performance, likely because it settles badly in the trade-off of properties.
In the following experiments, we will use $\lambda=0.9$ for \alg as it performs the best here. 

\begin{figure}[t]
	\centering
	\hspace{-15mm}
	%
	%
	\begin{subfigure}[b]{0.32\textwidth}
		\centering
		\includegraphics[width=0.8\textwidth]{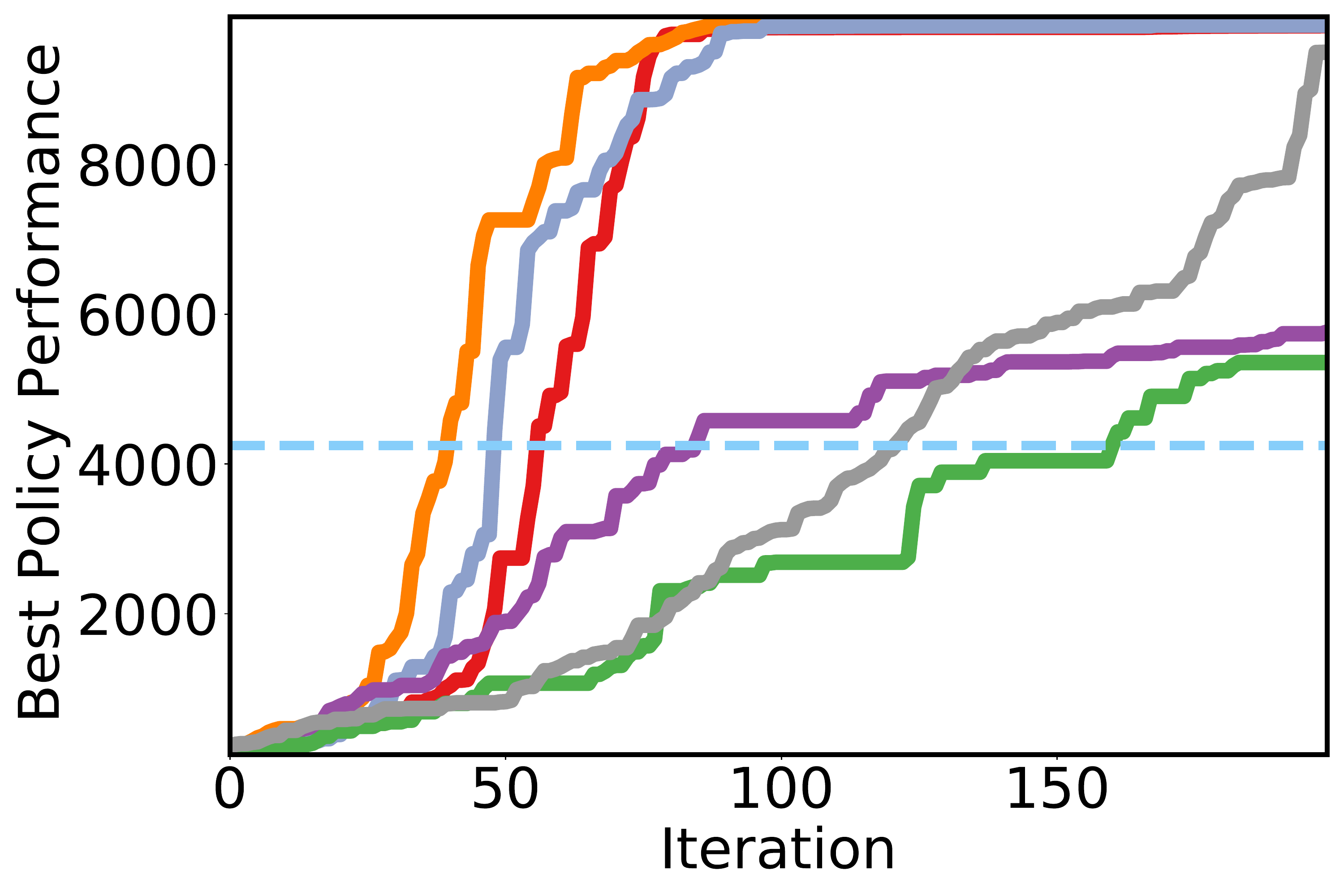}
		\caption{\small DIP}
		\label{fig:oracles effects (dip)}
	\end{subfigure}
	\hspace{-10mm}
	\begin{subfigure}[b]{0.32\textwidth}
		\centering
		\includegraphics[width=0.8\textwidth]{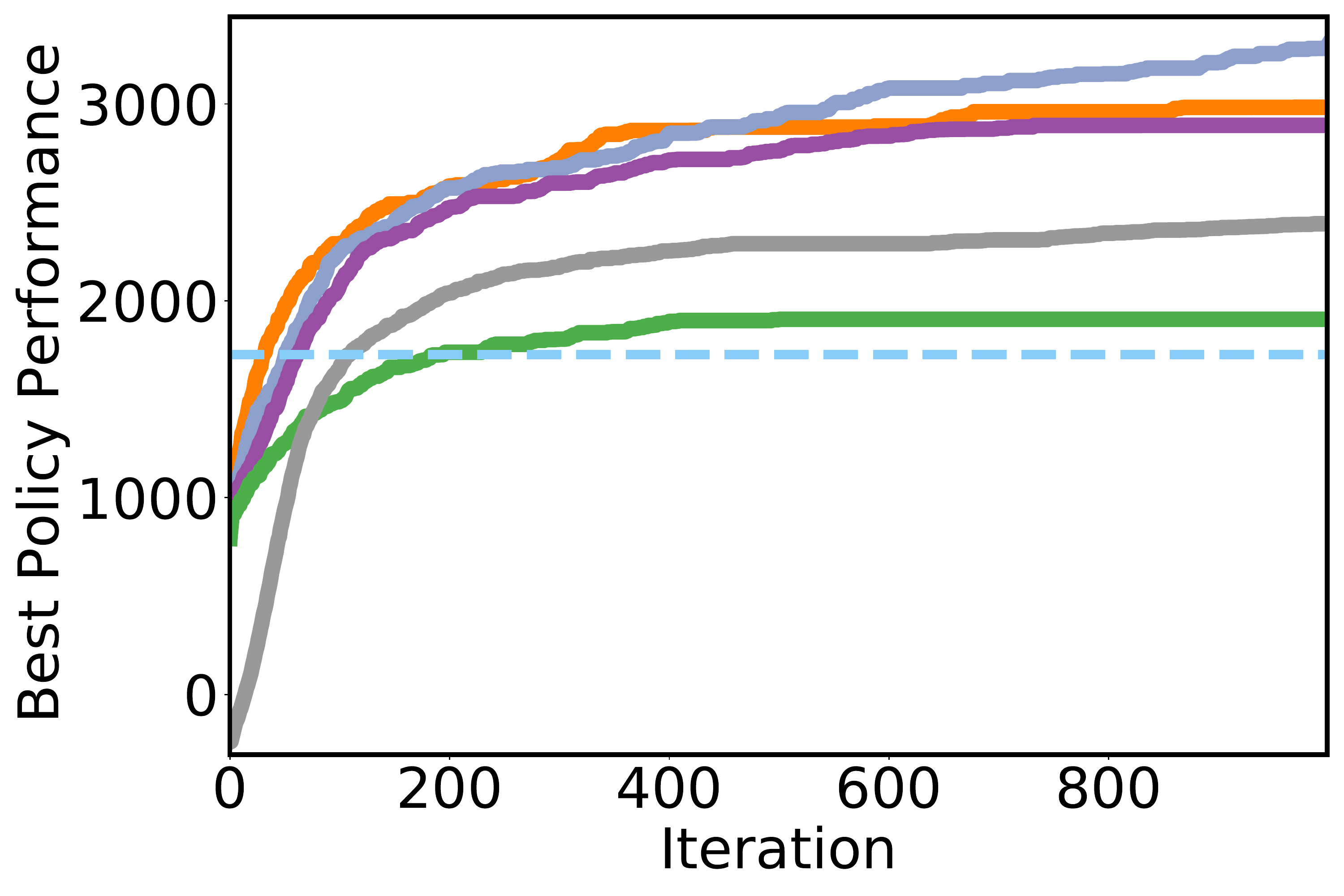}
		\caption{\small Halfcheetah}
		\label{fig:oracles effects (halfcheetah)}
	\end{subfigure}
	\hspace{-10mm}
	\begin{subfigure}[b]{0.32\textwidth}
		\centering
		\includegraphics[width=0.8\textwidth]{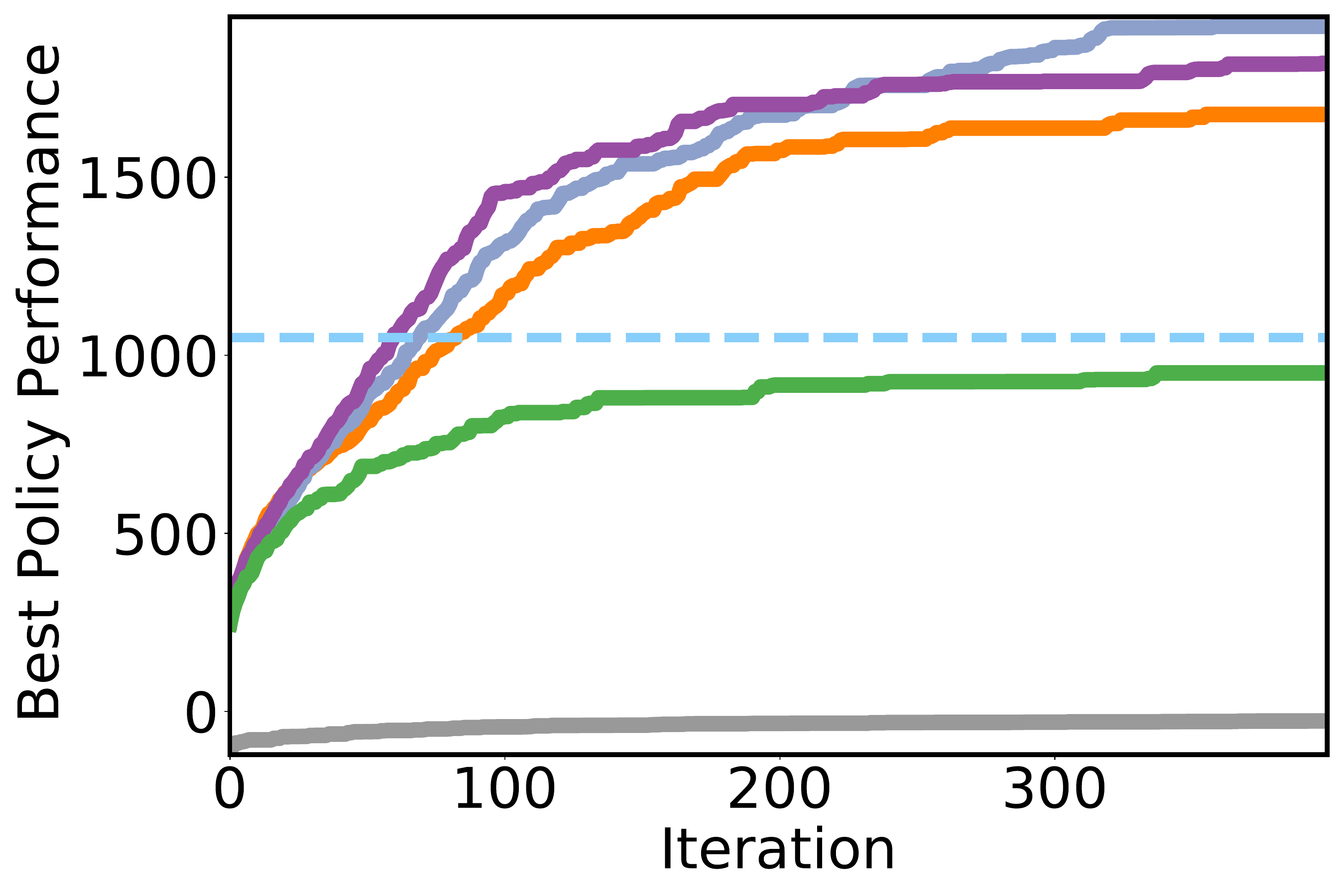}
		\caption{\small Ant}
		\label{fig:oracles effects (ant)}
	\end{subfigure}
	\hspace{-9mm}
	\begin{subfigure}[b]{0.25\textwidth}
		\centering
		\includegraphics[width=0.6\textwidth]{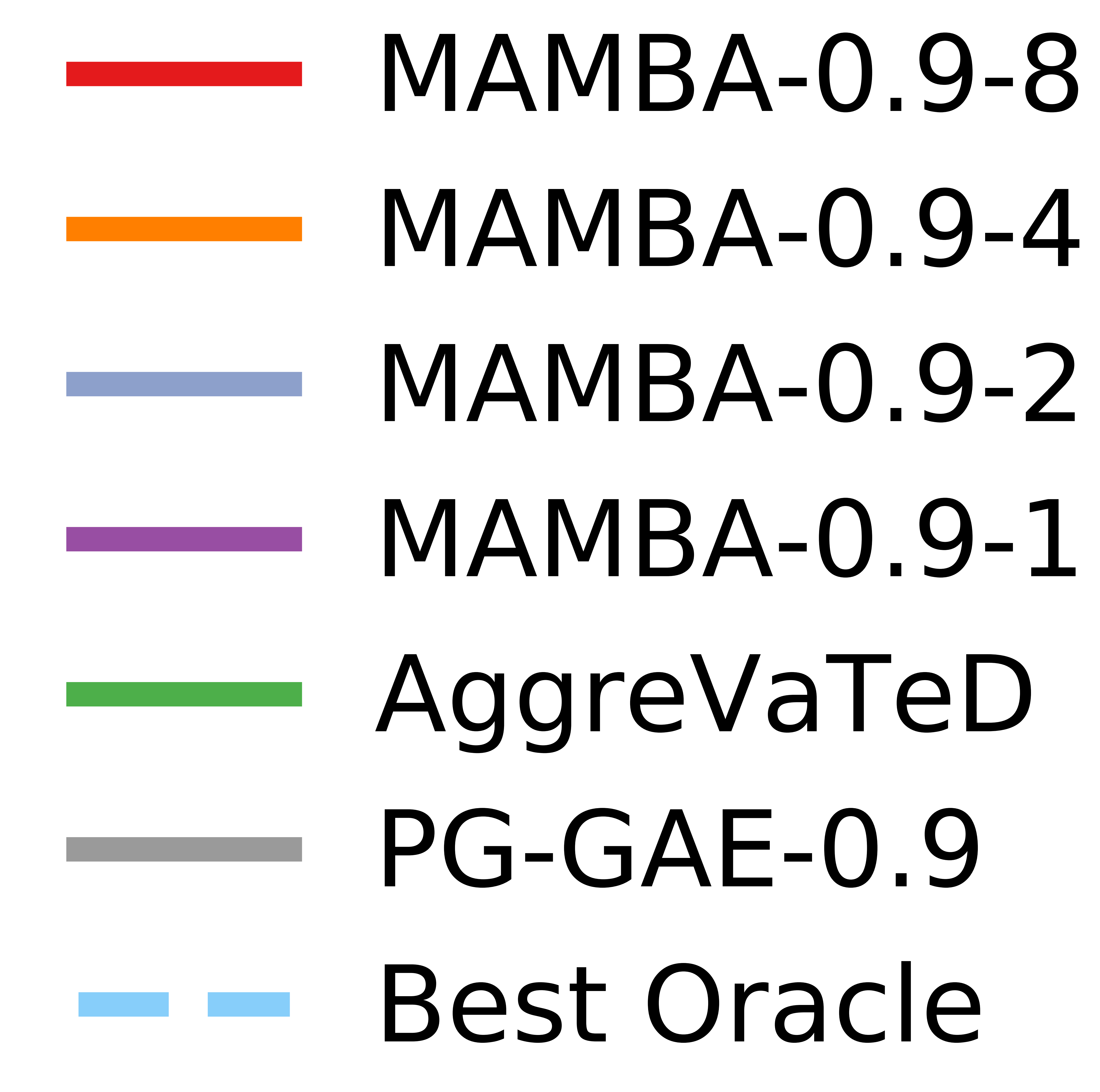}
		\vspace{8mm}
	\end{subfigure}
	\hspace{-20mm}
	\caption{\small
			Comparison of \alg with different number of oracles ($\lambda=0.9$).
			A curve shows an algorithm's median performance across 8 random seeds.
	}
	\vspace{-5mm}
	\label{fig:controlled exps}
\end{figure}

\vspace{-1mm}

\paragraph{Effects of multiple oracles} 
 We show the effects of using multiple oracles in \cref{fig:controlled exps} for the remaining three environments (the results of CartPole can be found in \cref{fig:all-results} and \cref{sec:exp details}). Here we run \alg with $\lambda=0.9$ with 1, 2, 4, or 8 oracles. We index these oracles in a descending order of their performances with respect to the initial state distribution;
 e.g., MAMBA-0.9-1 uses the best oracle and  MAMBA-0.9-2 uses the top two.
Interestingly, by including more but \emph{strictly weaker} oracles, \alg starts to improve the performance of policy optimization. Overall \alg greatly outperforms \pg and \aggd across all the environments, even when just using a single oracle.\footnote{\blue{The bad performance of \pg in Ant is due to that all learners can only collect a fixed number of trajectories in each iteration, as opposed to a fixed number of samples used in usual RL benchmarks. This setting is a harder RL problem and better resembles real-world data collection where reset is expensive.}}
%
\blue{The benefit of using more oracles becomes smaller, however, as we move to the right of \cref{fig:controlled exps} into higher dimensional problems. Although using more oracles can potentially yield higher performance, the learner also needs to spend more time to learn the oracles' value functions.
Therefore, under the same interaction budget, the value approximation quality worsens when there are more oracles.
This phenomenon manifests particularly in our direct implementation, which trains each value function estimator independently with Monte-Carlo samples. We believe that the scalability for dimensionality can be improved by using off-policy techniques and sharing parameters across different networks. We leave this as an important future direction.}

\vspace{-1mm}
\paragraph{Summary}
We conclude this paper by revisiting \cref{fig:all-results}, which showcases the best multi-oracle settings in \cref{fig:controlled exps}. Overall these results support the benefits of IL from multiple oracles and the new GAE-style IL gradient in \eqref{eq:lambda weighted gradient estimate}.
%
%
In conclusion, we propose a novel theoretical foundation and algorithm \alg for IL with multiple oracle policies. We study how the conflicts between different oracles can be resolved through the \benchmark and propose a new GAE-style gradient for the IL setting, which can also be used to improve the robustness and performance of existing single-oracle IL algorithms.
We provide regret-based theoretical guarantees on \alg and demonstrate its properties  empirically. The experimental results show that \alg is able to improve upon multiple, very suboptimal oracle policies to achieve the optimal performance, faster than both the pure RL method (\pg~\citep{schulman2015high}) and the single-oracle IL algorithm (\aggd~\citep{sun2017deeply}).

\clearpage

\section*{Broader Impact}
This paper is theoretical in nature, and so we expect the ethical and
societal consequences of our specific results to be minimal. More
broadly, we do expect that reinforcement learning will have
significant impact on society. There is much potential for benefits to
humanity in the often-referenced application domains of precision
medicine, personalized education, and elsewhere. There is also much
potential for harms, both malicious and unintentional. To this end, we
hope that research into the foundations of reinforcement learning can
help enable these applications and mitigate harms through the
development of algorithms that are efficient, robust, and safe.

\bibliographystyle{unsrtnat}
\bibliography{ref}

\clearpage
\appendix

\begin{center}
	{\Large\textbf{Appendix}}
\end{center}

\section{MDP Examples for \cref{sec:conceptual algorithm}} \label{sec:compare}

\paragraph{Problem with selecting oracles based on initial value.}
In the example of Figure~\ref{fig:mdp_example_2}, each oracle $\pi_\ell$ and $\pi_r$ individually gets same the suboptimal reward of $1/2$.
Alternatively, we can switch between the oracles once to get a reward of $3/4$, and twice to get the optimal reward of $1$.
\begin{figure}[H]
	\centering
	\includegraphics[width=0.8\textwidth, trim=15mm 30mm 15mm 10mm, clip]{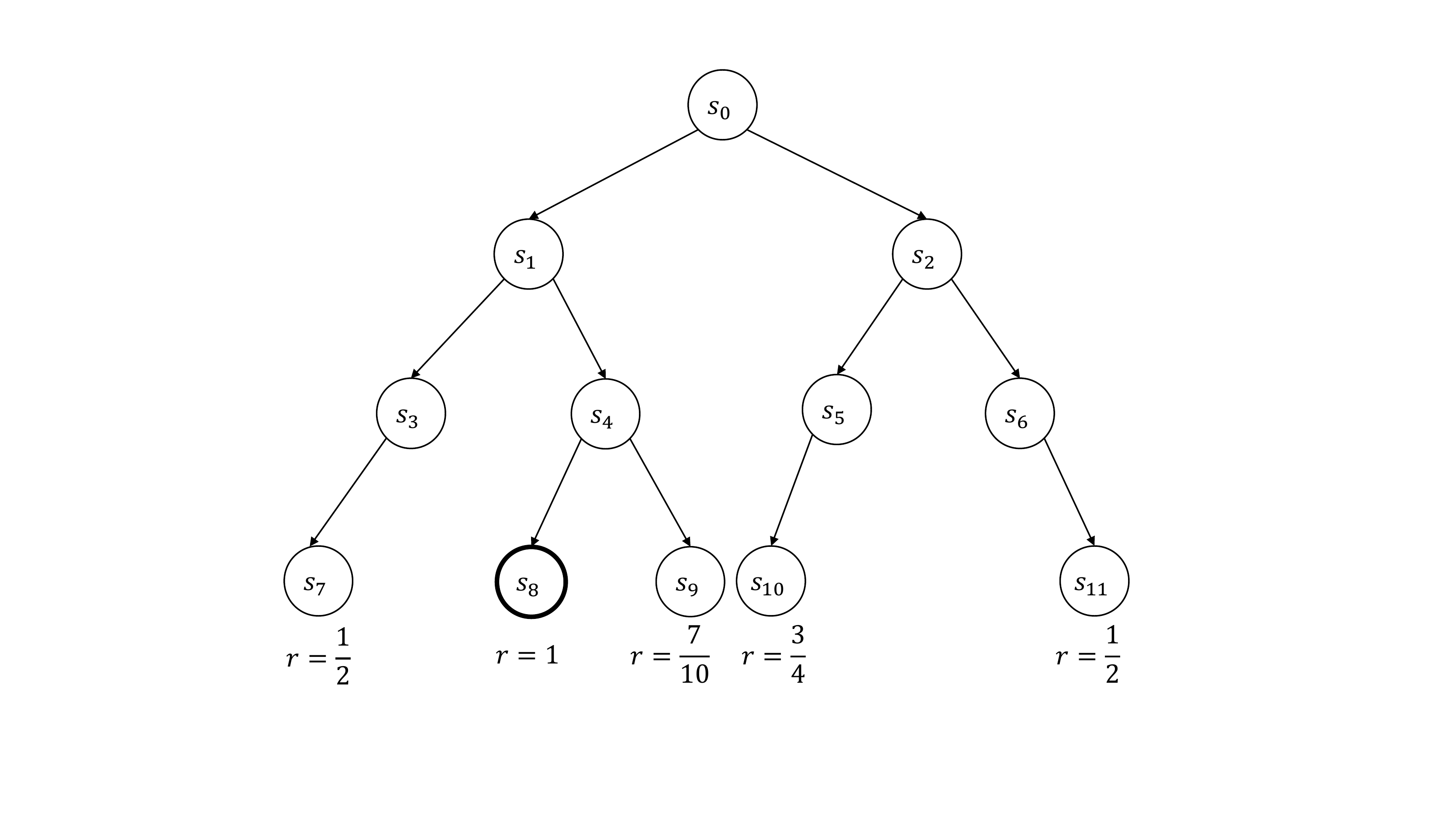}
	\caption{An example MDP for illustrative purposes. All terminal states not shown give a reward of 0 and intermediate states have no rewards. Two oracle policies $\pi_\ell$ and $\pi_r$ choose the left and right actions respectively in each state. The optimal terminal state is outlined in bold.}
	\label{fig:mdp_example_2}
\end{figure}

\paragraph{Ordering of $\pi^\bullet$ and $\pimax$.}
Consider the example MDP of Figure~\ref{fig:mdp_example_1}. In the state $s_0$, the policy $\pi^\bullet$ selects the oracle with largest value in $s_0$ and goes left. It subsequently selects the right oracle in $s_1$ and left in $s_4$ to get the optimal reward. $\pimax$ on the other hand chooses between the left and right actions in $s_0$ based on $\fmax(s_1) = 0.7$ and $\fmax(s_2) = 3/4$. Consequently it goes right and eventually obtains a suboptimal reward of $3/4$. In this case, we see that $\pi^\bullet$ is better than $\pimax$. On the other hand, if we swap the rewards of $s_7$ and $s_{11}$, then $\pi^\bullet$ chooses the right action in $s_0$ and gets a suboptimal reward. Further swapping the rewards of $s_9$ and $s_{10}$ makes $\pimax$ pick the left action in $s_0$ and it eventually reaches the optimal reward. This illustrates clearly that $\pimax$ and $\pi^\bullet$ are incomparable in general.
\begin{figure}[H]
	\centering
	\includegraphics[width=0.8\textwidth, trim=15mm 30mm 15mm 10mm, clip]{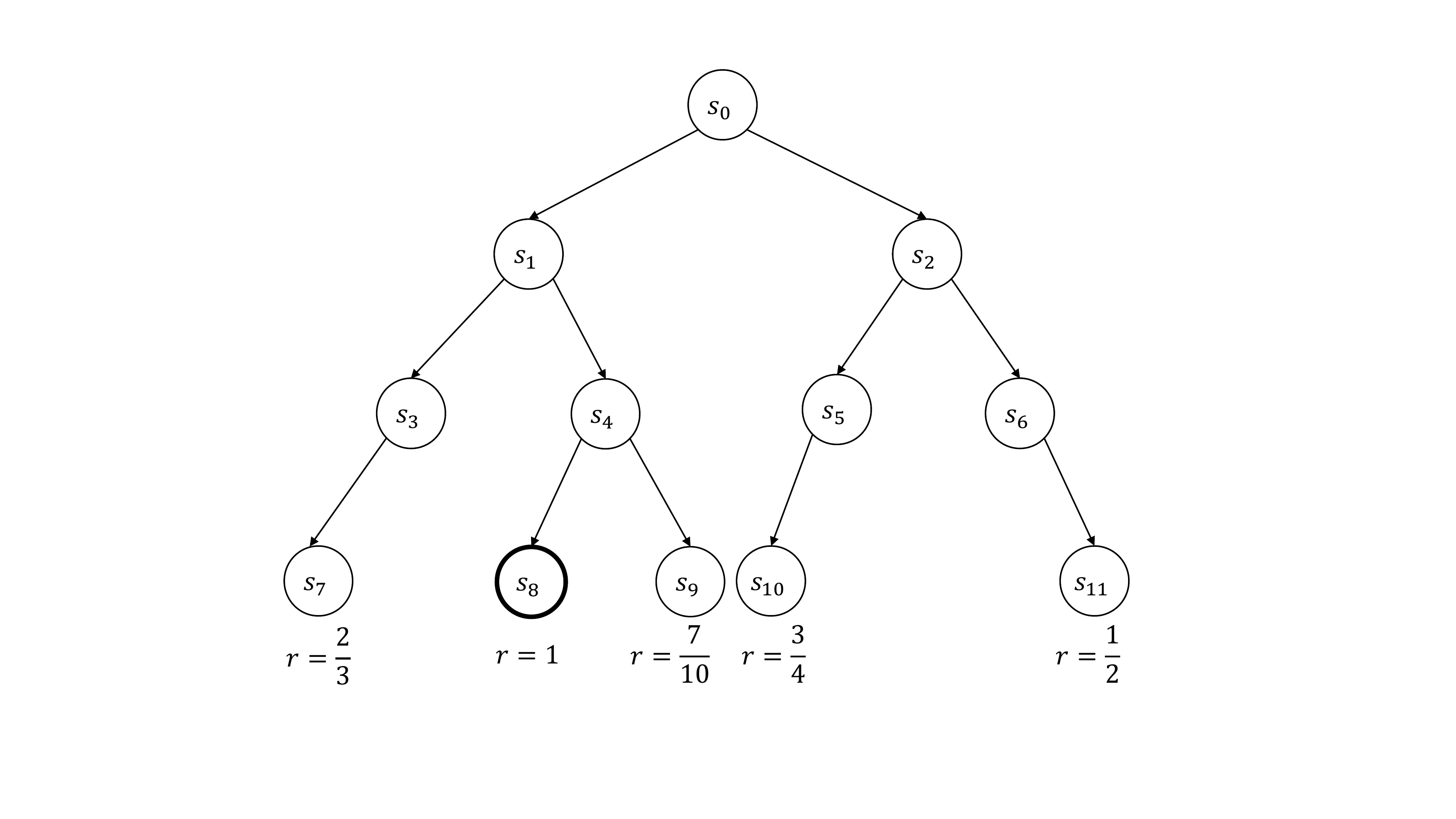}
	\caption{An alternate example MDP for illustrative purposes. All terminal states not shown give a reward of 0 and intermediate states have no rewards. Two oracle policies $\pi_\ell$ and $\pi_r$ choose the left and right actions respectively in each state. The optimal terminal state is outlined in bold.}
	\label{fig:mdp_example_1}
\end{figure}

\section{Additional Notes on Related Work}

Several prior works proposed empirical approaches to IL settings with multiple oracles. InfoGAIL~\citep{li2017infogail} is an extension of GAIL~\citep{ho2016gail} that aims at automatically identifying semantically meaningful latent factors that can explain variations in demonstrations across oracles. It assumes that demonstrations come from a mixture-of-oracles policy, where each demonstration is generated by sampling a value of the latent factors from a prior and using it to condition an oracle's action choices. InfoGAIL tries to recover this oracle mixture. In contrast, \alg can be viewed as choosing actions based on learned estimates of oracles' value functions without imitating any single oracle or their mixture directly. In multi-modal IL \citep{hausman2017multimodal}, latent factors conceptually similar to those in InfoGAIL correspond to different skills being demonstrated by the oracles, and \citet{tamar2018mi}'s approach focuses on settings where these factors characterize different oracles' intentions. OIL~\citep{li2018oil} is more similar to \alg: like \alg, it uses individual oracle policies' state values to decide on an action in a given state. However, OIL does so by using the best-performing oracle in a given state as the learner's ``critic" and doesn't justify its approach theoretically.

At least two algorithms have used a Bayesian approach to decide which oracle to trust in a multiple-oracle setting. AC-Teach~\citep{kurenkov2019acteach} models each oracle with a set of attributes and relies on a Bayesian approach to decide which action to take based on their demonstrations. \citet{gimelfarb2018bmc} assumes that oracles propose reward functions, some of which are inaccurate, and uses Bayesian model combination to aggregate their advice. Out of the two, AC-Teach can be regarded as a Bayesian counterpart of \alg, but, like other existing approaches to learning from multiple oracles, doesn't come with a theoretical analysis.

\blue{Lastly, we remark that the improvement made from the suboptimal oracles in \alg is attributed to using reward information in IL, similar to \aggd~\citep{sun2017deeply} but different from behavior cloning~\citep{pomerleau89alvinn} or DAgger~\citep{ross2011dagger}. While we allow weaker oracle policies than are typical in IL literature, they still need to have meaningful behaviors to provide an informative advantage \alg to improve upon; e.g., they cannot be completely uniformly random policies as is often done in the batch RL setting.
}

\section{Proofs} \label{app:proofs}

\subsection{Proof of \cref{lm:pdl}}
\PDL*
\begin{proof}
	By definition of $d^\pi$, we can write
	\[
		V^\pi(d_0) = T \E_{s \sim d^\pi}\E_{s \sim \pi|s}[r(s,a)] =  T \E_{s \sim d^\pi}[r(s,\pi)]
	\]
	 On the other hand, we can write
	\begin{align*}
	- f(d_0) &= \sum_{t=1}^{T-1} f(d_t) - \sum_{t=0}^{T-1} f(d_t) = T\E_{s \sim d^\pi}[\E_{a\sim\pi|s}\E_{s' \sim \PP|s,a}[f(s')] - f(s)]
	\end{align*}
	Combing the two equalities shows the result.
\end{proof}

\subsection{Proof of \cref{cr:improvable baseline}}
\ImprovableBaselinValueFunction*

\begin{proof}
	Because \cref{lm:pdl} holds for any MDP, given the state $s$ in \cref{cr:improvable baseline} we can define a new MDP whose initial state is at $s$ and properly adapt the problem horizon. Then \cref{cr:improvable baseline} follows directly from applying \cref{lm:pdl} to this new MDP.
\end{proof}

\subsection{Proof of \cref{pr:improvable_bullet}}

\VmaxImprovable*
\begin{proof}
	Let us recall the definition~\eqref{eq:k_s and pi_bullet} of $k_s$ and let us assume without loss of generality that $k_s = 1$. We observe that
	\begin{align*}
	\Amax(s,\pi^\bullet)
	&= r(s,\pi^\bullet) +  \E_{a\sim\pi^\bullet|s}\E_{s'\sim\PP|s,a}[\fmax(s')] - \fmax(s)\\
	&\geq r(s,\pi^\bullet) +  \E_{a\sim\pi^\bullet|s}\E_{s'\sim\PP|s,a}[V^{1}(s')] - V^1(s)
	= A^{V^1}(s,\pi^{1}) \geq 0
	\end{align*}
    where the last step follows since $\pi^\bullet(a|s) = \pi^{k_s}(a|s) = \pi^1(a|s)$ and the advantage of a policy with respect to its value function is always 0.
\end{proof}

\subsection{Proof of \cref{th:meta performance theroem}}
\MetaTheorem*
\begin{proof}
	By using \eqref{eq:regret equality} and the assumption on the first-order algorithm, we can write
	\begin{align*}
		\E\left[ \frac{1}{N}\sum_{n\in[N]} V^{\pi_n}(d_0) \right]
		&= \fmax(d_0) + \E \left[\Delta_N - \epsilon_N(\Pi) - \frac{\regret_N}{N} \right]\\
		&\geq \fmax(d_0) +  \E \left[\Delta_N - \epsilon_N(\Pi) \right] - O\left(\beta + \sqrt{\frac{\nu}{N}} \right)
	\end{align*}
	Finally, using $\frac{1}{N}\sum_{n\in[N]} V^{\pi_n}(d_0) \leq \max_{n\in[N]} V^{\pi_n}(d_0)$ and the definition of $\fmax$, we have the final statement.
\end{proof}

\subsection{Proof of \cref{pr:lambda weighed PDL}}

\LambdaWeightedPDL*

\begin{proof}

The proof is uses a new generalization of the Performance Difference Lemma (\cref{lm:pdl}), which we state generally for non-Markovian processes. A similar equality holds for the infinite-horizon discounted problems.

\begin{lemma}[Non-even performance difference lemma]
   \label{lm:noneven performance difference lemma}
   Let $\pi$ be a policy and let $f$ be any function that is history dependent such that $\E_{h_T\sim d_T^\pi}[f(h_T)]= 0$.
   Let $\tau_0, \tau_1, \tau_2, \dots \tau_I$ be monotonically increasing integers where $\tau_0=0$ and $\tau_I=T$.
   For any non-Markovian decision process, it holds that,
   \begin{align*}
   V^\pi(d_0) - f(d_0) = \sum_{k=0}^{I-1} \E_{h_{\tau_k} \sim d_{\tau_k}^\pi}[ A_{(i_k)}^{f,\pi} (h_{\tau_k}, \pi)]
   \end{align*}
   where $i_k = \tau_{k+1}-\tau_k-1$.
\end{lemma}
\begin{proof}[Proof of \cref{lm:noneven performance difference lemma}]
   By definition,
	\[   V^\pi(d_0)
   = \sum_{t=0}^{T-1}\E_{h_t \sim d_t^\pi}[r(h_t,\pi)]
   = \sum_{k=0}^{I-1}\E_{h_{\tau_k} \sim d_{\tau_k}^\pi } \E_{\rho^\pi |h_{\tau_k}} \left[ \sum_{t=\tau_k}^{\tau_{k+1}-1} r(h_t, a_t)  \right]
   \]
	On the other hand, we can write
	$
   - f(d_0)
   = \sum_{k=1}^{I}f(d_{\tau_k}^\pi) - \sum_{k=0}^{I-1} f(d_{\tau_k}^\pi)
   $.
   Combing the two equalities shows the result.
\end{proof}

Now return to the Markovian case. Using \cref{lm:noneven performance difference lemma}, we derive a $\lambda$-weighted Performance Difference Lemma (\cref{pr:lambda weighed PDL}). A history dependent (discounted) version can be shown similarly.
To simplify writing, we let $\Theta = V^\pi(d_0) - f(d_0)$ and $A_{(i)} = A_{(i)}^{f,\pi}$ as shorthands, and we will omit the dependency on random variables in the expectation. Using \cref{lm:noneven performance difference lemma}, we can write
\begin{align*}
\Theta &= \sum_{t=0,1,\dots,T-1} \E_{d_t^\pi}\E_{\pi}[ A_{(0)} ]\\
2\Theta
&= \sum_{t=0,2,4\dots} \E_{d_t^\pi}\E_{\pi}[ A_{(1)} ]
+  \left( \E_{d_0}\E_{\pi}[ A_{(0)} ] + \sum_{t=1,3\dots} \E_{d_t^\pi}\E_{\pi}[ A_{(1)} ] \right) \\
&= \E_{d_0}\E_{\pi}[ A_{(0)} ] + \sum_{t=0}^{T-1} \E_{d_t^\pi}\E_{\pi}[ A_{(1)} ]\\
3\Theta
&= \sum_{t=0,3,6\dots} \E_{d_t^\pi}\E_{\pi}[ A_{(2)} ]
+ \left(\E_{d_0}\E_{\pi}[ A_{(0)} ] +  \sum_{t=1,4\dots} \E_{d_t^\pi}\E_{\pi}[ A_{(2)} ] \right) \\
&\quad +\left ( \E_{d_0}\E_{\pi}[ A_{(1)} ] + \sum_{t=2,5\dots} \E_{d_t^\pi}\E_{\pi}[ A_{(2)} ] \right) \\
&= \E_{d_0}\E_{\pi}[ A_{(0)} ] + \E_{d_0}\E_{\pi}[ A_{(1)} ] + \sum_{t=0}^{T-1} \E_{d_t^\pi}\E_{\pi}[ A_{(2)} ]\\
&\hspace{2mm}\vdots
\end{align*}

Applying a $\lambda$-weighted over these terms, we then have
\begin{align*}
(1-\lambda)(1+2\lambda+3\lambda^2+\dots)\Theta = T \E_{d^\pi} \E_{\pi}\left[ (1-\lambda) \sum_{i=0}^\infty \lambda^i A_{(i)}  \right] + \lambda \sum_{i=0}^\infty \lambda^i\E_{d_0} \E_{\pi}[ A_{(i)}]
\end{align*}
Because for $\lambda < 1$,
$
\lambda+2\lambda^2+3\lambda^3+\dots =  \frac{\lambda}{(1-\lambda)^2}
$, we have
\begin{align*}
(1-\lambda)(1+2\lambda+3\lambda^2+\dots) = \frac{1-\lambda}{\lambda}(\lambda+2\lambda^2+3\lambda^3+\dots) = \frac{1-\lambda}{\lambda} \frac{\lambda}{(1-\lambda)^2} = \frac{1}{1-\lambda}
\end{align*}
The above derivation implies that
\begin{align*}
\Theta
&= V^\pi(d_0) - f(d_0) = (1-\lambda)T \E_{d^\pi} \E_{\pi}\left[ (1-\lambda) \sum_{i=0}^\infty \lambda^i A_{(i)}  \right]  + \lambda (1-\lambda) \sum_{i=0}^\infty \lambda^i \E_{d_0} \E_{\pi} \left[  A_{(i)} \right]
\end{align*}

\end{proof}

\subsection{Proof of \cref{pr:lambda weighed online gradient}}
\LambdaWeightedOnlineGradient*

\begin{proof}

We first show the gradient expression in the second term in $h(\pi;\lambda)$.
\begin{lemma} \label{lm:gradient of lambda weighted THOR}
\begin{align*}
\nabla \E_{s \sim d_0} \left[  A_\lambda^{f,\pi}(s,\pi) \right]
= \sum_{t=0}^{T-1} \lambda^t \E_{s_t \sim d_t^\pi } \E_{a_t \sim \pi|s_t} \left[ \nabla \log\pi(a_t|s_t) A_\lambda^{f,\pi}(s_t,a_t)  \right]
\end{align*}
\end{lemma}
\begin{proof}[Proof of \cref{lm:gradient of lambda weighted THOR}]
Define $Q_{(i-t)}^{f,\pi}(s_t, a_t) \coloneqq \E_{\rho^\pi|s_t, a_t} \left[ \sum_{\tau=0}^{i-t}  r(s_{t+\tau}, a_{t+\tau}) + f(s_{i+1}) \right]$.
By using the definition of $i$-step advantage function $A_{(i)}^{f,\pi}$, we can first rewrite the desired derivative as
\begin{align*}
\nabla \E_{s \sim d_0} \left[ A_{(i)}^{f,\pi}(s,\pi) \right]
&= \nabla  \E_{s \sim d_0} \E_{\rho^\pi|s_0} \left[r(s_0, a_0) + r(s_1, a_1) + \dots + r(s_i, a_i) + f(s_{i+1}) \right] - f(s_t) \\
&= \sum_{t=0}^i \E_{s_t\sim d_t^\pi} \E_{a_t\sim\pi|s_t} \left[ \nabla\log\pi(a_t|s_t)  Q_{(i-t)}^{f,\pi}(s_t, a_t) \right]\\
&= \sum_{t=0}^i \E_{s_t\sim d_t^\pi} \E_{a_t\sim\pi|s_t} \left[ \nabla\log\pi(a_t|s_t)  A_{(i-t)}^{f,\pi}(s_t, a_t) \right]
\end{align*}
where in the last equality we use the fact $\nabla \E_{a\sim\pi|s}[f(s)] = 0$ for any $f:\SS\to\R$.
Therefore, we can write the $\lambda$-weighted version as follows:
\begin{align*}
\nabla \E_{s \sim d_0} \left[  A_\lambda^f(s,\pi) \right]
&=  (1-\lambda) \sum_{i=0}^\infty  \nabla \E_{s \sim d_0} \left[  \lambda^i  A_{(i)}^{f,\pi}(s,a)  \right]\\
&=  (1-\lambda) \sum_{i=0}^\infty  \sum_{t=0}^i \lambda^i \E_{s_t\sim d_t^\pi} \E_{a_t\sim\pi|s_t} \left[ \nabla\log\pi(a_t|s_t) A_{(i-t)}^{f,\pi}(s_t, a_t) \right] \\
&=  (1-\lambda) \sum_{t=0}^{T-1} \sum_{i=t}^\infty  \lambda^i  \E_{s_t\sim d_t^\pi} \E_{a_t\sim\pi|s_t} \left[ \nabla\log\pi(a_t|s_t)  A_{(i-t)}^{f,\pi}(s_t, a_t) \right] \\
&=  (1-\lambda) \sum_{t=0}^{T-1}  \lambda^t \sum_{j=0}^\infty  \lambda^j \E_{s_t\sim d_t^\pi} \E_{a_t\sim\pi|s_t} \left[ \nabla\log\pi(a_t|s_t) A_{(j)}^{f,\pi}(s_t, a_t) \right] \\
&=  \sum_{t=0}^{T-1}  \lambda^t \E_{s_t\sim d_t^\pi} \E_{a_t\sim\pi|s_t} \left[ \nabla\log\pi(a_t|s_t) A_\lambda^{f,\pi}(s_t, a_t) \right]
\end{align*}
\end{proof}

With this intermediate result, we can further derive the gradient expression in the first term in $\nabla h(\pi;\lambda)$ when $\mu=\pi$:
\begin{align*}
T \E_{s \sim d^\pi}  \left[ \nabla  A_\lambda^{f,\pi}(s,\pi) \right]
&= \sum_{t=0}^{T-1}  \E_{s \sim d_t^\pi }  \nabla\left[  A_\lambda^f(s,\pi) \right]\\
&= \sum_{t=0}^{T-1}  \sum_{\tau=t}^{T-1} \lambda^{\tau-t} \E_{s_\tau \sim d_\tau^\pi} \E_{a_\tau \sim\pi|s_\tau} \left[ \nabla\log\pi(a_\tau|s_\tau) A_\lambda^{f,\pi}(s_\tau, a_\tau) \right] \\
&= \sum_{\tau=0}^{T-1}  \E_{s_\tau \sim d_\tau^\pi} \E_{a_\tau \sim\pi|s_\tau} \left[ \nabla\log\pi(a_\tau|s_\tau) A_\lambda^{f,\pi}(s_\tau, a_\tau) \right]  \left( \sum_{t=0}^{\tau}  \lambda^{\tau-t} \right)\\
&= \sum_{t=0}^{T-1}   \frac{1-\lambda^{t+1}}{1-\lambda} \E_{s_t \sim d_t^\pi} \E_{a_t \sim\pi|s_t} \left[ \nabla\log\pi(a_t|s_t) A_\lambda^{f,\pi}(s_t, a_t) \right]
\end{align*}
Finally, combining the two equalities, we arrive at a very clean expression:
\begin{align*}
	\nabla h(\pi;\lambda) |_{\mu=\pi}
&= (1-\lambda) T \E_{s \sim d^{\pi}}  \left[ \nabla  A_\lambda^f(s,\pi) \right]  + \lambda  \E_{s \sim d_0} \left[ \nabla  A_\lambda^f(s,\pi) \right]   \\
&= \sum_{t=0}^{T-1}  \E_{s_t \sim d_t^\pi} \E_{a_t \sim\pi|s_t} \left[ \nabla\log\pi(a_t|s_t) A_\lambda^{f,\pi}(s_t, a_t) \right]  \left(  (1-\lambda)  \frac{1-\lambda^{t+1}}{1-\lambda}+ \lambda \cdot \lambda^t \right)\\
&= \sum_{t=0}^{T-1}  \E_{s_t \sim d_t^\pi} \E_{a_t \sim\pi|s_t} \left[ \nabla\log\pi(a_t|s_t) A_\lambda^{f,\pi}(s_t, a_t) \right]
\end{align*}
\end{proof}

\subsection{Proof of \cref{lm:lambda advantage (practical form)}}

\LambdaAdvantagePracticalFrom*
\begin{proof}
This equality can be derived as follows:
\begin{align*}
	&\Ah_\lambda^\pi(s_t,a_t)\\
	&= (1-\lambda) \E_{\xi_t \sim \rho^\pi|s_t} \left[ \sum_{i=0}^\infty \lambda^i \left( \sum_{\tau=t}^{t+i} r(s_\tau, a_\tau) +\fmaxh(s_{t+i+1}) \right) \right]  -\fmaxh(s_t) \\
	&= \E_{\xi_t \sim \rho^\pi|s_t} \left[  \sum_{\tau=t}^{T-1} r(s_\tau, a_\tau) \left(   (1-\lambda)  \sum_{i=\tau-t}^\infty \lambda^i \right) +  (1-\lambda) \sum_{i=0}^\infty \lambda^i\fmaxh(s_{t+i+1}) \right]  -\fmaxh(s_t) \\
	&= \E_{\xi_t \sim \rho^\pi|s_t} \left[  \sum_{\tau=t}^{T-1} \lambda^{\tau-t} r(s_\tau, a_\tau)  + (1-\lambda) \sum_{\tau=t}^{T-1} \lambda^{\tau-t}\fmaxh(s_{\tau+1}) \right]  -\fmaxh(s_t) \\
	&= \E_{\xi_t \sim \rho^\pi|s_t} \left[  \sum_{\tau=t}^{T-1} \lambda^{\tau-t} r(s_\tau, a_\tau)  +\sum_{\tau=t}^{T-1} \lambda^{\tau-t}\fmaxh(s_{\tau+1}) - \sum_{\tau=t+1}^{T-1} \lambda^{\tau-t}\fmaxh(s_{\tau})  \right]  -\fmaxh(s_t) \\
	&= \E_{\xi_t \sim \rho^\pi|s_t} \left[  \sum_{\tau=t}^{T-1} \lambda^{\tau-t} \left( r(s_\tau, a_\tau) + \fmaxh(s_{\tau+1}) - \fmaxh(s_{\tau}) \right) \right] \\
	&= \E_{\xi_t \sim \rho^\pi|s_t} \left[  \sum_{\tau=t}^{T-1} \lambda^{\tau-t} \Ah(s_\tau, a_\tau) \right] \\
\end{align*}
\end{proof}

\subsection{Proof of \cref{th:mamba performance}} \label{app:proof of th2}
\MambaPerformance*

\begin{proof}


	The proof \cref{th:mamba performance} is based on the non-trivial technical lemma of this general $\lambda$-weighted advantage, which we recall below.
	\LambdaWeightedPDL*


	To prove the theroem, we can then write down an equality like \eqref{eq:regret equality} by the equality $V^{\pi_n}(d_0) - \fmax(d_0) = \ell_n(\pi_n;\lambda)$ we just obtained:
	\begin{align*}
		\frac{1}{N}\sum_{n\in[N]} V^{\pi_n}(d_0) = \fmax(d_0) + \Delta_N - \epsilon_N(\Pi) - \frac{\regret_N}{N}
	\end{align*}
	where $\regret_N$, $\Delta_N $, and $\epsilon_N(\Pi)$ are now defined with respect to the $\lambda$-weighted online loss $\ell_n(\pi;\lambda)$.
	Therefore, running a no-regret algorithm with respect to the approximate gradient \eqref{eq:lambda weighted gradient estimate} of this online loss function $\ell_n(\pi;\lambda)$ would imply a similar performance guarantee shown in \cref{th:meta performance theroem} (see the proof \cref{th:meta performance theroem}).

	Finally, to justify the use of \eqref{eq:lambda weighted gradient estimate}, what remains to be shown is that the term $\Delta_N-\epsilon_N(\Pi)$ behaves similarly as before.
	First we notice that, because $\pimax$ may not be the best policy for the multi-step advantage in the online loss $\ell_n(\pi;\lambda)$, $\epsilon_N(\Pi)$ now be negative (which is in our favor).
	Next, we show that $\Delta_N\geq 0$ is true by \cref{pr:improvement for the lambda weighted loss} (whose proof is given below).
	\begin{restatable}{proposition}{ImprovementFromLambdaLoss} 	\label{pr:improvement for the lambda weighted loss}
		It holds $-\ell_n(\pimax;\lambda)\geq0$.
	\end{restatable}

	These results conclude the proof of \cref{th:mamba performance}.

\end{proof}

\subsubsection{Proof of \cref{pr:improvement for the lambda weighted loss}}
\ImprovementFromLambdaLoss*
 \begin{proof}

	We first prove a helpful lemma.
	\begin{lemma} \label{eq:helpful lemma for proving improvement}
		For $\pimax$, it holds that $ \Amax_{(i)}(s,\pimax) \geq 0$.
	\end{lemma}
	\begin{proof}[Proof of \cref{eq:helpful lemma for proving improvement}]
		Without loss of generality, take $s=s_0$.
		First we arrange
		\begin{align*}
		\Amax_{(i)}(s_0,\pimax)
		&= \E_{\rho^{\pimax}|s_0} \left[r(s_0, a_0) + r(s_1, a_1) + \dots + r(s_i, a_i) +  V^{k_{s_{i+1}}}(s_{i+1}) \right] - V^{k_{s_0}}(s_0)\\
		&=  \E_{\rho^{\pimax}|s_0} \left[r(s_0, a_0) + r(s_1, a_1) + \dots + Q^{\max}(s_i, \pimax) \right] - V^{k_{s_0}}(s_0)
		\end{align*}
		where we have the inequality
		\begin{align*}
		Q^{\max}(s_i, \pimax)
		&\coloneqq \E_{a_i\sim\pimax|s_i}[ r(s_i, a_i) + \E_{s_{i+1}\sim\PP|s_i, a_i}[ V^{k_{s_{i+1}}}(s_{i+1}) ]] \\
		&\geq \E_{a_i\sim \pi^{k_{s_i}}|s_i}[ r(s_i, a_i) + \E_{s_{i+1}\sim\PP|s_i, a_i}[ V^{k_{s_{i+1}}}(s_{i+1}) ]]\\
		&\geq \E_{a_i\sim \pi^{k_{s_i}}|s_i}[ r(s_i, a_i) + \E_{s_{i+1}\sim\PP|s_i, a_i}[ V^{k_{s_{i}}}(s_{i+1}) ]] \\
		&=  V^{k_{s_i}}(s_i)
		\end{align*}
		By applying this inequality recursively, we get
		\begin{align*}
		\Amax_{(i)}(s_0,\pimax)  \geq
		V^{k_{s_0}}(s_0)-  V^{k_{s_0}}(s_0) \geq  0
		\end{align*}
	\end{proof}
	The lemma above implies $\Amax_\lambda(s,\pimax) \geq 0$ for $\lambda\geq0$ and therefore we have
	\begin{align*}
	-l_n(\pimax;\lambda) =(1-\lambda) T \E_{s \sim d^{\pi_n}}\left[ \Amax_\lambda(s,\pimax) \right]  + \lambda \E_{s \sim d_0} \left[ \Amax_\lambda(s,\pimax) \right]  \geq 0
	\end{align*}
\end{proof}

\section{Experiment Details and Additional Results} \label{sec:exp details}

In this section we describe the details of \alg and additional experimental results.

\subsection{Implementation Details of \alg}




\def\ValueUpdate{\text{MonteCarloRegression}}
\def\SampleSwitchTime{\text{SampleSwitchTime}}
\def\UpdateWhitening{\text{UpdateInputWhitening}}
\def\UpdatePolicy{\text{MirrorDescent}}
\def\tavg{t_{\text{avg}}}
\def\RIRO{\text{RIRO}}



\begin{algorithm*}[h]
	{\small
		\caption{Implementation details of \alg for IL with multiple oracles }
		\begin{algorithmic} [1] \label{alg:mamba details}
			\renewcommand{\algorithmicensure}{\textbf{Input:}}
			\renewcommand{\algorithmicrequire}{\textbf{Output:}}
			\ENSURE Initial learner policy $\pi_1$, oracle polices $\{\pi^k\}_{k\in[K]}$, function approximators $\{\what{V}^k\}_{k\in[K]}$.
			\REQUIRE The best policy in $\{\pi_1, \dots, \pi_N\}$.

			\STATE For each  $k\in[K]$, collect data $\DD^k$ by rolling out $\pi^k$ for the full problem horizon.
			\STATE Update value models $\what{V}^k = \ValueUpdate(\DD^k)$ for $k\in[K]$.
			\STATE$\pi_1 = \UpdateWhitening(\pi_1,\bigcup \DD^k )$

			\FOR {$n = 1\dots N-1$}

			\STATE Sample a trajectory using $\pi_n$ to collect data $\DD_n'$.

			\STATE Let $\ts = \SampleSwitchTime(\tavg) \in [T-1]$, where $\tavg$ is the average trajectory length of $\DD_{1:n-1}'$ (for the first iteration set $\tavg=0$). Sample a RIRO trajectory using $\pi_n$ and then $\pi^k$ after $t\geq\ts$ to collect data $\DD_{\RIRO}$, where $k$ is uniformly sampled in $[K]$.
			If the sampled trajectory in $\DD_{\RIRO}$ is longer than $\ts$, aggregate the trajectory after $\ts$ in $\DD_{\RIRO}$ into $\DD^k$ with importance weight $\frac{1}{Tp(\ts)}$. Otherwise, aggregate $\DD_{\RIRO}$ into $\DD_n'$.
			\STATE Update value model $\what{V}^k = \ValueUpdate(\DD^k)$.
			\STATE Let $\pi_n' = \UpdateWhitening(\pi_n, \DD_n')$ and compute the sampled gradient based on $\DD_n'$ with one-step importance sampling as
			\begin{align*}
				g_n  &= -\sum_{t=0}^{T-1} \nabla \log\pi_n'(a_t|s_t) \frac{\pi_n'(a_t|s_t)}{\pi_n(a_t|s_t)}\left(\sum_{\tau=t}^{T-1} \lambda^{\tau-t} \left(r(s_\tau,a_\tau) + \fmaxh(s_{\tau+1}) - \fmaxh(s_\tau) \right) \right)
			\end{align*}
			where the recursion starts with $\what{V}^k(s_T)=0$ and $\fmaxh(s) = \max_{k\in[K]} \what{V}^k(s)$.
			\STATE Update the policy $\pi_{n+1} = \UpdatePolicy(\pi_n', g_n)$.

			\ENDFOR
		\end{algorithmic}
	}
\end{algorithm*}

We provide the details of \alg in \cref{alg:mamba} as \cref{alg:mamba details}, which closely follows \cref{alg:mamba} with a few minor, practical modifications which we describe below:
\begin{itemize}
	\item The $\UpdateWhitening$ function keeps a moving average of the first and the second moment of the states it has seen, which is used to provide whitened states (by subtracting the estimated mean and dividing by the estimated standard deviation) as the input to the neural network policy.

	\item In \cref{alg:mamba details}, $\ts = \SampleSwitchTime(\tavg)$  samples $\ts$ based on a geometric distribution whose mean is $\tavg$, because in the learner might not always be able to finish the full $T$ steps. The trajectory data are therefore suitably weighted by an importance weight  $\frac{1}{Tp(\ts)}$ to correct for this change from using the uniform distribution in sampling for $\ts$.
\end{itemize}
Apart from these two changes, \cref{alg:mamba details} follows the pseudo-code in \cref{alg:mamba}.

We provide additional details of the experimental setups below.
\begin{itemize}
	\item Time is appended as a feature in addition to the raw state, i.e. $s = (t,\bar{s})$.
	\item The policy is a Gassian with mean modeled by a $(128,128)$ FCNN (fully connected neural network), and the standard deviation is diagonal, learnable and independent of the state. The value function is estimated by a $(256,256)$ FCNN. In these FCNNs, the activation functions are $tanh$ except the last layer is linear. The policy and the value networks are parameterized independently.

	\item $\ValueUpdate$ performs weighted least-squared regression by first whitening the input and then performing 100 (CartPole, DoubleInvertedPendulum) or 800 (HalfCheetah, Ant) steps of ADAM with a batchsize of 128 samples and step size 0.001.
	The target is constructed by Monte-Carlo estimate and one-step importance sampling, if necessary.

	\item $\UpdatePolicy$ is set to be either ADAM~\citep{kingma2014adam} or Natural Gradient Descent (NGD)~\citep{amari1998natural}. We adopt the default hyperparamters of ADAM ($\beta_1=0.9$ and $\beta_2=0.99$) and a stepsize $0.001$. For NGD, we adopt the  ADAM-style adaptive implementation described in \citep[Appendix C.1.4]{cheng2019predictor} using $\beta_2=0.99$ and a stepsize of $0.1$.

	\item $\DD^k$ is limited to data from the past 100 (CartPole, DoubleInvertedPendulum) or 2 (HalfCheetah, Ant) iterations. Policy gradient always keeps 2 iterations of data.

	\item All \alg, \pg, and \aggd follow the protocol in \cref{alg:mamba details}. In the pre-training phase (lines 1-3), each oracle policy (or the learner policy for \pg) would perform 16 (CartPole, DoubleInvertedPendulum) or 512 (HalfCheetah, Ant) rollouts to collect the initial batch of data to train its value function estimator. For HalfCheetah and Ant, these data are also used to pretrain the learner policies by behavior cloning\footnote{We found that without behavior cloning \aggd would not learn to attain the oracle's performance, potentially due to high dimensionality. Therefore, we chose to perform behavior cloning with the best oralce policy as the initialization for all IL algorithms in comparison here.}
	In every iteration, each algorithm would perform 8 (CartPole, DoubleInvertedPendulum) or 256 (HalfCheetah, Ant) rollouts: For \alg and \aggd, half rollouts are used to collect data for updating the value functions (line 6) and half rollouts (line 5) are for computing the gradient. For \pg, all  rollouts are performed by the learner; they are just used to compute the gradient and then update the value function (so there is no correlation).

	\item Additional 8 (CartPole, DoubleInvertedPendulum) or 256 (HalfCheetah, Ant) rollouts are performed to evaluate the policy's performance, which generate the plots presented in the paper.

	\item All environments have continuous states and actions. CartPole and DoubleInvertedPendulum have a problem horizon of 1000 steps. HalfCheetah and Ant have a problem horizon of 500 steps.
	For CartPole, the dimensions of the state and the action spaces are 4 and 1, respectively. For DoubleInvertedPendulum,  the dimensions of the state and the action spaces are 8 and 1, respectively.
	For HalfCheetah, the dimensions of the state and the action spaces 17 and 6, respectively. Finally, for Ant, the dimensions of the state and the action spaces 111 and 8, respectively.
\end{itemize}

\paragraph{Computing Infrastructure}
The CartPole and DoubleInvertedPendulum experiments were conducted using Azure Standard F64s\_v2,
which was based on the Intel Xeon Platinum 8168 processor with 64 cores (base clock 3.4 GHz; turbo clock 3.7 GHz) and 128 GB memory. The HalfCheetah and Ant experiments were conducted using Azure Standard HB120rs\_v2, which was based on the AMD EPYC 7002-series processor with 120 cores (base clock 2.45 GHz; turbo clock 3.1 GHz) and 480 GB memory.
No GPU was used.
The operating system was Ubuntu 18.04.4 LTS. The prototype codes were based on Python and Tensorflow 2.
Using a single process in the setup above, a single seed of the CartPole experiment (100 iterations) took about 27 minutes to 45 minutes to finish, whereas a single seed of the DoubleInvertedPendulum experiment (200 iterations) took about 110 minutes to 125 minutes to finish. For HalfCheetah, 8 cores were used for each seed (1000 iterations) and the experiments took about 19.5 to 30.7 hours. For Ant, 8 cores were used for each seed (400 iterations) and the experiments took about 8 to 12 hours.

\paragraph{Hyperparameter Selection}
For CartPole and DoubleInvertedPendulum, we only performed a rough search of the step sizes of ADAM ($0.01$ vs $0.001$) and Natural Gradient Descent ($0.1$ vs $0.01$). We conducted experiments with different $\lambda$ and number of oracles in order to study their effects on \alg. The main paper presents the results of the pilot study: we first investigated the effect of $\lambda$ by comparing \alg with \aggd and concluded with a choice of $\lambda=0.9$; using this $\lambda$ value, we then studied the effects of the number of oracles. For completeness, we present and discuss results of all the hyperparamters for CartPole and DoubleInvertedPendulum below.
For HalfCheetah, we did a hyperparamter search over the size of replay buffer and the optimization steps for value function fitting such that \aggd can achieve the oracle-level performance. Once that is chose, we apply the same hyperparamters to other algorithms.
For Ant, we simply apply the same hyperparamters used in HalfCheetah.

\paragraph{Oracle Performance}
For Cartpole and DoubleInvertedPendulum, please see the detailed discussion in the next section. For HalfCheetah, the four oracles used in the experiments have scores of 1725.80, 1530.85, 1447.84, and 1395.36, respectively. For Ant, the four oracles used in the experiments have scores of 1050.57, 1038.03, 883.18, 775.70, respectively.

\subsection{Additional Experimental Results of Hyperparamter Effects}

\begin{figure}
	\centering
	\hspace{-15mm}
	\begin{subfigure}{0.5\textwidth}
		\centering
		\includegraphics[width=0.8\textwidth]{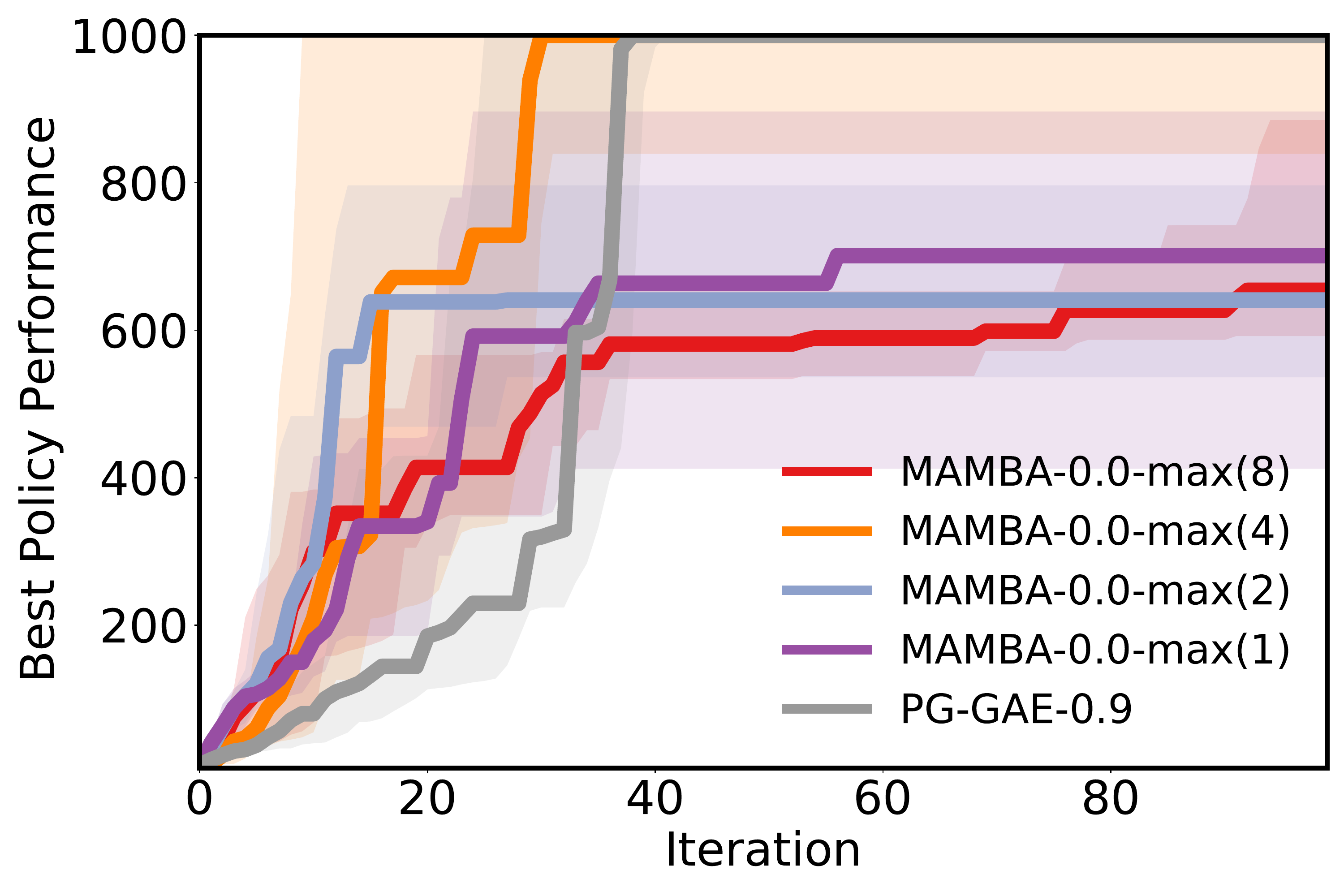}
		\caption{\small $\lambda=0$}
	\end{subfigure}
	\hspace{-10mm}
	\begin{subfigure}{0.5\textwidth}
		\centering
		\includegraphics[width=0.8\textwidth]{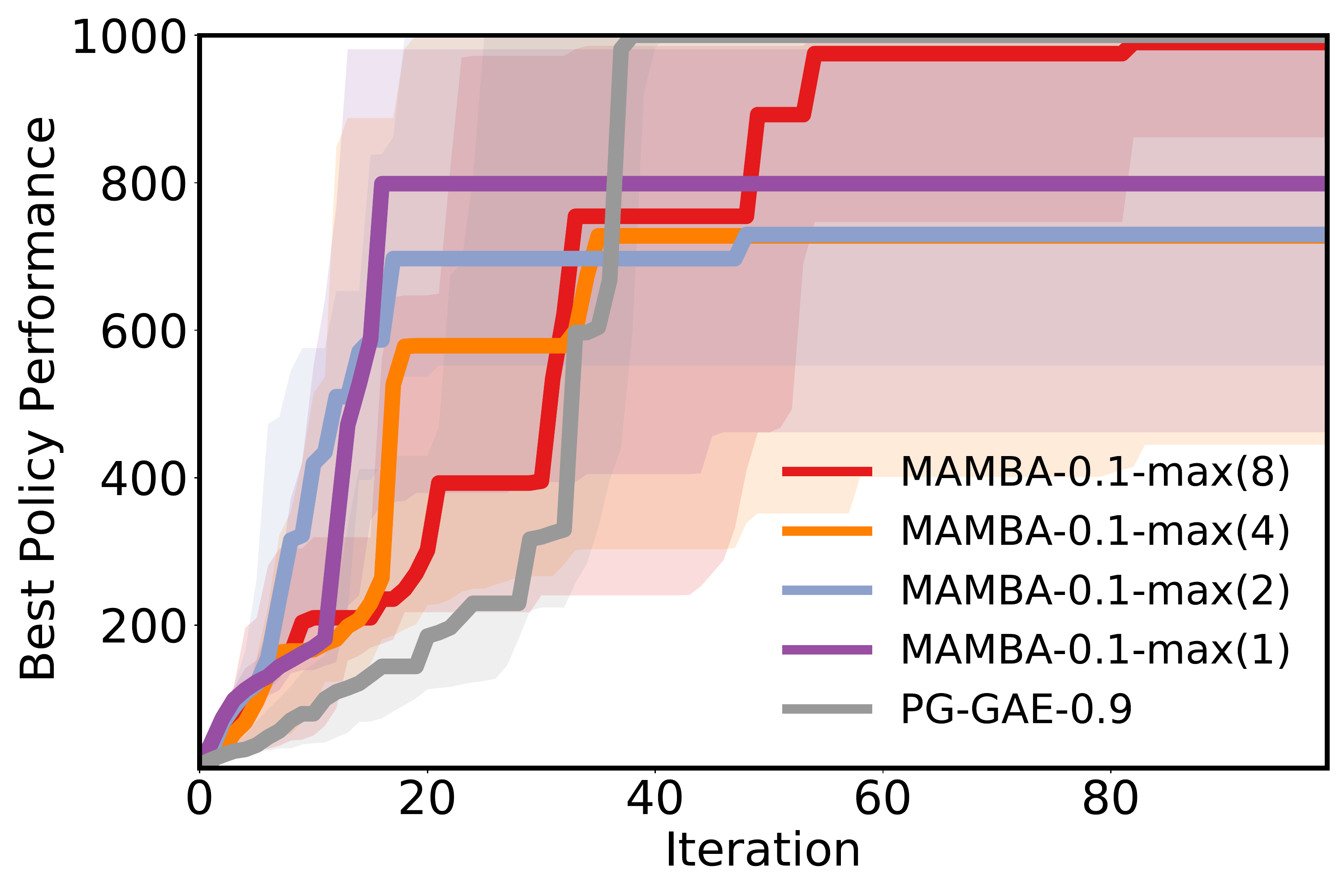}
		\caption{\small $\lambda=0.1$}
	\end{subfigure}
	\hspace{-15mm}
	\\
	\hspace{-15mm}
	\begin{subfigure}{0.5\textwidth}
		\centering
		\includegraphics[width=0.8\textwidth]{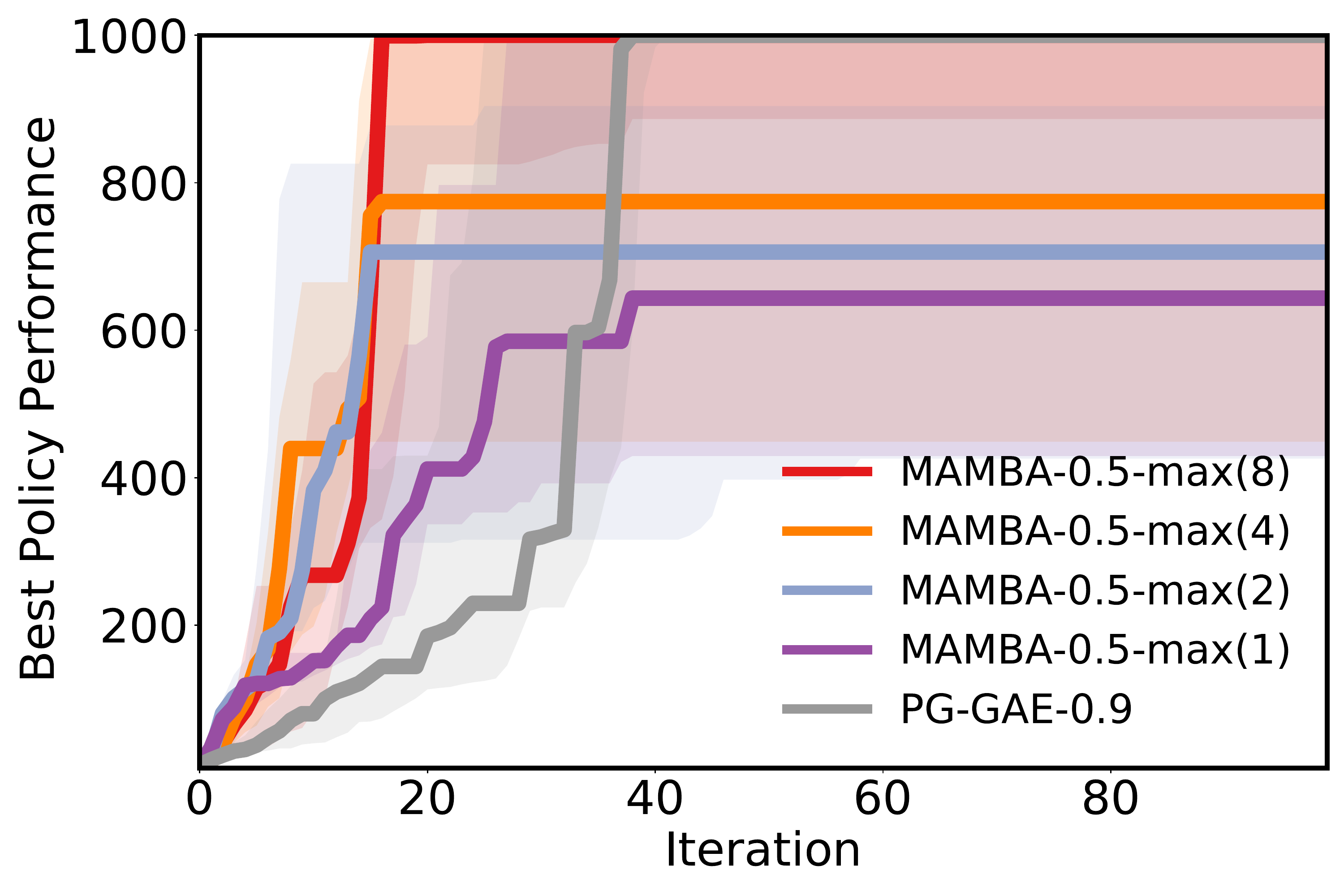}
		\caption{\small $\lambda=0.5$}
	\end{subfigure}
	\hspace{-10mm}
	\begin{subfigure}{0.5\textwidth}
		\centering
		\includegraphics[width=0.8\textwidth]{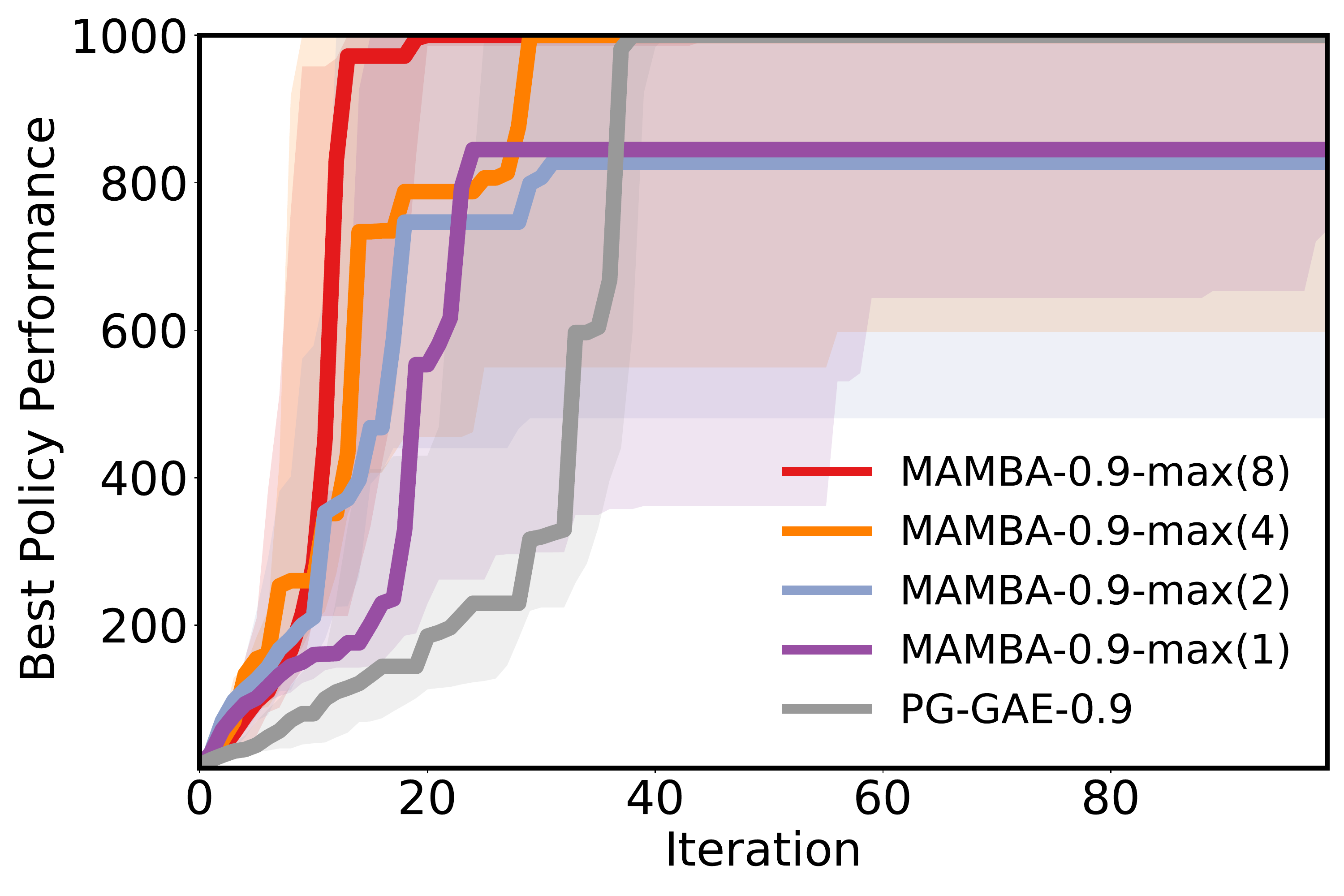}
		\caption{\small $\lambda=0.9$}
	\end{subfigure}
	\hspace{-15mm}
	\\
	\hspace{-25mm}
	\begin{subfigure}{0.5\textwidth}
		\centering
		\includegraphics[width=0.8\textwidth]{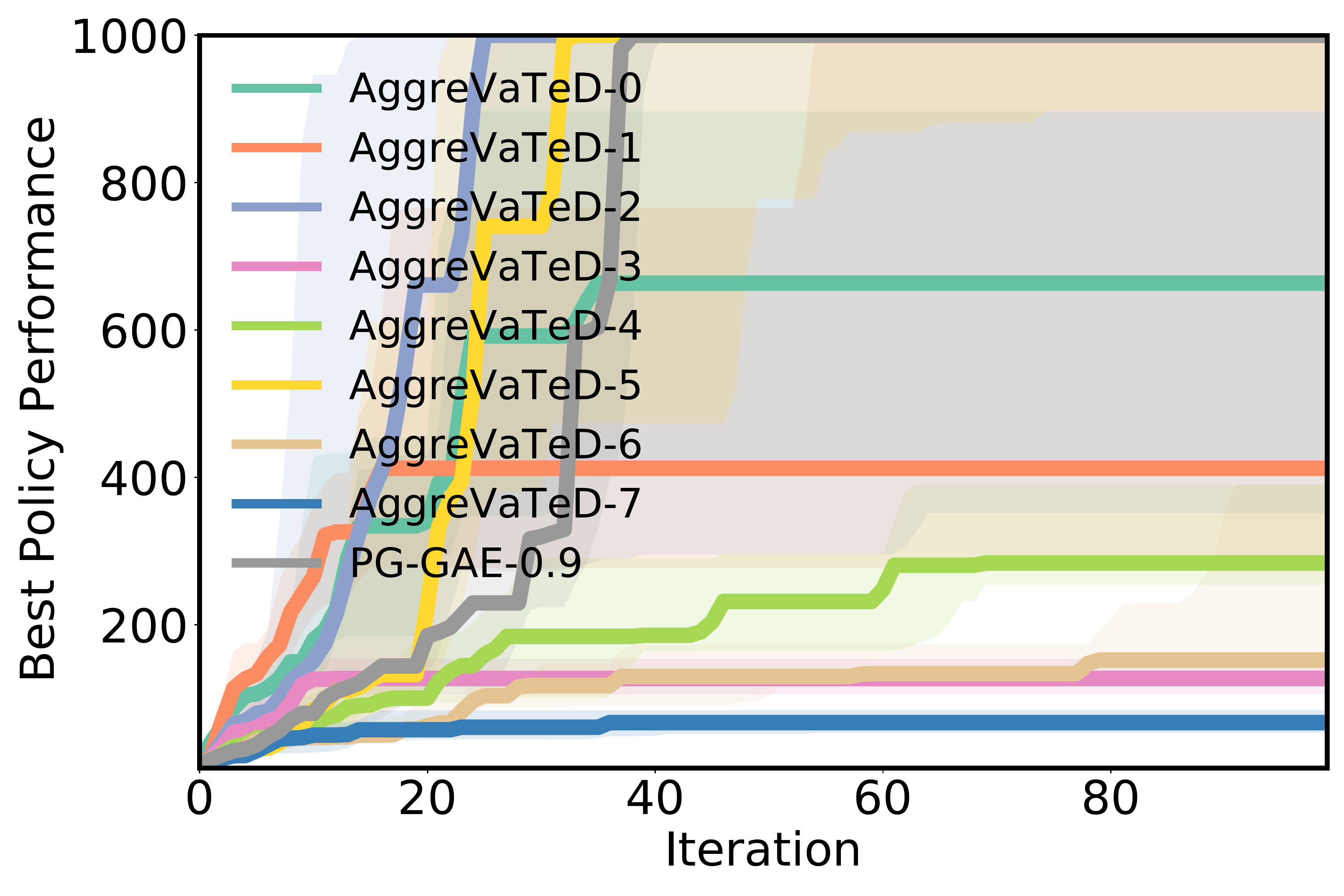}
		\caption{\small \aggd}
	\end{subfigure}
	\hspace{-10mm}
	\begin{subfigure}{0.5\textwidth}
		\footnotesize
		\begin{center}
			\begin{tabular}{c|c }
			  &  Return \\[0.5ex]
			  \hline
			 oracle0 & 89.04\\
			 oracle1 & 61.73 \\
			 oracle2 & 34.78 \\
			 oracle3 & 26.34 \\
			 oracle4 & 18.19 \\
			 oracle5 & 16.36 \\
			 oracle6 & 10.38 \\
			 oracle7 & 9.78
			\end{tabular}
			\end{center}
			\caption{Oracle Performance}
	\end{subfigure}
	\hspace{-10mm}
	\caption{\small
			Performance of the best policies in CarlPole with Bad Oracles.
			(a)-(d) \alg with $\lambda=0, 0.1, 0.5, 0.9$ (e) \aggd with different oracles. (f) The return of each oracle policy.
			A curve shows an algorithm's median performance across 8 random seeds. The center and right figures use the same line colors for all methods.
			The shaded area shows 25th and 75th percentiles.
	}
	\label{fig:CartPole with bad oracles}
\end{figure}

\begin{figure}
	\centering
	\hspace{-15mm}
	\begin{subfigure}{0.5\textwidth}
		\centering
		\includegraphics[width=0.8\textwidth]{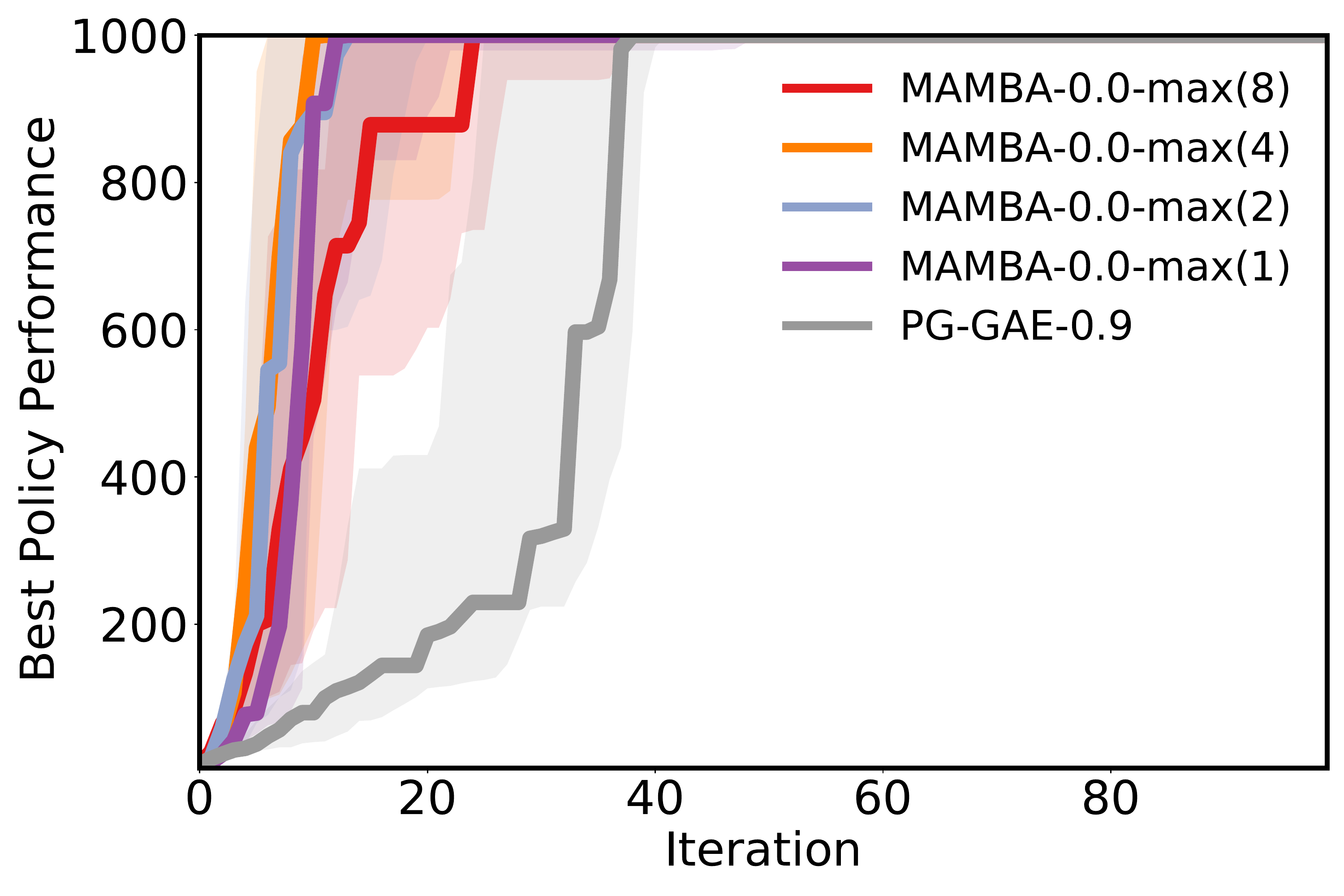}
		\caption{\small $\lambda=0$}
		\label{fig:oracles effects (cp)}
	\end{subfigure}
	\hspace{-10mm}
	\begin{subfigure}{0.5\textwidth}
		\centering
		\includegraphics[width=0.8\textwidth]{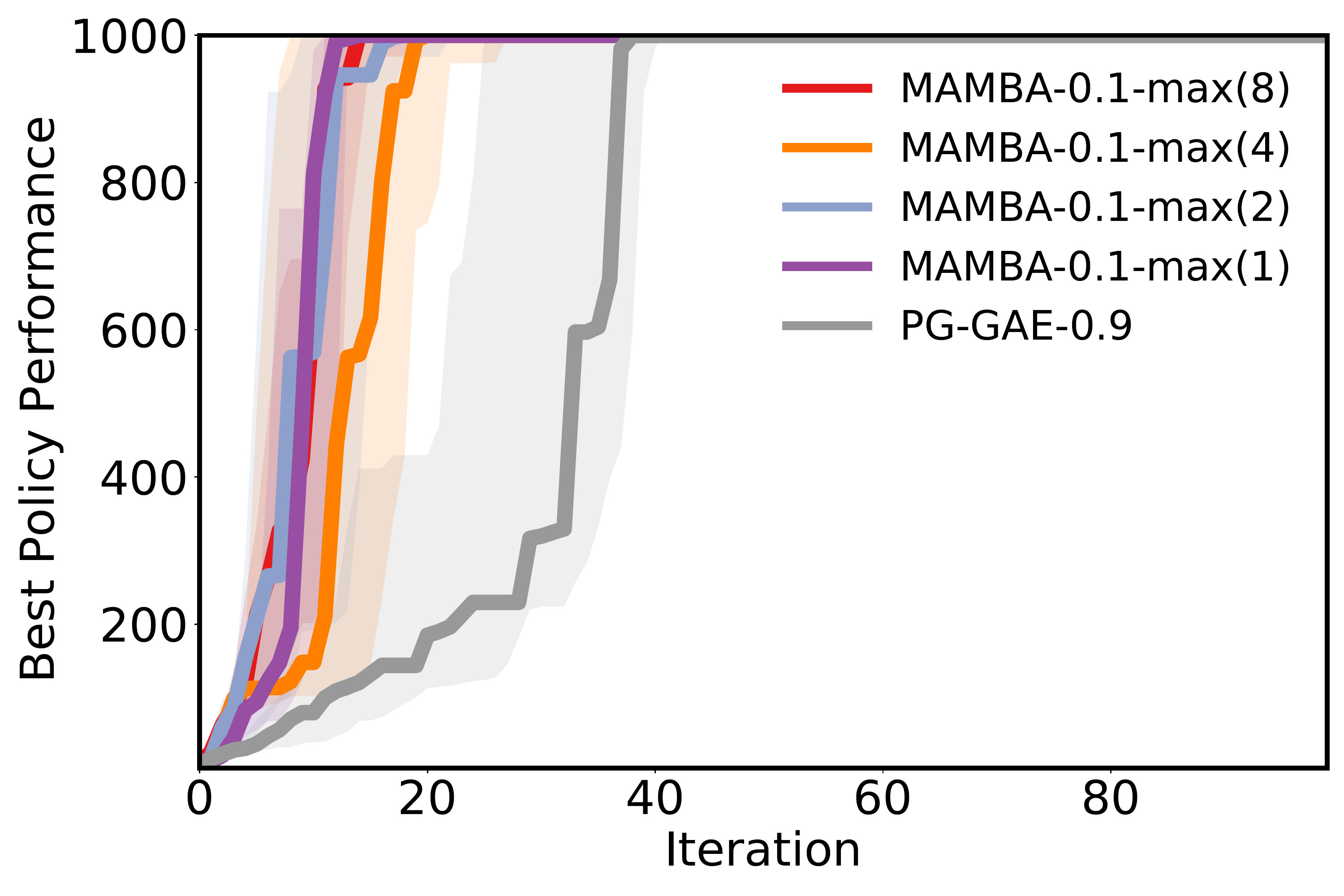}
		\caption{\small $\lambda=0.1$}
	\end{subfigure}
	\hspace{-15mm}
	\\
	\hspace{-15mm}
	\begin{subfigure}{0.5\textwidth}
		\centering
		\includegraphics[width=0.8\textwidth]{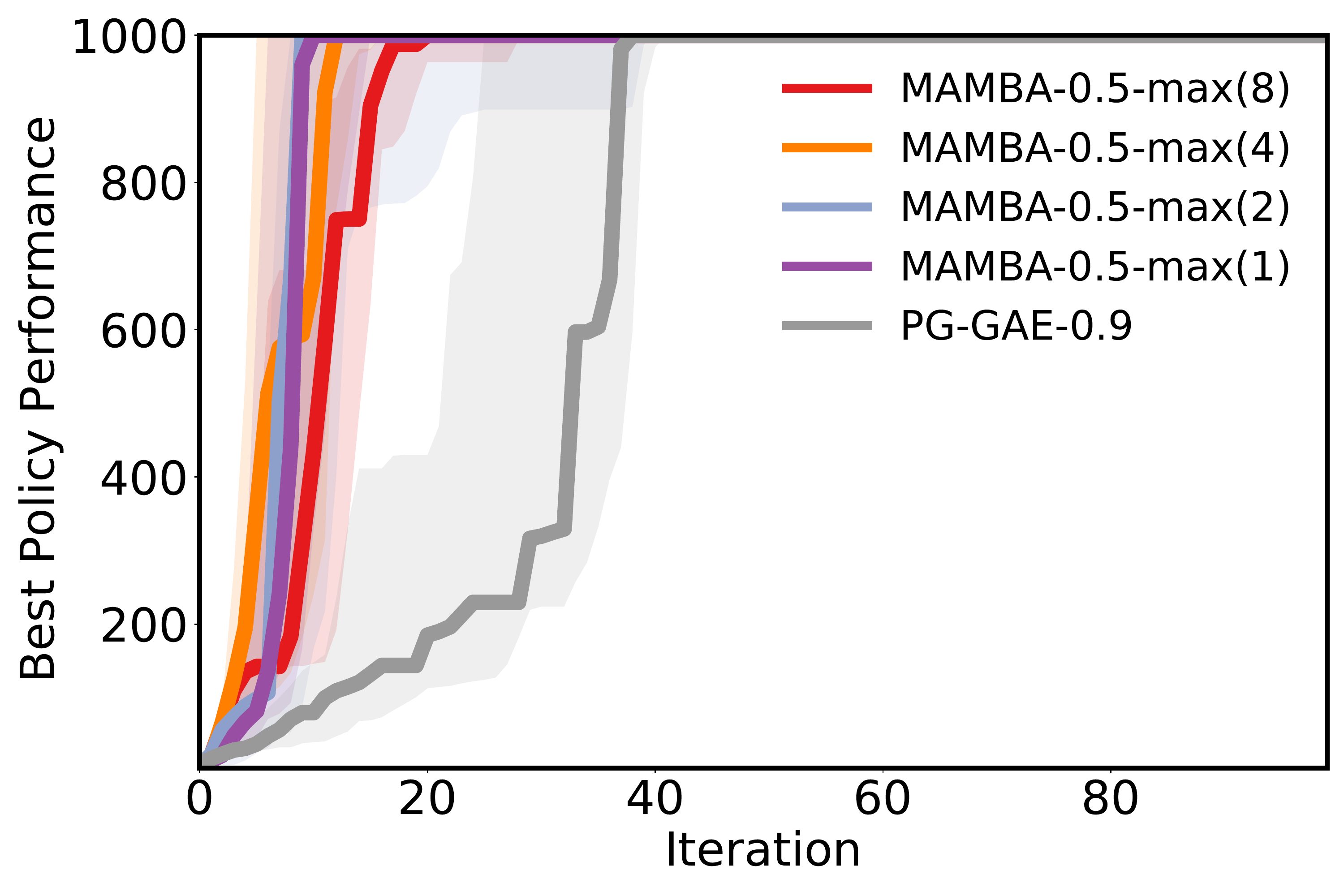}
		\caption{\small $\lambda=0.5$}
		\label{fig:lambda effects}
	\end{subfigure}
	\hspace{-10mm}
	\begin{subfigure}{0.5\textwidth}
		\centering
		\includegraphics[width=0.8\textwidth]{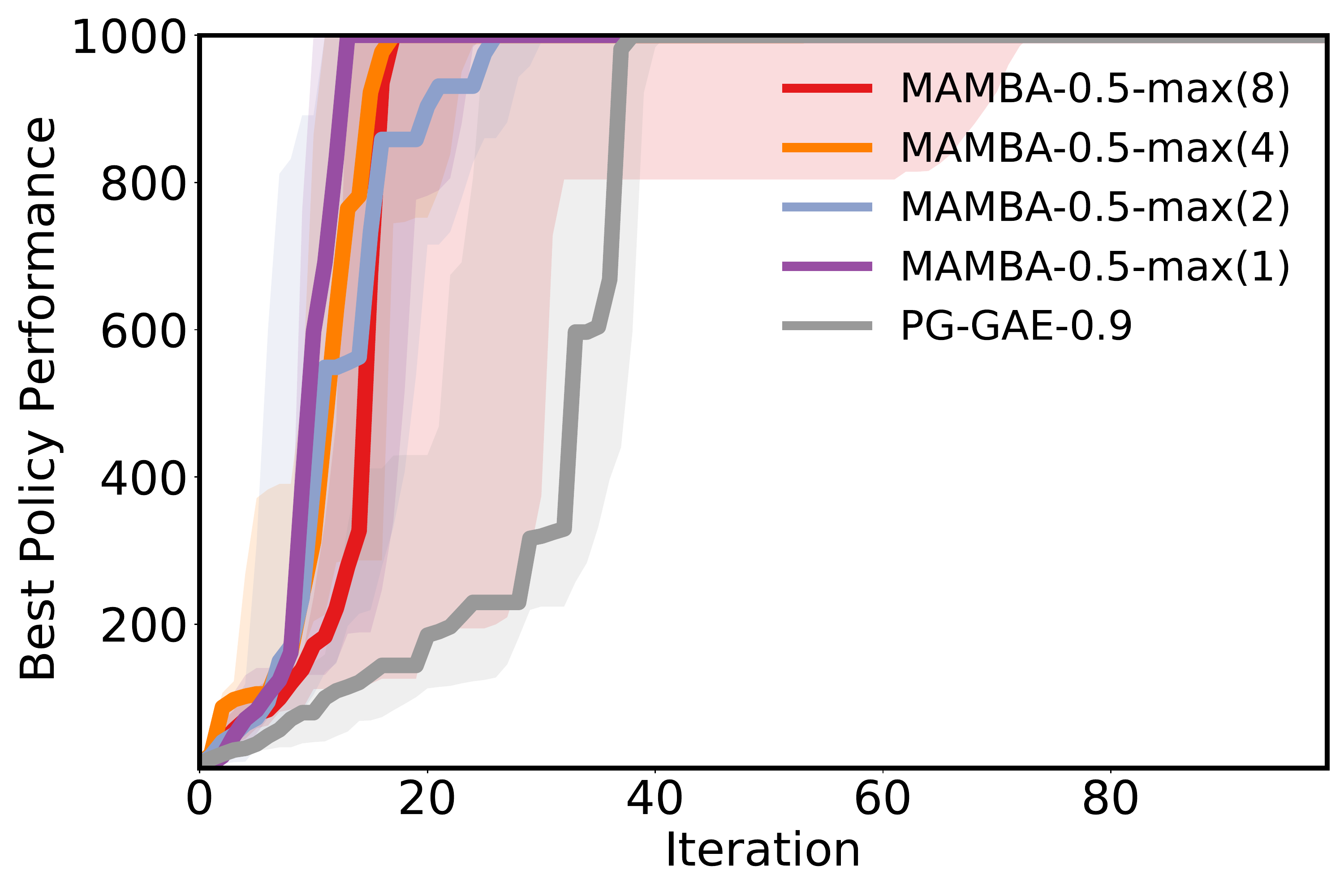}
		\caption{\small $\lambda=0.9$}
	\end{subfigure}
	\hspace{-15mm}
	\\
	\hspace{-25mm}
	\begin{subfigure}{0.5\textwidth}
		\centering
		\includegraphics[width=0.8\textwidth]{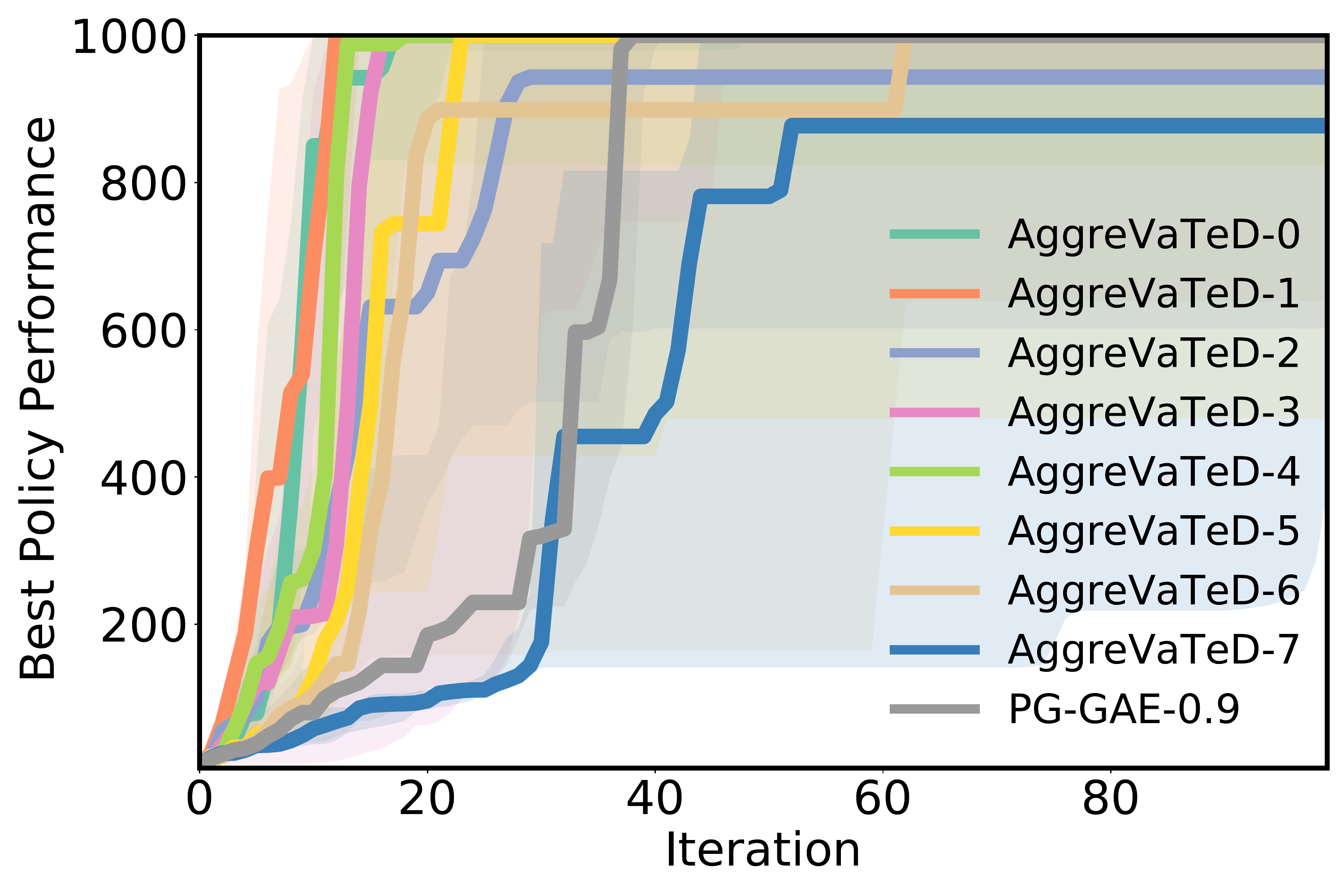}
		\caption{\small \aggd}
		\label{fig:lambda effects}
	\end{subfigure}
	\hspace{-10mm}
	\begin{subfigure}{0.5\textwidth}
		\footnotesize
		\begin{center}
			\begin{tabular}{c|c }
			  &  Return \\[0.5ex]
			  \hline
			 oracle0 & 411.40 \\
			 oracle1 & 226.89 \\
			 oracle2 & 102.80 \\
			 oracle3 & 82.24 \\
			 oracle4 & 70.10 \\
			 oracle5 & 23.14 \\
			 oracle6 & 23.10 \\
			 oracle7 & 21.68
			\end{tabular}
			\end{center}
			\caption{Oracle Performance}
	\end{subfigure}
	\hspace{-10mm}
	\caption{\small
			Performance of the best policies in CarlPole with Mediocre Oracles.
			(a)-(d) \alg with $\lambda=0, 0.1, 0.5, 0.9$ (e) \aggd with different oracles. (f) The return of each oracle policy.
			A curve shows an algorithm's median performance across 8 random seeds. The center and right figures use the same line colors for all methods.
			The shaded area shows 25th and 75th percentiles.
	}
	\label{fig:CartPole with mediocore oracles}
\end{figure}

\begin{figure}
	\centering
	\hspace{-15mm}
	\begin{subfigure}{0.5\textwidth}
		\centering
		\includegraphics[width=0.8\textwidth]{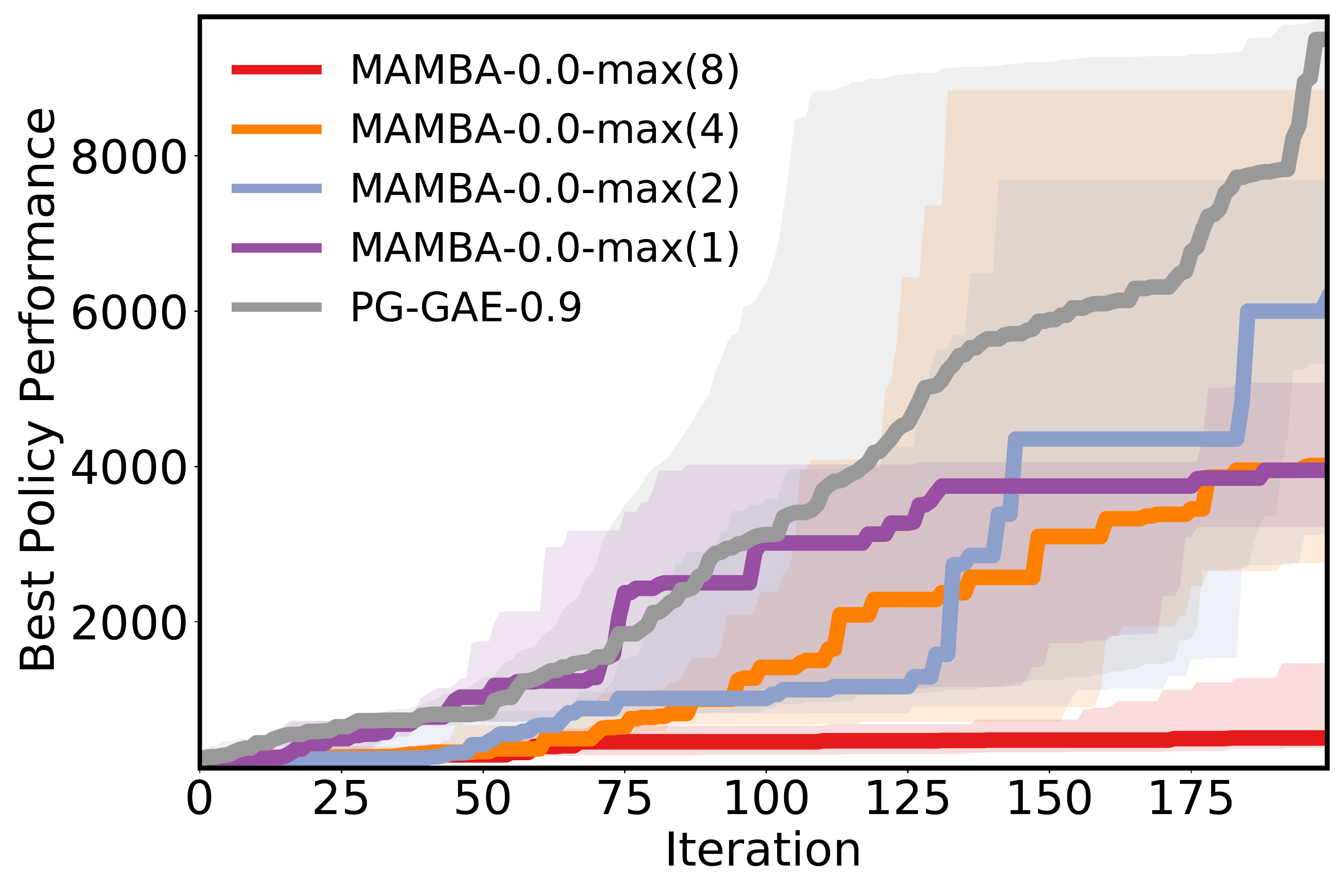}
		\caption{\small $\lambda=0$}
		\label{fig:oracles effects (cp)}
	\end{subfigure}
	\hspace{-10mm}
	\begin{subfigure}{0.5\textwidth}
		\centering
		\includegraphics[width=0.8\textwidth]{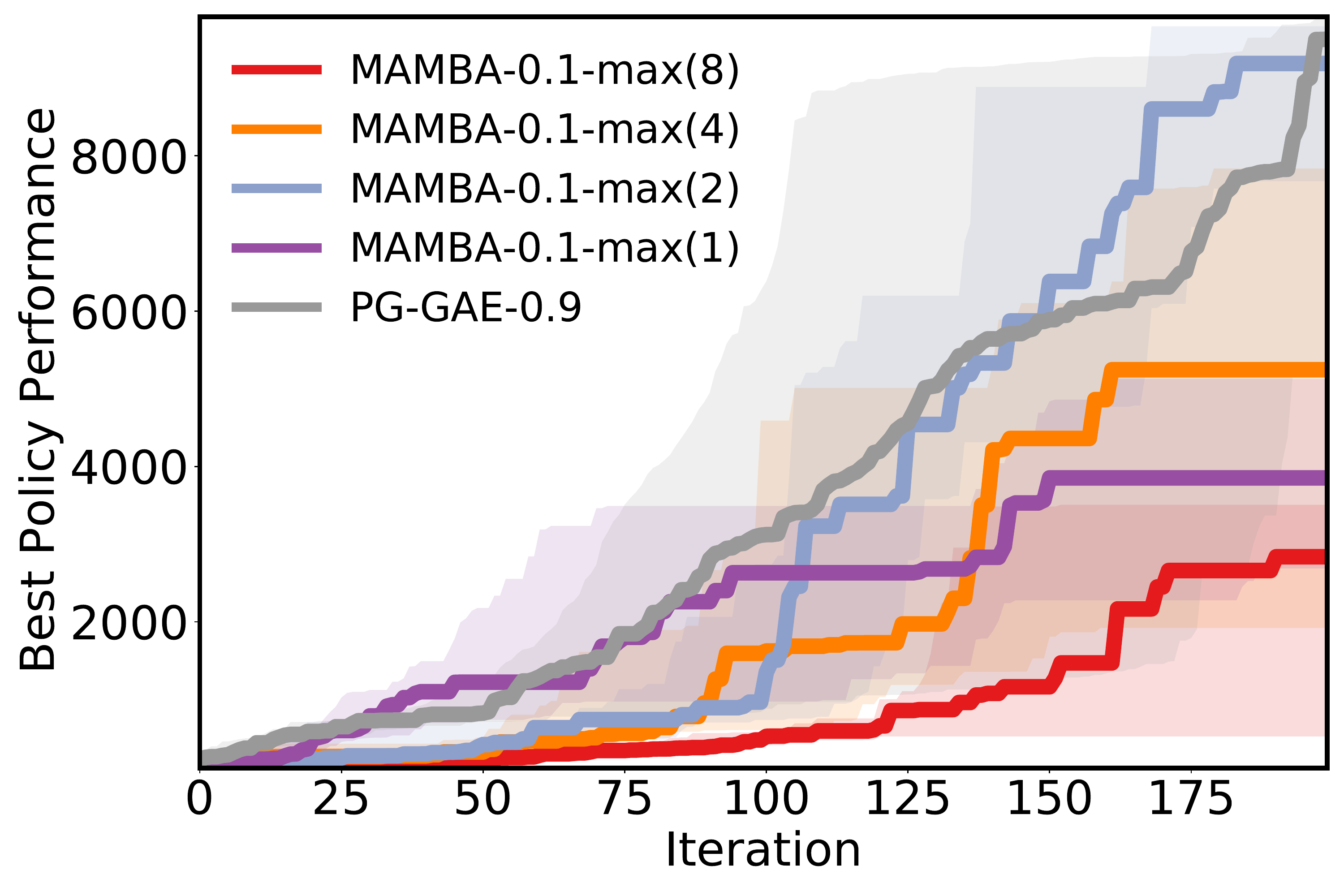}
		\caption{\small $\lambda=0.1$}
	\end{subfigure}
	\hspace{-15mm}
	\\
	\hspace{-15mm}
	\begin{subfigure}{0.5\textwidth}
		\centering
		\includegraphics[width=0.8\textwidth]{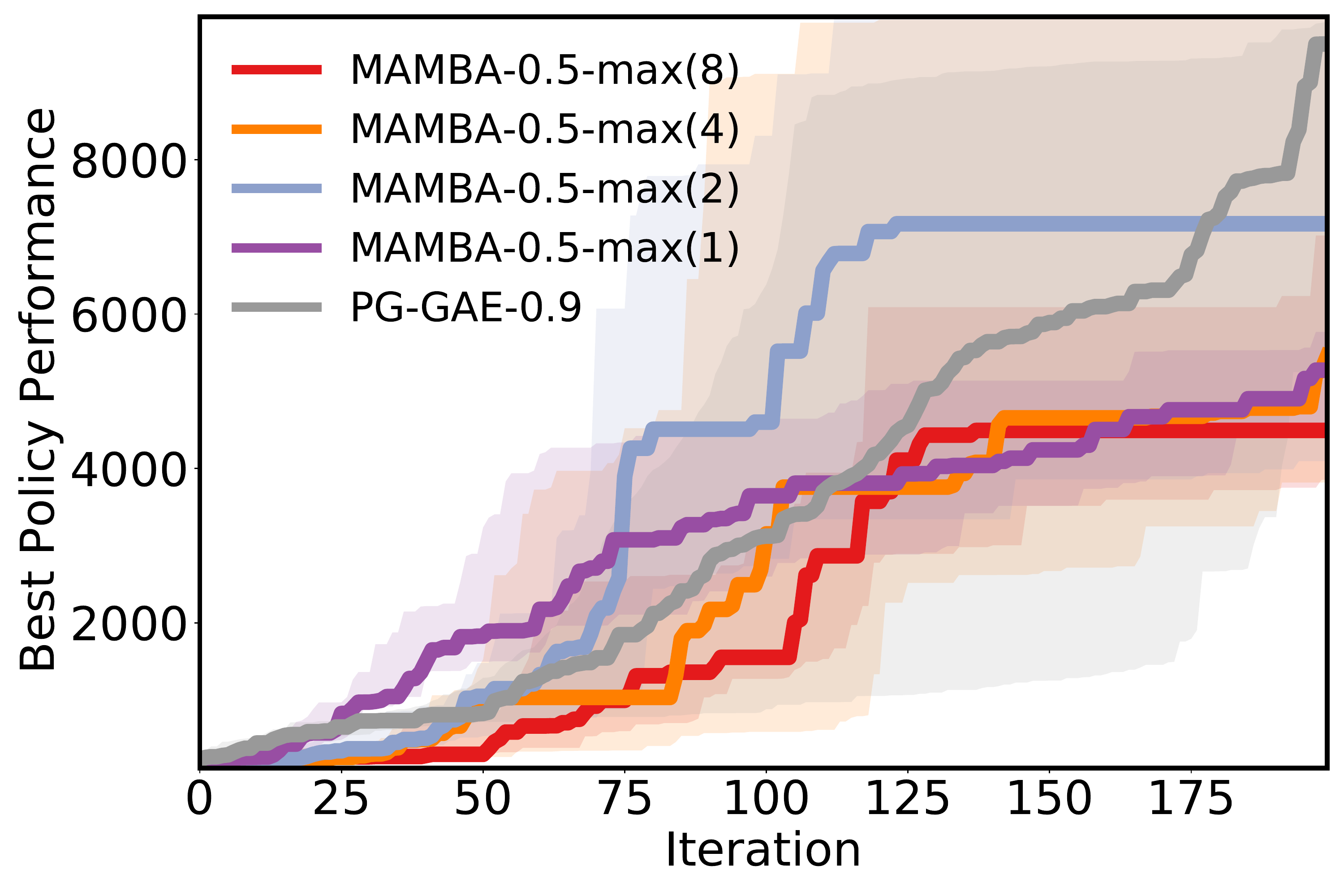}
		\caption{\small $\lambda=0.5$}
		\label{fig:lambda effects}
	\end{subfigure}
	\hspace{-10mm}
	\begin{subfigure}{0.5\textwidth}
		\centering
		\includegraphics[width=0.8\textwidth]{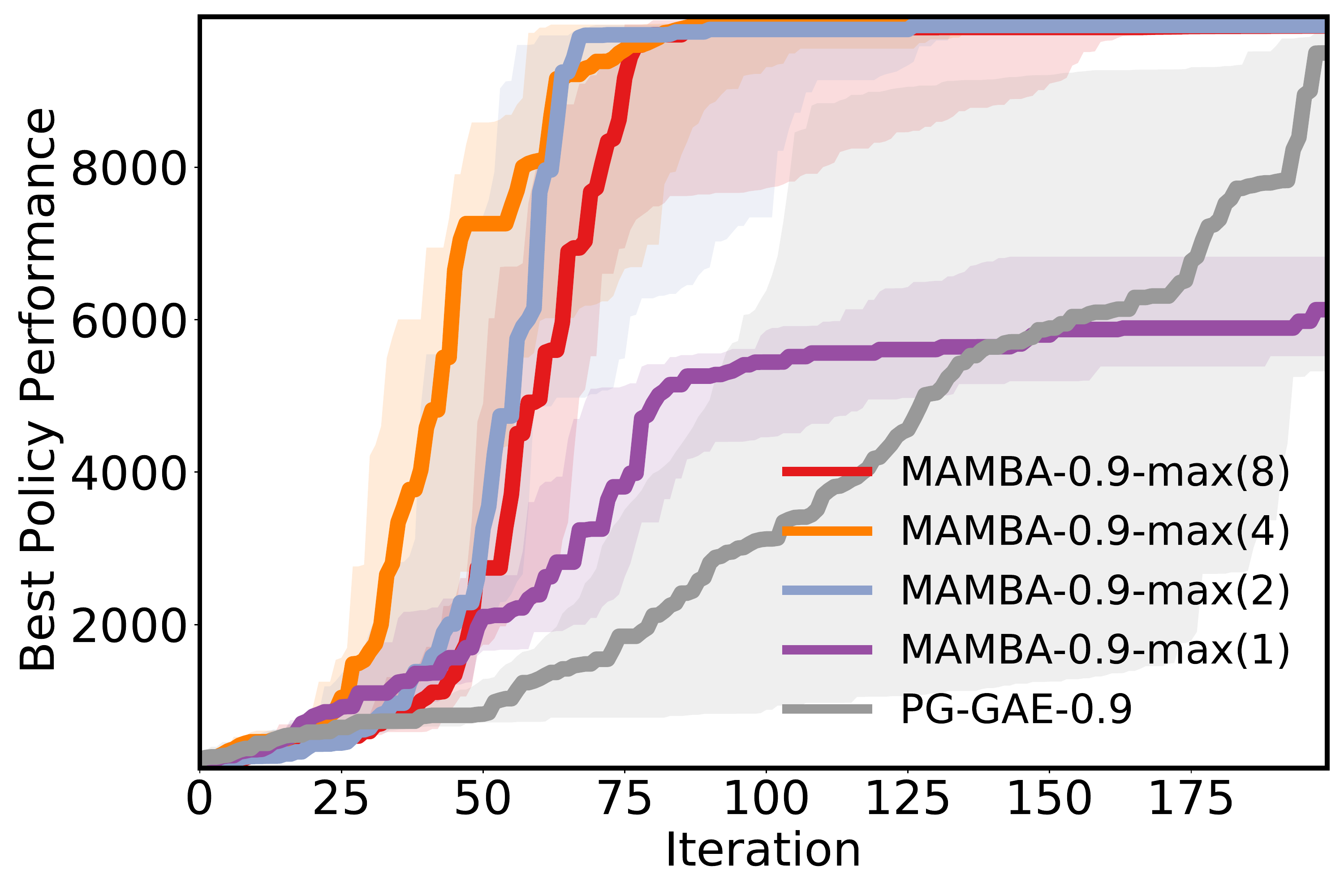}
		\caption{\small $\lambda=0.9$}
	\end{subfigure}
	\hspace{-15mm}
	\\
	\hspace{-25mm}
	\begin{subfigure}{0.5\textwidth}
		\centering
		\includegraphics[width=0.8\textwidth]{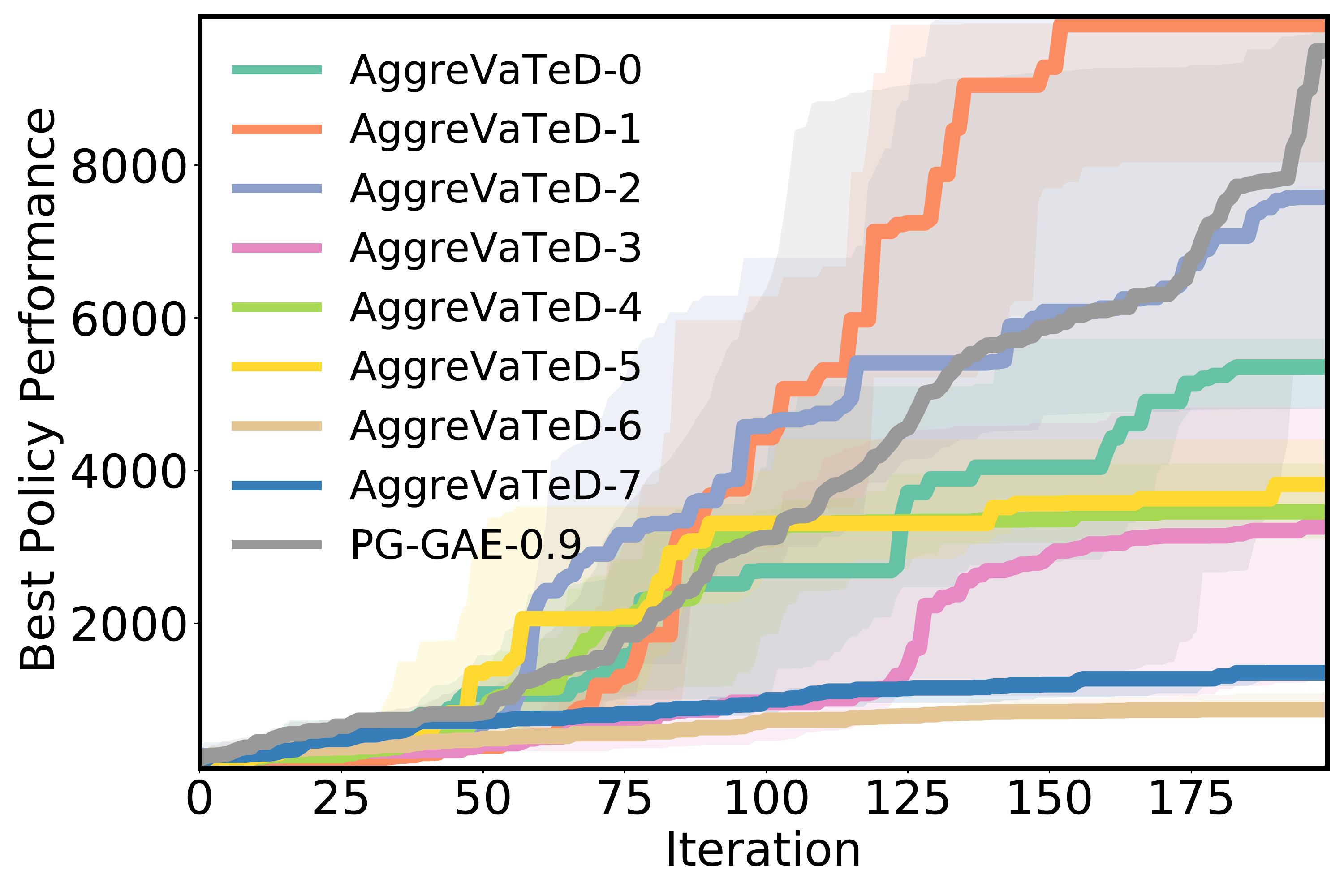}
		\caption{\small \aggd}
		\label{fig:lambda effects}
	\end{subfigure}
	\hspace{-10mm}
	\begin{subfigure}{0.5\textwidth}
		\footnotesize
		\begin{center}
			\begin{tabular}{c|c }
			  &  Return \\[0.5ex]
			  \hline
			 oracle0 & 4244.16\\
			 oracle1 & 3408.81 \\
			 oracle2 & 2775.02 \\
			 oracle3 & 2440.19 \\
			 oracle4 & 2329.90 \\
			 oracle5 & 2177.31 \\
			 oracle6 & 859.82 \\
			 oracle7 & 712.61
			\end{tabular}
			\end{center}
			\caption{Oracle Performance}
	\end{subfigure}
	\hspace{-10mm}
	\caption{\small
			Performance of the best policies in DoubleInvertedPendulum with Bad Oracles.
			(a)-(d) \alg with $\lambda=0, 0.1, 0.5, 0.9$ (e) \aggd with different oracles. (f) The return of each oracle policy.
			A curve shows an algorithm's median performance across 8 random seeds. The center and right figures use the same line colors for all methods.
			The shaded area shows 25th and 75th percentiles.
	}
	\label{fig:DoubleInvertedPendulum with bad oracles}
\end{figure}

\begin{figure}
	\centering
	\hspace{-15mm}
	\begin{subfigure}{0.5\textwidth}
		\centering
		\includegraphics[width=0.8\textwidth]{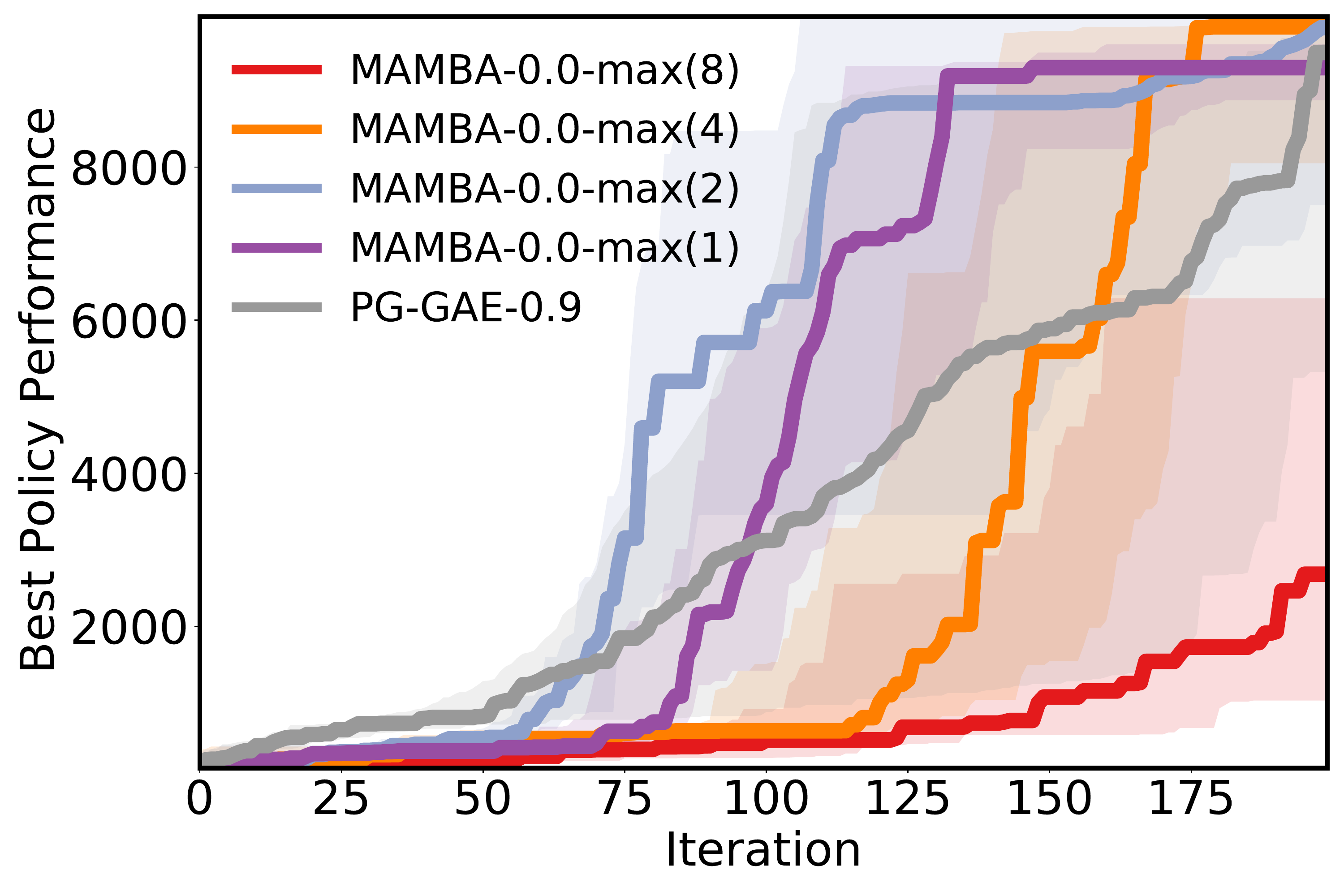}
		\caption{\small $\lambda=0$}
		\label{fig:oracles effects (cp)}
	\end{subfigure}
	\hspace{-10mm}
	\begin{subfigure}{0.5\textwidth}
		\centering
		\includegraphics[width=0.8\textwidth]{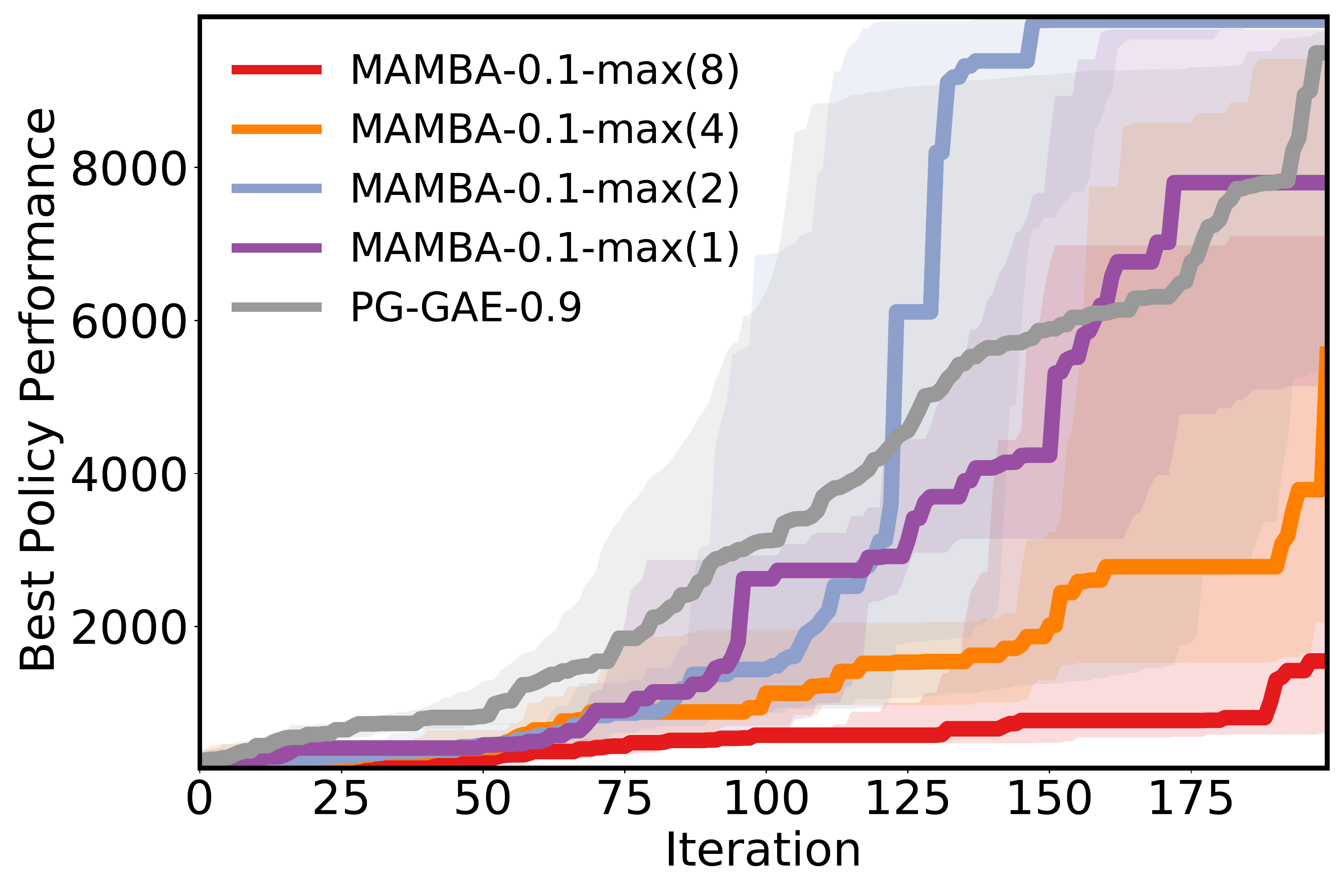}
		\caption{\small $\lambda=0.1$}
	\end{subfigure}
	\hspace{-15mm}
	\\[5mm]
	\hspace{-15mm}
	\begin{subfigure}{0.5\textwidth}
		\centering
		\includegraphics[width=0.8\textwidth]{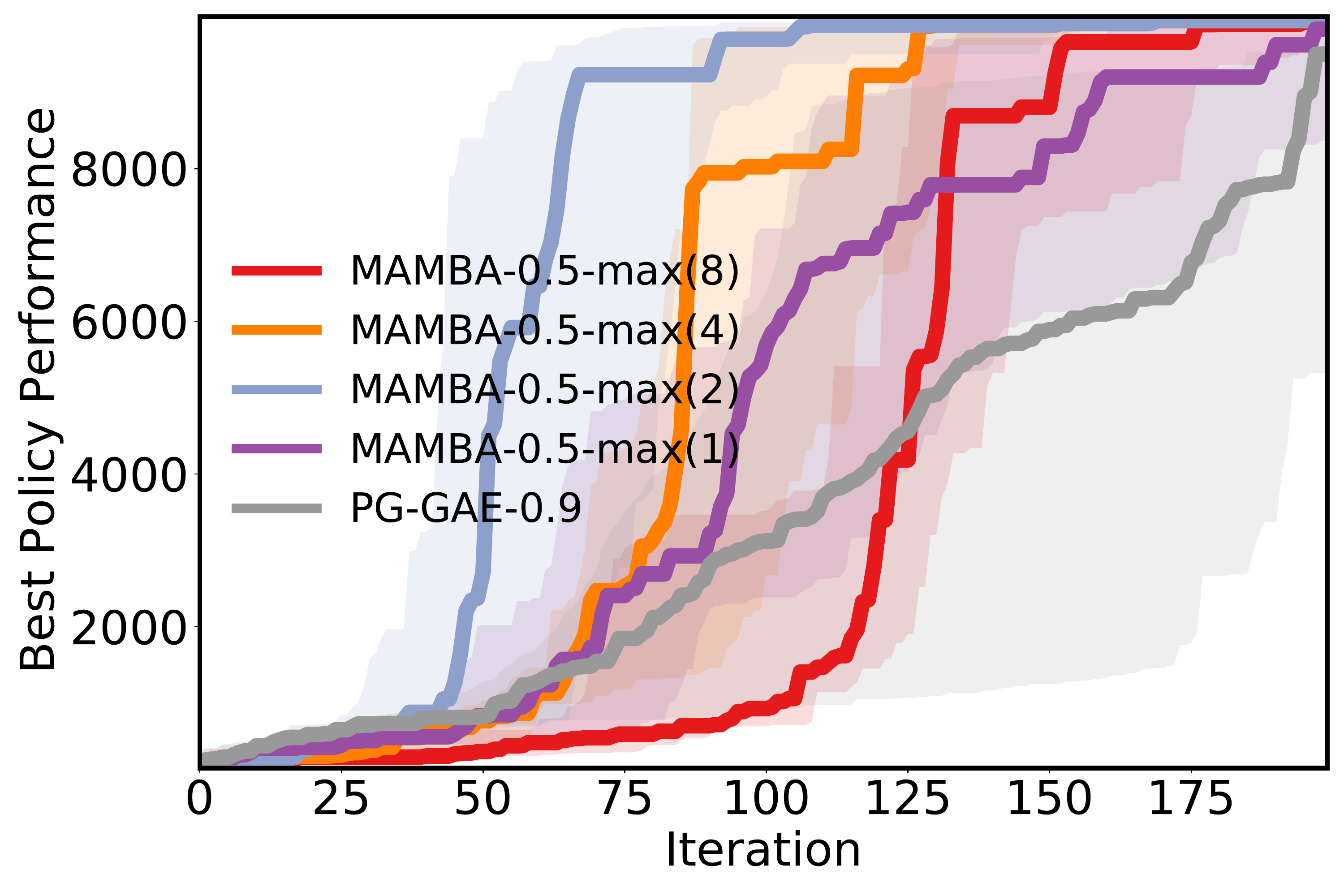}
		\caption{\small $\lambda=0.5$}
		\label{fig:lambda effects}
	\end{subfigure}
	\hspace{-10mm}
	\begin{subfigure}{0.5\textwidth}
		\centering
		\includegraphics[width=0.8\textwidth]{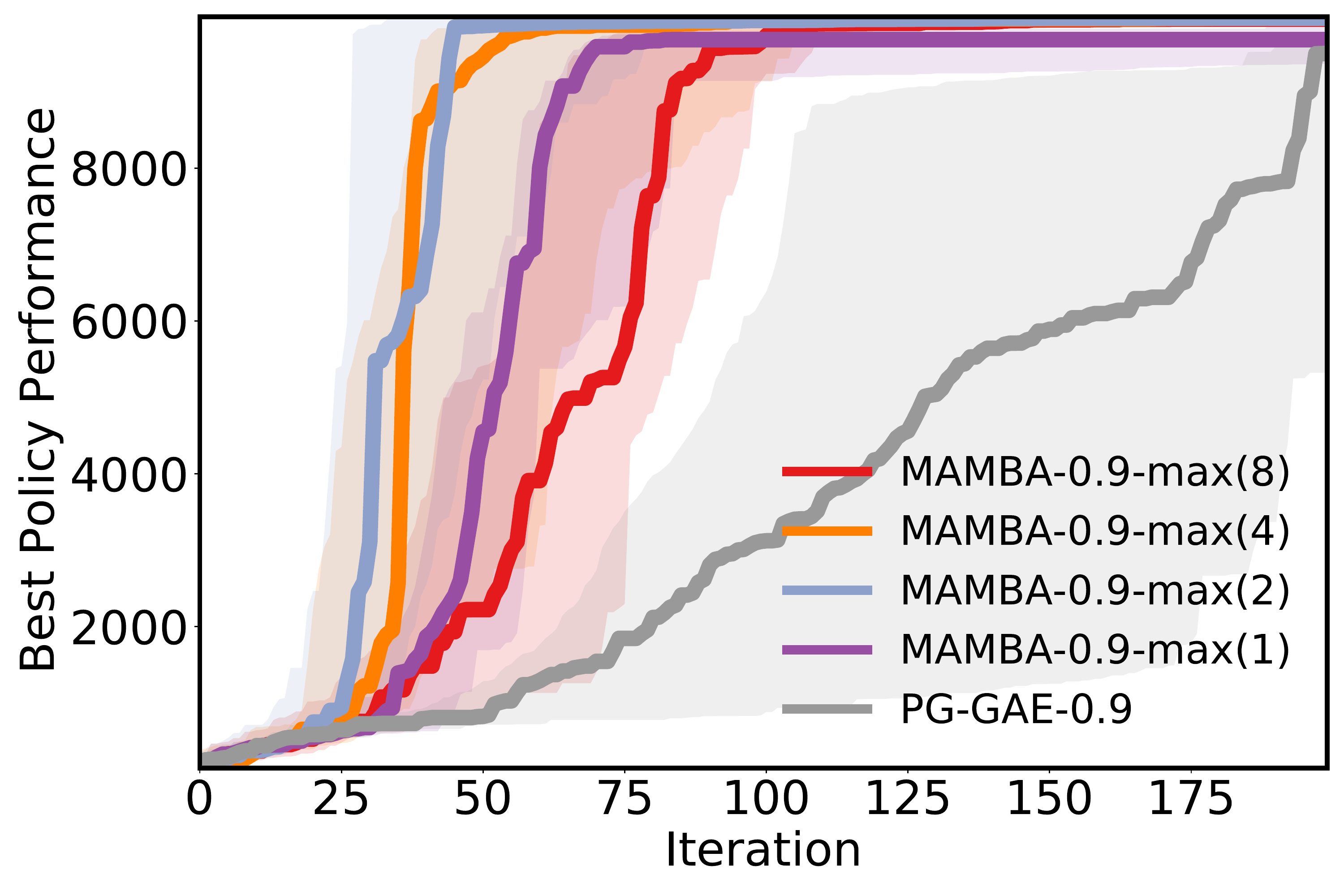}
		\caption{\small $\lambda=0.9$}
	\end{subfigure}
	\hspace{-15mm}
	\\[5mm]
	\hspace{-25mm}
	\begin{subfigure}{0.5\textwidth}
		\centering
		\includegraphics[width=0.8\textwidth]{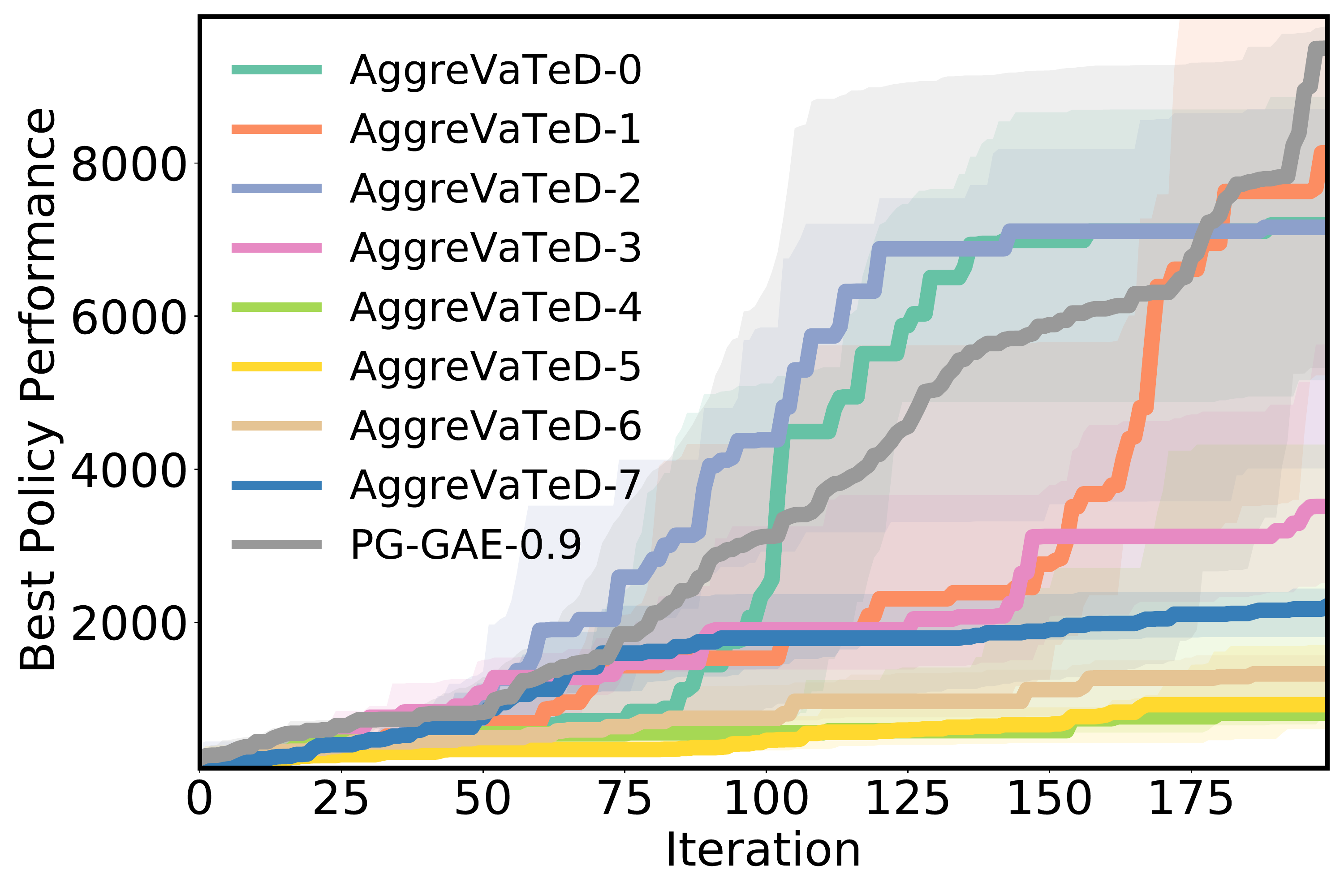}
		\caption{\small \aggd}
		\label{fig:lambda effects}
	\end{subfigure}
	\hspace{-10mm}
	\begin{subfigure}{0.5\textwidth}
		\footnotesize
		\begin{center}
			\begin{tabular}{c|c }
			  &  Return \\[0.5ex]
			  \hline
			 oracle0 & 9453.84 \\
			 oracle1 & 9079.43 \\
			 oracle2 & 6193.54 \\
			 oracle3 & 5791.93 \\
			 oracle4 & 4227.62 \\
			 oracle5 & 4206.01 \\
			 oracle6 & 1535.97 \\
			 oracle7 & 1480.38
			\end{tabular}
			\end{center}
			\caption{Oracle Performance}
	\end{subfigure}
	\hspace{-10mm}
	\caption{\small
			Performance of the best policies in DoubleInvertedPendulum with Mediocre Oracles.
			(a)-(d) \alg with $\lambda=0, 0.1, 0.5, 0.9$ (e) \aggd with different oracles. (f) The return of each oracle policy.
			A curve shows an algorithm's median performance across 8 random seeds. The center and right figures use the same line colors for all methods.
			The shaded area shows 25th and 75th percentiles.
	}
	\label{fig:DoubleInvertedPendulum with mediocre oracles}
\end{figure}

In this section, we include in \cref{fig:CartPole with bad oracles,fig:CartPole with mediocore oracles,fig:DoubleInvertedPendulum with bad oracles,fig:DoubleInvertedPendulum with mediocre oracles} additional experimental results of CartPole and DoubleInvertedPendulum environments. The purpose of these extra results is to provide a more comprehensive picture of the properties of \alg under various hyperparamter settings,. 

\paragraph{Setup}
For each of the environments (CartPole and DoubleInvertedPendulum), we conduct experiments with Bad Oracles (\cref{fig:CartPole with bad oracles,fig:DoubleInvertedPendulum with bad oracles}) and Mediocre Oracles (\cref{fig:CartPole with mediocore oracles,fig:DoubleInvertedPendulum with mediocre oracles}), where the results of the Bad Oracles are the ones presented in the main paper. In each experiment, we run \alg with $\lambda=0,0.1,0.5$ and $0.9$, and with the number of oracles varying between $1,2,4$ and $8$. In addition, we run \aggd with each of the oracles (whereas the main paper only presents the results of the oracles with the highest return). Finally, for baselines, we include the learning curve of \pg as well as the return of each oracle. Recall that the oracles are indexed in a descending order of their returns, which are estimated by performing 8 rollouts.

\subsubsection{Brittleness of \aggd} \label{sec:aggrevated brittleness}
First, the experiments of \aggd highlight that performing IL via policy improvement\footnote{\aggd is an approximate policy improvement algorithm~\citep{sun2017deeply}.} from the best oracle (in terms of the return) does not always lead to the best learner policy, as constantly shown in \cref{fig:CartPole with bad oracles,fig:CartPole with mediocore oracles,fig:DoubleInvertedPendulum with bad oracles,fig:DoubleInvertedPendulum with mediocre oracles}. In general there is no upper bound on the amount of performance change that policy improvement can make, because the policy improvement operator is myopic, only looking at one step ahead. As a result, running \aggd with the best oracle does not always give the best performance that can be made with an oracle chosen in the hindsight.
Another factor to the differences between the best foresight and hindsight oracles is that the return of each oracle is only estimated by 8 rollouts here.

Our experimental results show that such sensitivity is reduced in \alg: even in the single-oracle setting, using a $\lambda>0$ in \alg generally leads to more robust performance than \aggd using the same, best oracle, which is namely \alg with $\lambda=0$. We should remark that this robustness is not fully due to the bias-variance trade-off property~\citep{schulman2015high}, but also by large attributed to the incorporation of the multi-step information in online loss in \eqref{eq:online loss with lambda weights} (cf. \cref{th:mamba performance}). By using $\lambda>0$, \alg can start to see beyond one-step improvement and becomes less dependent on the oracle quality. In the experiments, we observe by picking a large enough $\lambda$, \alg with a single oracle usually gives comparable if not better performance than \aggd with the best policy chosen in the \emph{hindsight}.

\subsubsection{Effects of $\lambda$-weighting}
Beyond the single-oracle scenario discussed above, we see consistently in \cref{fig:CartPole with bad oracles,fig:CartPole with mediocore oracles,fig:DoubleInvertedPendulum with bad oracles,fig:DoubleInvertedPendulum with mediocre oracles} that using a non-zero $\lambda$ improves the performance of \alg.
The importance of $\lambda$ is noticeable particularly in setups with Bad Oracles, as well as in the experiments with the higher-dimensional environment DoubleInvertedPendulum. Generally, when the oracles are bad (as in \cref{fig:CartPole with bad oracles,fig:DoubleInvertedPendulum with bad oracles}), using a larger $\lambda$ provides the learner a chance to outperform the suboptimal oracles as suggested by \cref{th:mamba performance}, because the online loss function in \eqref{eq:online loss with lambda weights} starts to use multi-step information. On the other hand, when the top oracles' performance is better and the state space is not high-dimensional, as in CartPole with Mediocre Oracles in \cref{fig:CartPole with mediocore oracles}, the effects of $\lambda$ is less prominent. The usage of $\lambda>0$ also helps reduce the dependency on function approximation error, which is a known GAE property~\citep{schulman2015high}, as we see in the experiments with DoubleInvertedPendulum in \cref{fig:DoubleInvertedPendulum with bad oracles,fig:DoubleInvertedPendulum with mediocre oracles}.

\subsubsection{Effects of multiple oracles}
Using more than one oracles generally lead to better performance across \cref{fig:CartPole with bad oracles,fig:CartPole with mediocore oracles,fig:DoubleInvertedPendulum with bad oracles,fig:DoubleInvertedPendulum with mediocre oracles}. In view of \cref{th:mamba performance}, using more oracles can improve the quality of the baseline value function, though at the cost of having a higher bias in the function approximators (because more approximators need to be learned). We see that the benefit of using more oracles particularly shows up when higher values of $\lambda$ are used; the change is smaller in the single-oracle settings.

However, in the settings with Mediocre Oracles in \cref{fig:CartPole with mediocore oracles,fig:DoubleInvertedPendulum with mediocre oracles}, increasing the number of oracles beyond a certain threshold degrades the performance of \alg. Since a fixed number of rollouts are performed in each iteration, having more oracles implies that the learner would need to spend more iterations to learn the oracle value functions to a fixed accuracy. In turn, this extra exploration reflects as slower policy improvement. Especially, because using more oracles here means including strictly weaker oracles, this phenomenon is visible, e.g., in \cref{fig:DoubleInvertedPendulum with mediocre oracles}.

\begin{figure}[h!]
	\centering
	\hspace{-15mm}
	\begin{subfigure}[b]{0.5\textwidth}
		\centering
		\includegraphics[width=0.8\textwidth]{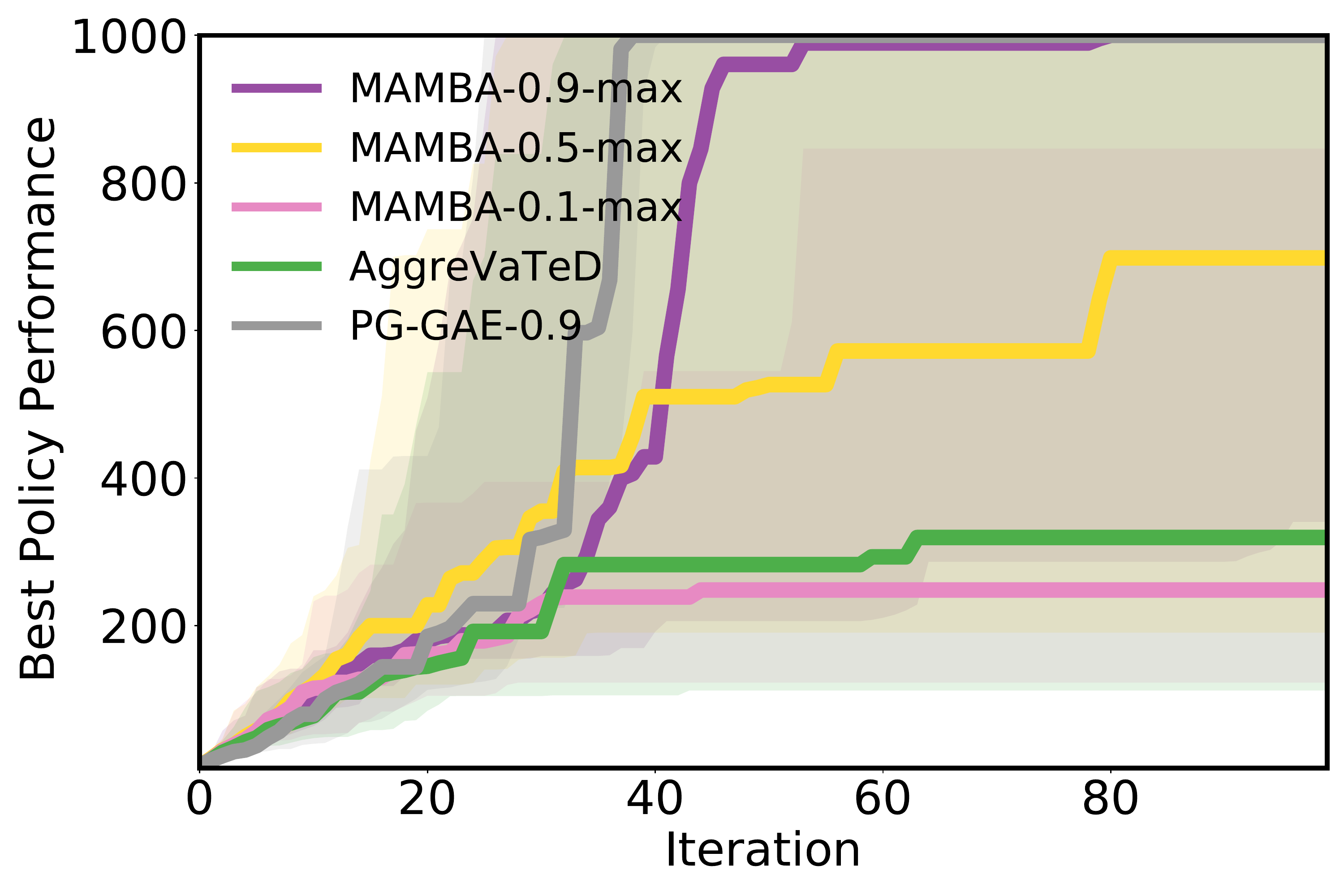}
		\caption{\small Effects of $\lambda$-weighting}
		\label{fig:lambda effects (random)}
	\end{subfigure}
	\hspace{-10mm}
	\begin{subfigure}[b]{0.5\textwidth}
		\centering
		\includegraphics[width=0.8\textwidth]{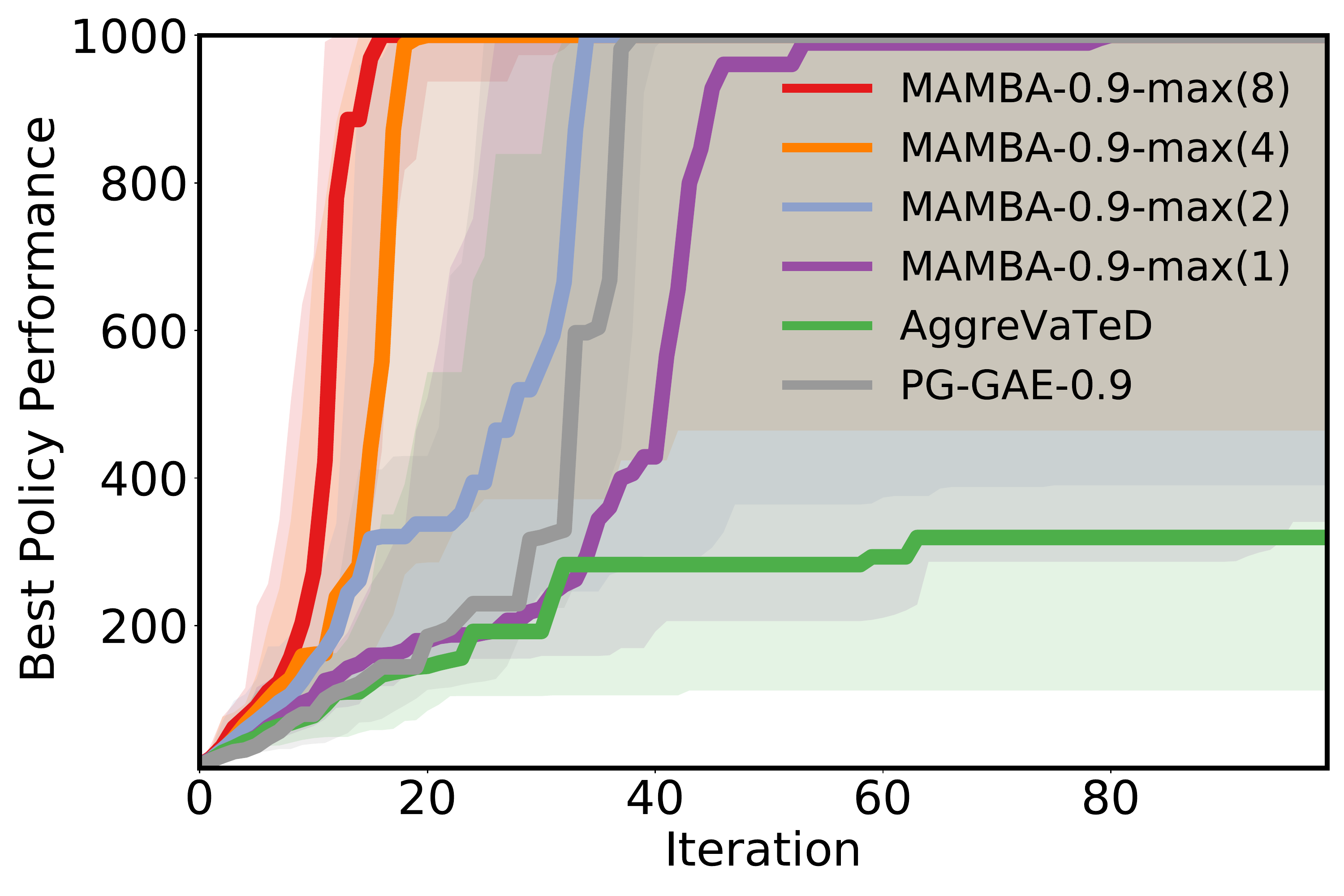}
		\caption{\small Effects of number of oracles}
		\label{fig:oracles effects (cp) (random)}
	\end{subfigure}
	\hspace{-15mm}
	\caption{\small
			Performance of the best policies with \emph{random orderings of oracles} in CartPole with Bad Oracles.
			(a) shows the single-oracle setup comparing \alg with different $\lambda$ values.
			(b) show \alg with different number of oracles ($\lambda=0.9$).
			A curve of \alg and \aggd shows the performance across 32 random seeds. The curves of \pg shows the performance across 8 random seeds.
			The shaded area shows 25th and 75th percentiles.
	}
	\label{fig:cartpole with random oracle ordering (bad oracles)}
\end{figure}

\subsection{Additional Experimental Results of Oracle Ordering}

In all the previous experiments, we order the oracles based on the their performance in terms of their return.
However, these return estimates are only empirical and do not always correspond to the true ordering of the oracles, as we discussed in \cref{sec:aggrevated brittleness}.
To study the robustness to oracle selection, here we randomly order the oracles before presenting them to the IL algorithms (\alg and \aggd) and repeat the controlled experiment of testing the effects of $\lambda$-weighting and the number of oracles in \cref{fig:controlled exps}. The results of random oracle ordering are presented in \cref{fig:cartpole with random oracle ordering (bad oracles)}; because of this extra randomness we inject in oracle ordering, we use more seeds in these experiments.
First, we see in \cref{fig:lambda effects (random)} using the random ordering degrades of the performance of the single-oracle setup. This is reasonable because there is a high chance of selecting an extremely bad oracle (see \cref{fig:CartPole with bad oracles} for the oracle quality).
Nonetheless, the usage of $\lambda>0$ still improves the performance: the learning is faster and converges to higher performance, though it is still slower than \pg because of the extremely bad oracles.

But interestingly once we start to use \emph{multiple} oracles in \cref{fig:oracles effects (cp) (random)}, \alg starts to significantly outperform \aggd and \pg. By using more than one oracle, there is a higher chance of selecting one reasonable oracle in the set filled with bad candidates. In addition, the diversity of oracle properties also help strengthen the baseline value function (cf. \cref{th:mamba performance}). Thus, overall we observe that \alg with $\lambda>0$ and multiple oracles yields the most robust performance.

\end{document}